\documentclass{article}

    \usepackage[final]{neurips_2022}

\usepackage[utf8]{inputenc} 
\usepackage[T1]{fontenc}    
\usepackage{hyperref}       
\usepackage{url}            
\usepackage{booktabs}       
\usepackage{amsfonts}       
\usepackage{nicefrac}       
\usepackage{microtype}      
\usepackage{xcolor}         

\usepackage{amsmath}
\usepackage{amsfonts}
\usepackage{graphicx}
\usepackage{placeins}
\graphicspath{ {./images/} }
\usepackage{xcolor}
\newcommand{\tp}[1]{\textcolor{red}{#1}}

\usepackage{wrapfig}

\newcommand{\x}{\mathbf{x}}

\usepackage{amssymb}
\usepackage{mathtools}
\usepackage{listings}

\usepackage{bbm} 
\usepackage{transparent} 
\usepackage{fancyvrb} 

\usepackage{algorithm}
\usepackage[noend]{algpseudocode}

\usepackage{pifont}
\newcommand{\cmark}{\ding{51}}%
\newcommand{\xmark}{\ding{55}}%

\usepackage{amsthm}
\usepackage[utf8]{inputenc}
\newtheorem{theorem}{Theorem}

\newtheorem{lemma}{Lemma}[theorem]

\newtheorem*{remark}{Remark}

\title{Censored Quantile Regression Neural Networks \\ for Distribution-Free Survival Analysis}

\author{%
  Tim Pearce$^{1,2,*}$, Jong-Hyeon Jeong$^3$, Yichen Jia$^3$, Jun Zhu$^1$\thanks{Majority of work completed while Tim Pearce was at Tsinghua University.  Correspondance to: tim.pearce@microsoft.com;~dcszj@tsinghua.edu.cn.}\\
  $^1$Dept. of Comp. Sci. \& Tech., NRist Center, Tsinghua-Bosch Joint ML Center, Tsinghua University \\ $^2$Microsoft Research, $^3$University of Pittsburgh\\
} 

\begin{document}

\maketitle

\begin{abstract}
\looseness=-1
This paper considers doing quantile regression on censored data using neural networks (NNs). This adds to the survival analysis toolkit by allowing direct prediction of the target variable, along with a distribution-free characterisation of uncertainty, using a flexible function approximator. We begin by showing how an algorithm popular in linear models can be applied to NNs. However, the resulting procedure is inefficient, requiring sequential optimisation of an individual NN at each desired quantile. Our major contribution is a novel algorithm that simultaneously optimises a grid of quantiles output by a single NN. To offer theoretical insight into our algorithm, we show firstly that it can be interpreted as a form of expectation-maximisation, and secondly that it exhibits a desirable `self-correcting' property. Experimentally, the algorithm produces quantiles that are better calibrated than existing methods on 10 out of 12 real datasets.

Code: \small{\textcolor{blue}{\url{https://github.com/TeaPearce/Censored_Quantile_Regression_NN}}.}
\end{abstract}


\section{Introduction}

Domains such as biomedical sciences and reliability engineering often produce datasets with a particular challenge -- for many datapoints the target variable is not directly observed and only a lower or upper bound is recorded. For example, when modelling the time to failure of a machine, in many cases it might have been preemptively maintained before a failure event was observed, meaning only a lower bound on the time to failure is recorded. This is known as censored data, which is studied in the field of survival analysis.




Recently, there has been interest in combining ideas from survival analysis with neural networks (NNs), in the hope of leveraging the capabilities of deep learning for this important class of problem.
Quantile regression has proven valuable in survival analysis, since it directly predicts the variable of interest, naturally capturing the uncertainty of the conditional distribution with no assumptions made about the distribution of residual errors \citep{Peng2021}. Despite a wealth of research in the linear setting, explorations into its combination with NNs are at an early stage \citep{Jia2022}. 
This paper advances this line of research, showing how an estimator from \cite{Portnoy2003}, used for doing quantile regression in linear models on left and right censored data, can be combined with NNs.





Section \ref{sec_bg} introduces Portnoy's estimator, and outlines a sequential grid algorithm that is used in its optimisation for linear models \citep{Neocleous2006}.
As our first contribution, Section \ref{sec_cqrnn_overview_seq_grid} shows how this algorithm can be directly adapted to work with NNs. Unfortunately, the resulting method is inefficient, requiring sequential optimisation of a new NN at each predicted quantile. Section \ref{sec_cqrnn_overview_cqrnn_alg} outlines our main contribution; an improved algorithm that is well suited to NNs, named `censored quantile regression neural network' (CQRNN). Our CQRNN algorithm is a significant departure from the sequential approach, and we offer theoretical insights into it in Section \ref{sec_cqrnn_analysis}, interpreting it as a form of expectation-maximisation (EM), and also analysing a `self-correcting' property.
In Section \ref{sec_experiments} the algorithm's effectiveness is empirically demonstrated by benchmarking against alternative methods on simulated and real data. 
Careful ablations illuminate effects of the algorithm's hyperparameters.



\section{Background}
\label{sec_bg}


We first introduce censored data. 
Consider a dataset $\mathcal{D} = \{ \{ \x_1, y_1, 
\Delta_1 \}, ... , \{ \x_N, y_N, 
\Delta_N \} \}$, where $\x_i \in \mathbb{R}^D$ is the input, $y_i \in \mathbb{R}$ is the (possibly censored) variable to regress on, and $\Delta_i$ represents the observed/censored indicator.
We assume two data generating distributions, one for the target variable, $t_i \sim p_t(t | \x_i)$, and another for the censoring variable, $c_i \sim p_c(c | \x_i)$. In the case of right censoring, we only observe the smaller of these two, which forms our dataset labels, $y_i = \min(t_i, c_i)$, and $\Delta_i = 
\Big\{\begin{smallmatrix}
0 \;\; \text{ if censored, } c_i<t_i \\
1 \;\; \text{ else \textcolor{white}{------------------}}
\end{smallmatrix} $.
Whilst this paper concentrates on the right censored case, for all estimators and algorithms discussed, left censoring can be handled by simply inverting the labels, i.e. using $-y_i \; \forall i$ \citep{Koenker2005}.



In this paper we make common assumptions about $p_t(t | \x_i)$ and $p_c(c | \x_i)$; the target is independent of censoring given the covariates, $t \perp c \; | \; \x$ (sometimes termed `random' censoring), but both $t$ and $c$ depend on $\x$ in different and possibly non-linear ways.
These are more general than some alternative assumptions, such as 
using fixed-value censoring, $c_i=\operatorname{constant} \;\forall i$ (e.g. \citep{Powell1986}) or requiring the censoring distribution to be independent of $\x$ (e.g. \citep{Jia2022}).

\subsection{Quantile Regression without Censoring}
Ignoring censoring for a moment, the conditional quantile function is given by,
\begin{equation}
    \label{eq_quantile_definition}
    Q(\tau | \x_i) = \inf \{ y_{i,\tau} : p(y_i \leq y_{i,\tau}| \x_i) = \tau \},
\end{equation}
where, $\tau \in (0,1)$. To learn such a function for a single target quantile, $\tau$, quantile regression minimises a `checkmark' loss, $\rho_\tau(\cdot,\cdot):\mathbb{R} \times \mathbb{R} \to \mathbb{R}$, \citep{Koenker1978},
\begin{align}
\label{eq_loss_no_censor}
\mathcal{L}_\text{check}(\theta, \mathcal{D}, \tau) \coloneqq& \frac{1}{N} \sum_{i=1}^N \rho_\tau(y_i, \hat{y}_{i,\tau}),\\
\label{eq_loss_checkmark}
\rho_\tau(y_i, \hat{y}_{i,\tau}) \coloneqq& (y_i - \hat{y}_{i,\tau})\left(\tau - \mathbb{I} [ \hat{y}_{i,\tau}>y_i ] \right),
\end{align}
for model parameters, $\theta$, and a prediction, $\hat{y}_{i,\tau} = \psi_\tau(\x_i, \theta)$, made by some model $\psi_\tau : \mathbb{R}^D \to \mathbb{R}$, with $\mathbb{I}[\cdot]$ as the indicator function. A linear model refers to when, $\hat{y}_{i,\tau} = \theta_\tau^\intercal \x_i $, with, $\theta_\tau \in \mathbb{R}^D$. 

\subsection{Portnoy's Censored Quantile Regression Estimator}

Directly optimising the loss in Eq. \ref{eq_loss_no_censor} can produce undesirable models when a dataset contains censored observations. Naively ignoring censoring indicators or excluding all censored data will, in the general case, induce unfavourable bias \citep{Zhong2021}.
The algorithms proposed in this paper use a re-weighting scheme introduced in \cite{Portnoy2003} that does account for censoring. Define two disjoint sets of indices, one for censored and one for observed datapoints, $\mathcal{S}_\text{censored}$ and $\mathcal{S}_\text{observed}$, letting, $N_c \coloneqq |\mathcal{S}_\text{censored}|$ and $N_o \coloneqq |\mathcal{S}_\text{observed}|$.
Portnoy's estimator minimises,
\begin{equation}
\label{eq_loss_reweight}
\mathcal{L}_\text{Port.}(\theta, \mathcal{D}, \tau, \mathbf{w}, y^*) = \sum_{i \in \mathcal{S}_\text{observed}} \rho_\tau(y_i, \hat{y}_{i,\tau}) 
+ \sum_{j \in \mathcal{S}_\text{censored}} w_j \rho_\tau(y_j, \hat{y}_{j,\tau}) +  (1-w_j) \rho_\tau(y^*, \hat{y}_{j,\tau}) .
\end{equation}
While the loss for observed datapoints is unchanged, censored datapoints have been split into two `pseudo' datapoints -- one at the censored value, and one at some large value, $y^* \gg \max_i y_i$ (Section \ref{sec_exp_hyperparam} discusses $y^*$ further).
Weight is apportioned between each pair of pseudo datapoints by,
\begin{align}
    w_j = \frac{\tau - q_j}{1 - q_j},
\end{align}
where $q_j$ is the quantile at which the datapoint was censored, with respect to the target value distribution, i.e. $p_t(t_j>c_j|\x_j)$. We use $\mathbf{w} \in \mathbb{R}^{N_c}$ and $\mathbf{q} \in \mathbb{R}^{N_c}$ to denote the vector of all weights and quantiles respectively, indexed as $w_j$ and $q_j$. 
Given these weights, Portnoy's estimator has been shown to be analogous to the Kaplan-Meier (KM) estimator \citep{Portnoy2003}. 

\begin{algorithm}[t]
\caption{Sequential grid algorithm for NNs.}\label{alg_seq_grid}
\begin{algorithmic}
\Require Dataset $\mathcal{D}$, $M$ parametric models $\psi_\tau(\cdot)$ each with randomly initialised parameters $\theta_\tau$, ordered quantiles to be estimated $\operatorname{grid}_\tau$, learning rate $\alpha$, pseudo y value $y^*$.
\newline
\State $\mathcal{S}_\text{censored} \gets$  $\{i \in \{0,1, \dots , N\} : \Delta_i = 0 \}, \mathcal{S}_\text{observed} \gets$  $\{i \in \{0,1, \dots , N\} : \Delta_i = 1 \}$, $K_\text{cross} \gets \emptyset $

\For{$i = 0$ \textbf{to} $M-1$}
        \State $\tau \gets \operatorname{grid}_\tau[i]$
        \If {$i = 0$} \Comment{Initialise quantile estimates to zero}
        \State $\mathbf{\hat{q}} \gets \operatorname{zeros(N_c)}$ 
        \Else \Comment{Find indices which have been crossed and update their estimated quantiles}
        \State $K \gets \{ j \in \mathcal{S}_\text{censored}: \psi_{\operatorname{grid}_\tau[i-1]}(\x_j) \leq y_j \cap \psi_{\operatorname{grid}_\tau[i]}(\x_j) >  y_j \}$ 
        \State $\mathbf{\hat{q}}[K] \gets  \operatorname{grid}_\tau[i-1]$
        \State $K_\text{cross} \gets K_\text{cross} \cup K$ 
        \State $\mathbf{\hat{q}}[\neg K_\text{cross}] \gets \tau$ \Comment{Sets $\mathbf{\hat{w}}=0$  for uncrossed datapoints}
        \EndIf
        
        \State $\mathbf{\hat{w}} \gets (\tau - \mathbf{\hat{q}})/(1 - \mathbf{\hat{q}})$
        
        \State $\theta \gets \arg \min_{\theta_\tau} \mathcal{L}_\text{Port.}(\theta_\tau, \mathcal{D}, \tau, \mathbf{\hat{w}}, y^*) $  \Comment{Fully optimise Eq. \ref{eq_loss_reweight}}
        
        
        
\EndFor
\end{algorithmic}
\end{algorithm}

\subsection{Sequential Grid Algorithm}

A challenge with Portnoy's estimator is that it creates a circularity. If the true quantiles $\mathbf{q}$ and hence $\mathbf{w}$ for all censored datapoints were known, Eq. \ref{eq_loss_reweight} could be optimised straightforwardly, but the very reason to perform this optimisation is to estimate such quantiles!
Prior work has tackled this issue in various ways (see Section \ref{sec_relatedwork}). Here we discuss the most widely-used algorithm. It was originally presented in \cite{Portnoy2003} and slightly modified in \cite{Neocleous2006}, we refer to it as the `sequential grid algorithm'.


Described in Algorithm \ref{alg_seq_grid}, it sequentially steps through a grid of $M$ desired quantiles, arranged in strictly increasing order, and typically evenly spaced, e.g., $ \operatorname{grid}_\tau = \{0.1, 0.2, \dots , 0.8 , 0.9\}$ for $M = 9$. A new model is fitted at each quantile. 
The algorithm terminates either when all quantiles in the grid have been iterated through, or if only censored datapoints lie above the current quantile.
(\cite{Portnoy2003} suggest handling the first quantile specially, but we have simplified this -- see later.)
Note we introduce $\mathbf{\hat{q}}$ \& $\mathbf{\hat{w}}$ to explicitly designate model estimates of $\mathbf{{q}}$ \& $\mathbf{{w}}$.




\begin{figure}[b!]
\begin{center}
\vspace{-0.02in}

\small CQRNN (ours) \hspace{0.6in}  Excl. censor \hspace{0.6in}  DeepQuantReg \hspace{0.6in}  LogNorm MLE

\vspace{-0.05in}
\includegraphics[width=0.245\columnwidth,height=0.13\columnwidth]{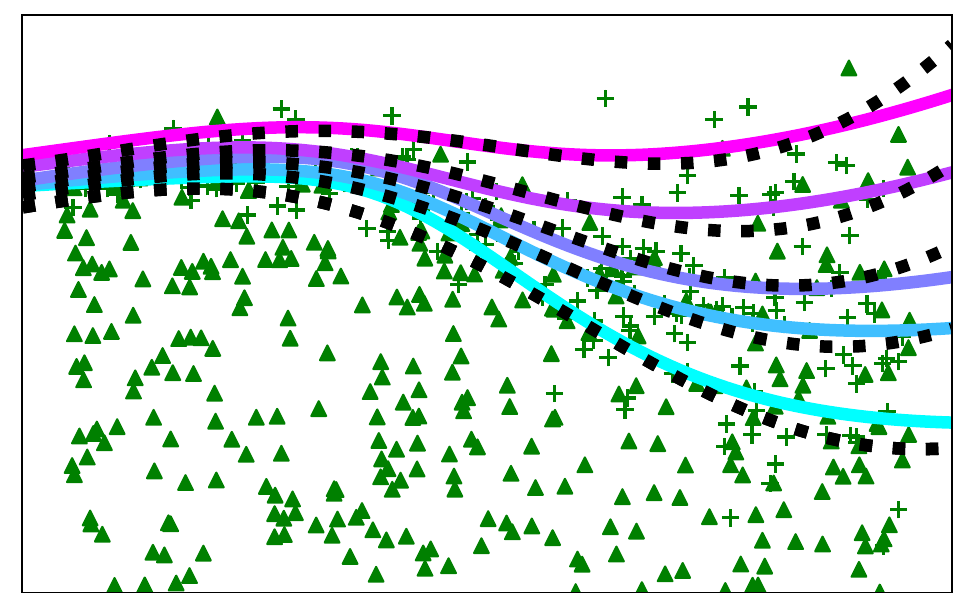}
\put(-107,0){\rotatebox{90}{\small Norm uniform}}
\includegraphics[width=0.245\columnwidth,height=0.13\columnwidth]{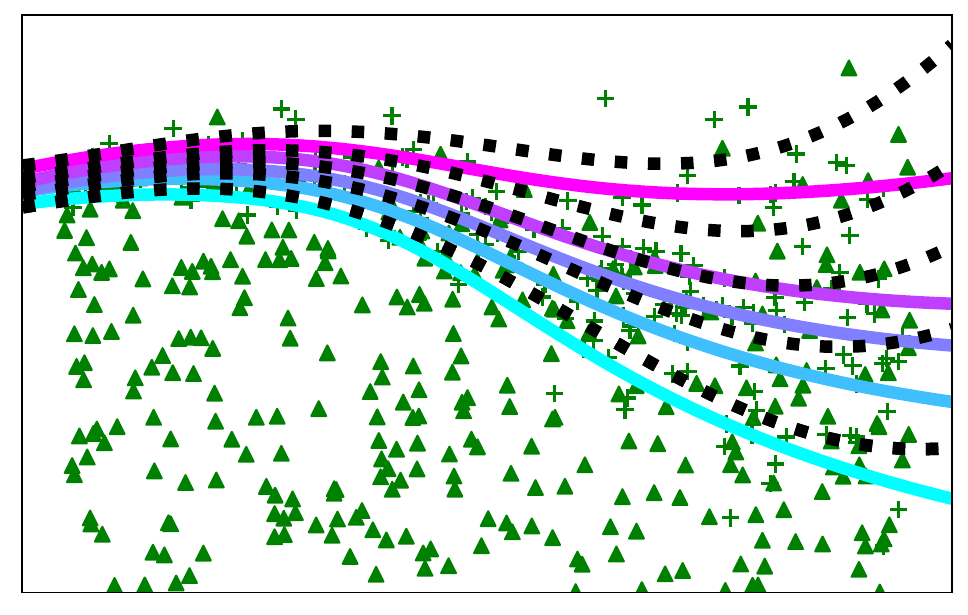}
\includegraphics[width=0.245\columnwidth,height=0.13\columnwidth]{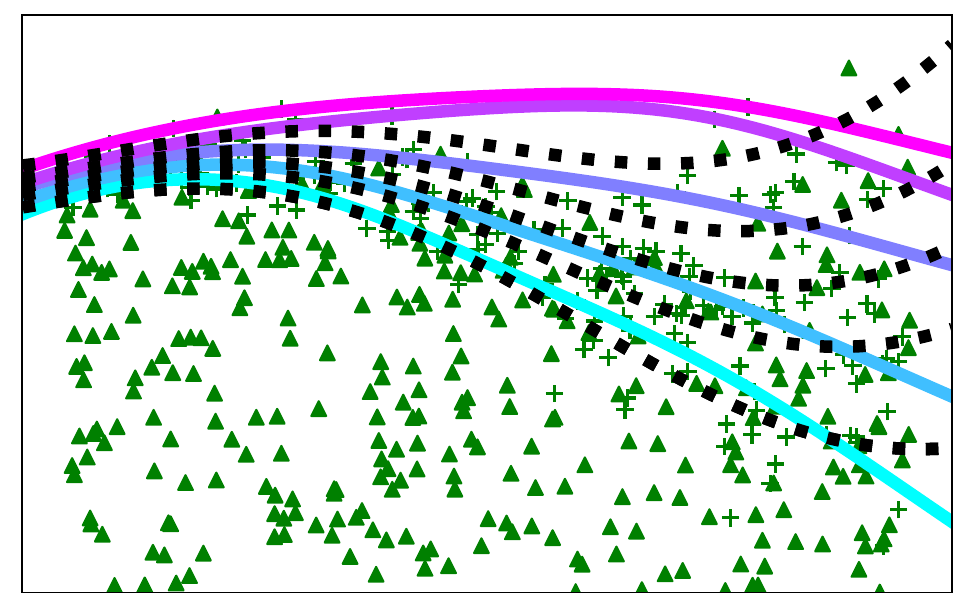}
\includegraphics[width=0.245\columnwidth,height=0.13\columnwidth]{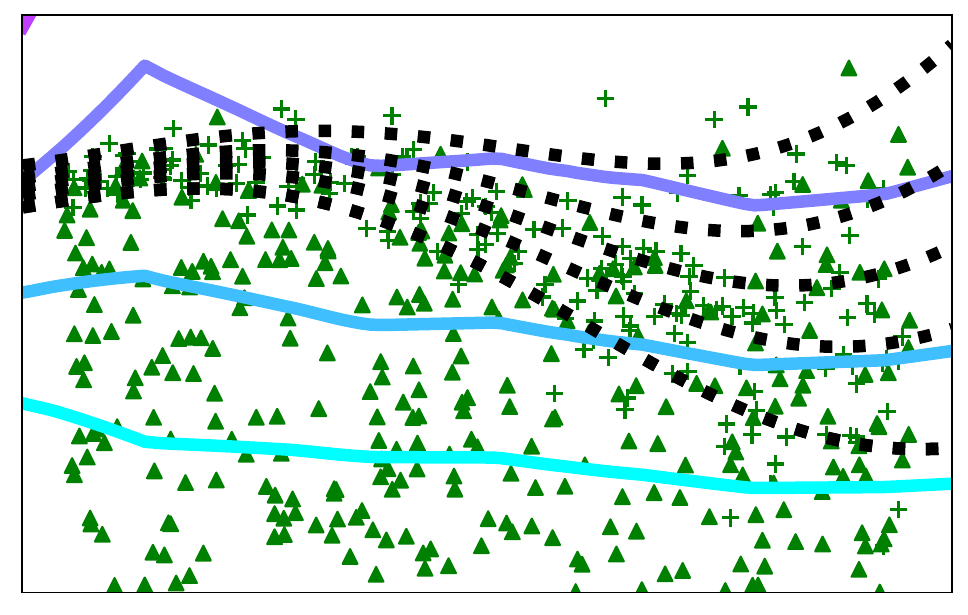}
\put(-0,4){\transparent{1.}\includegraphics[width=0.1\columnwidth]{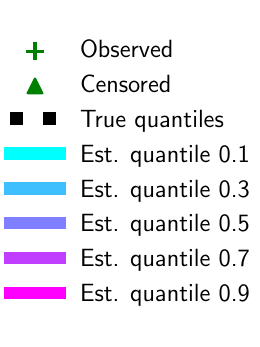}}

\includegraphics[width=0.245\columnwidth,height=0.13\columnwidth]{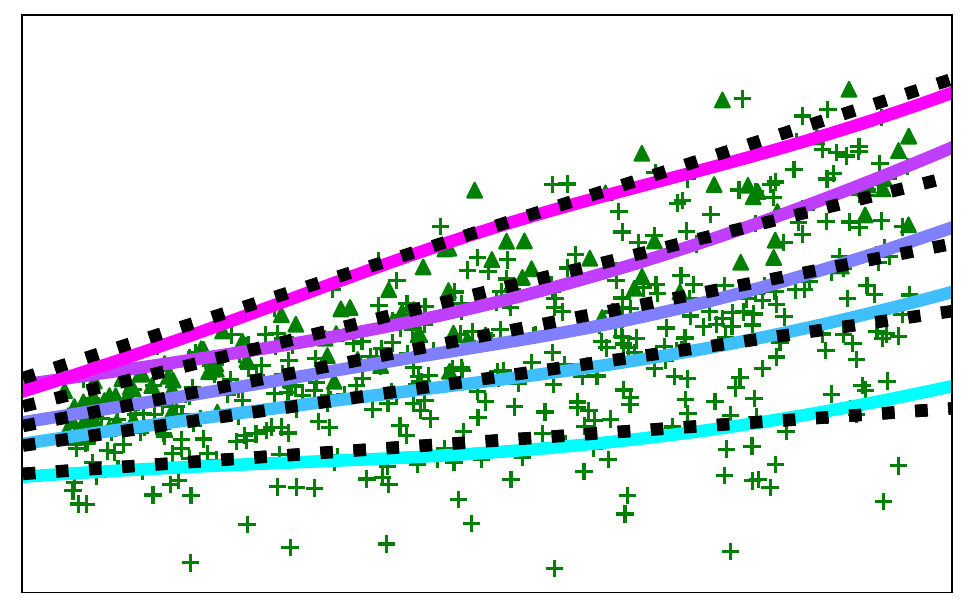}
\put(-107,4){\rotatebox{90}{\small Norm linear}}
\includegraphics[width=0.245\columnwidth,height=0.13\columnwidth]{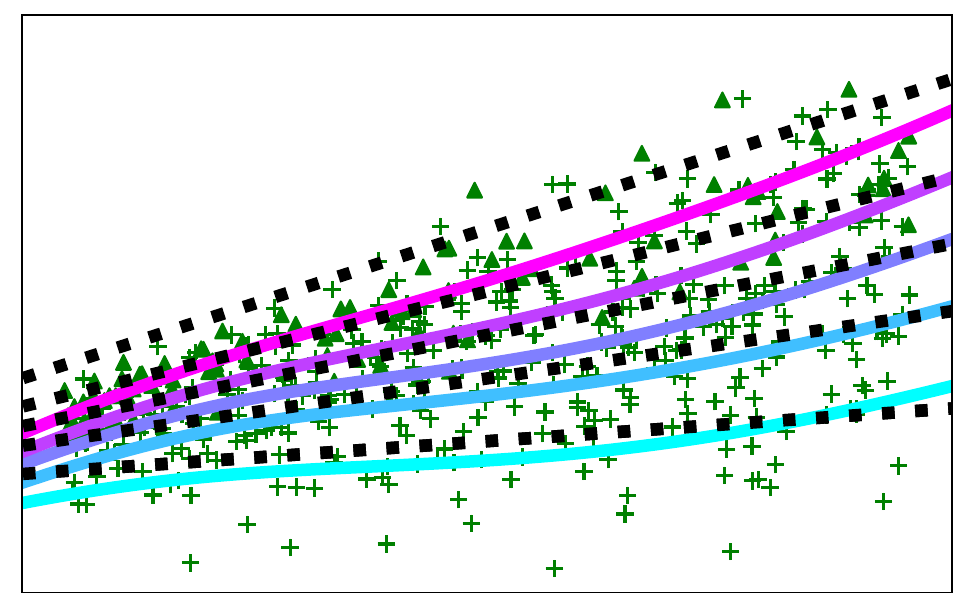}
\includegraphics[width=0.245\columnwidth,height=0.13\columnwidth]{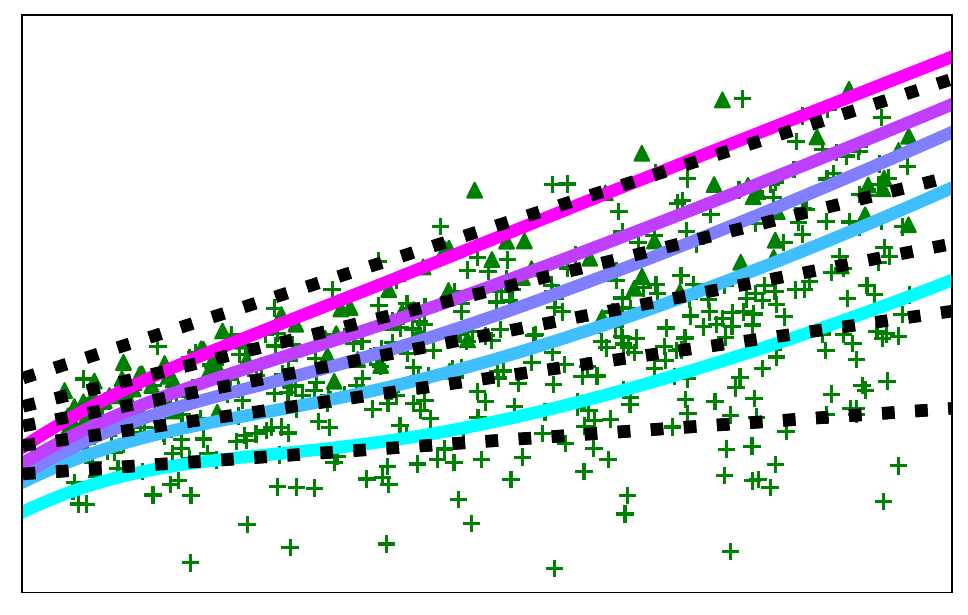}
\includegraphics[width=0.245\columnwidth,height=0.13\columnwidth]{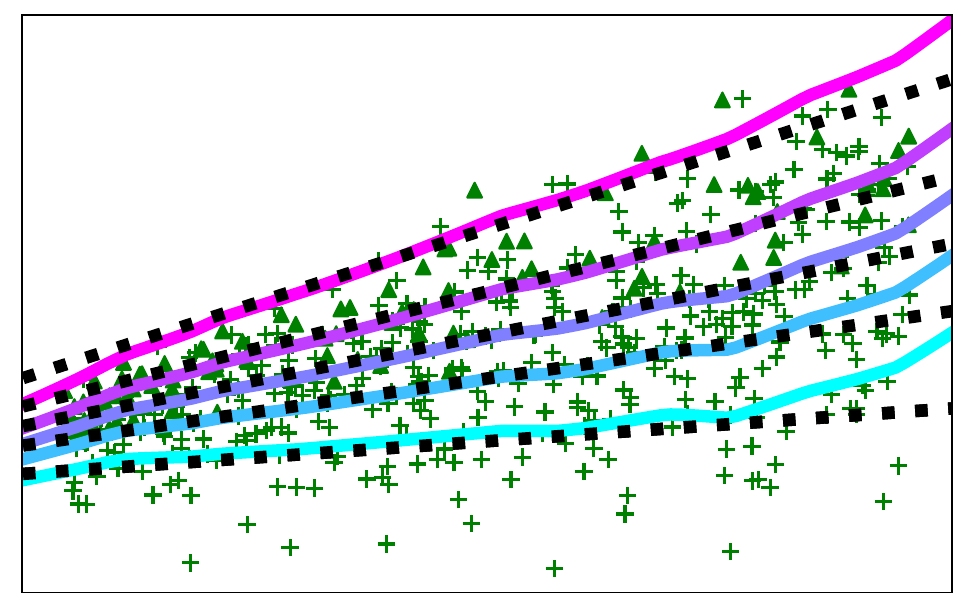}


\includegraphics[width=0.245\columnwidth]{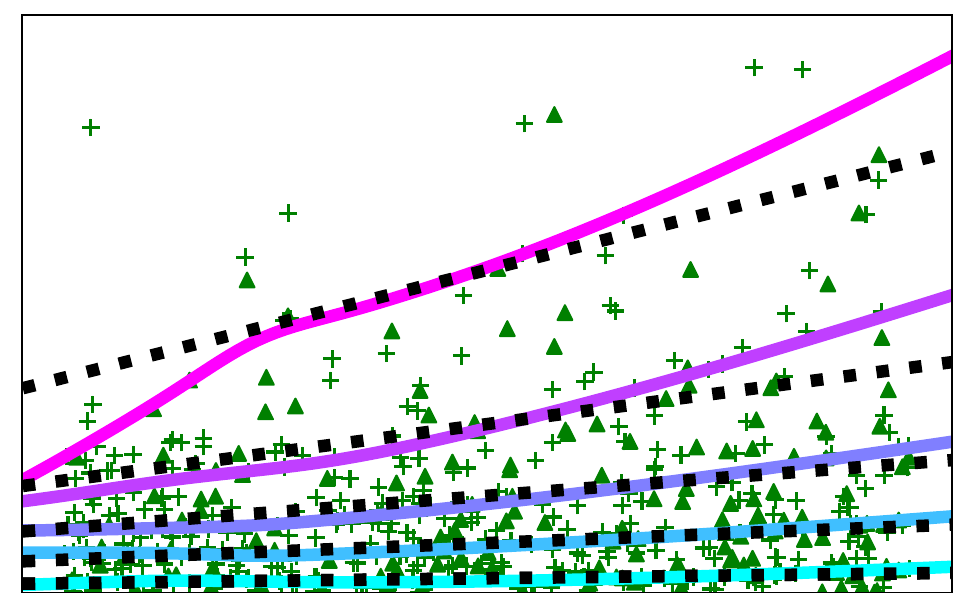}
\put(-107,8){\rotatebox{90}{\small Exponential}}
\includegraphics[width=0.245\columnwidth]{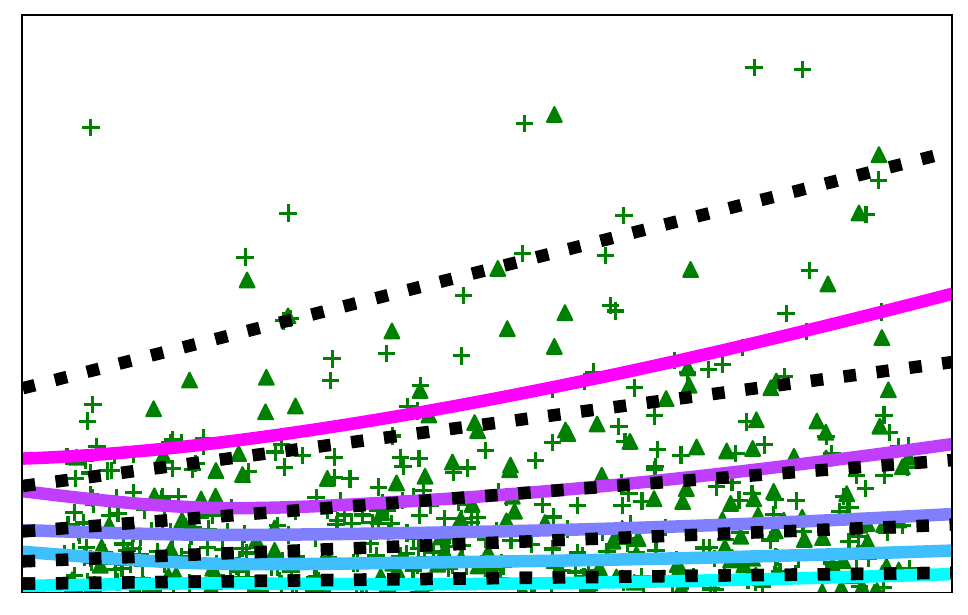}
\includegraphics[width=0.245\columnwidth]{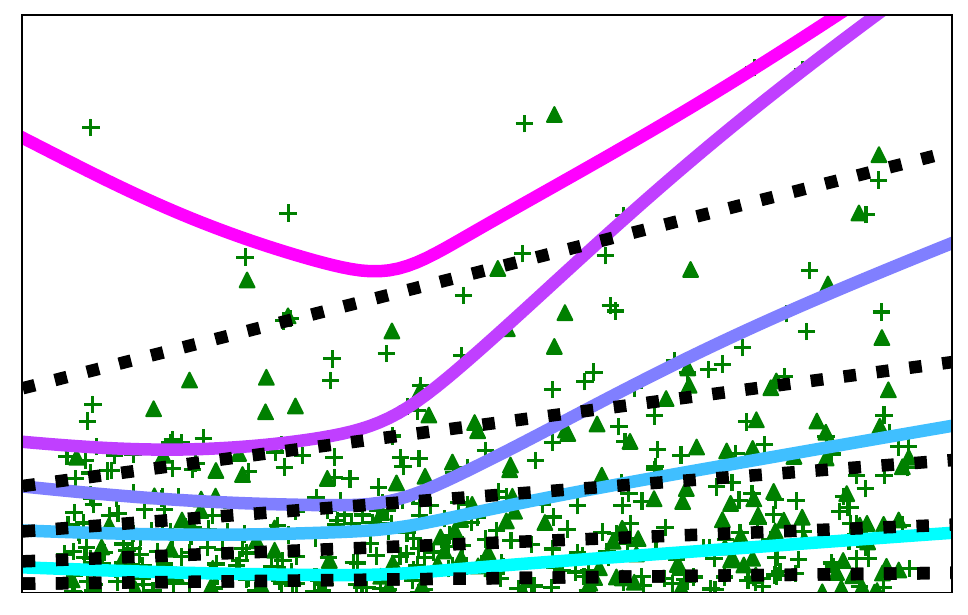}
\includegraphics[width=0.245\columnwidth]{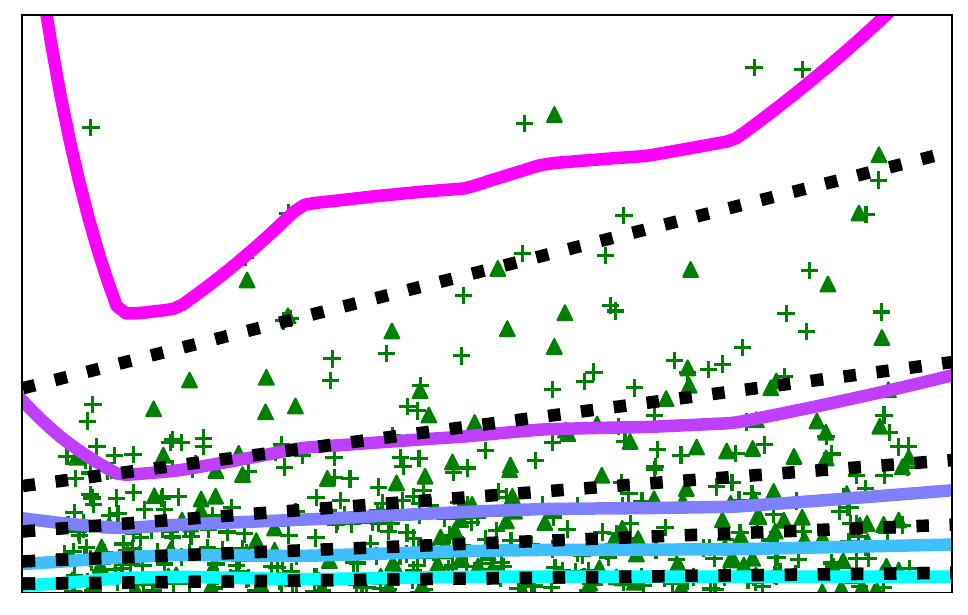}


\includegraphics[width=0.245\columnwidth,height=0.13\columnwidth]{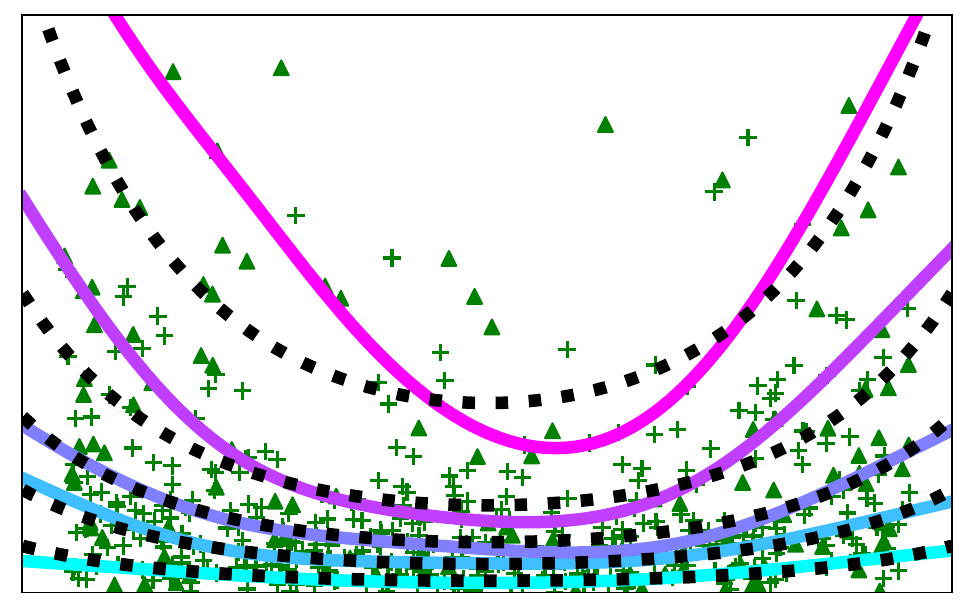}
\put(-107,6){\rotatebox{90}{\small LogNorm}}
\includegraphics[width=0.245\columnwidth,height=0.13\columnwidth]{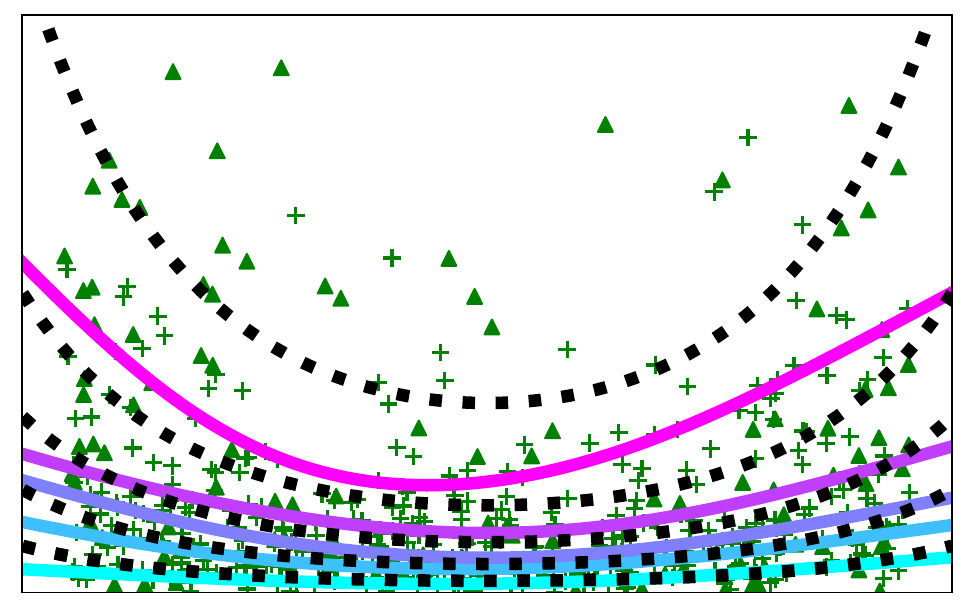}
\includegraphics[width=0.245\columnwidth,height=0.13\columnwidth]{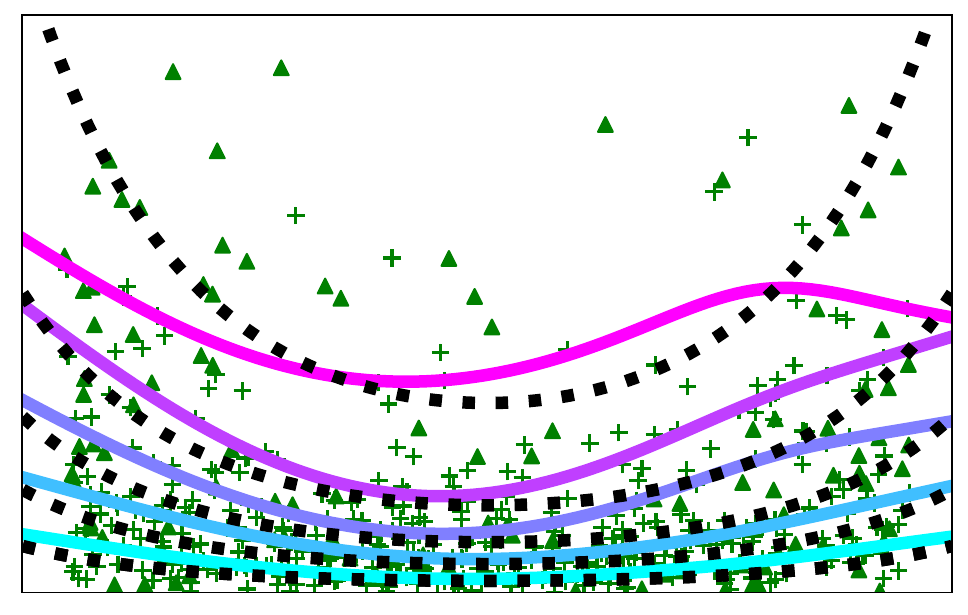}
\includegraphics[width=0.245\columnwidth,height=0.13\columnwidth]{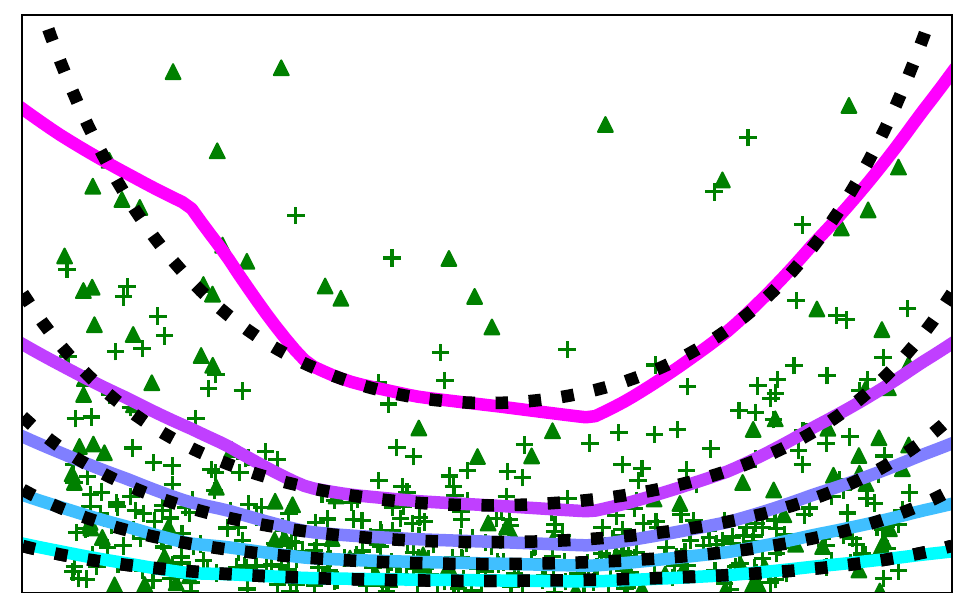}

\caption{1D synthetic datasets of varying functions (rows), fitted by various methods (columns). Estimated quantiles (blue through pink) compared to ground truth quantiles (dashed black lines).
CQRNN recovers quantiles closest to the ground truth on most datasets.}
\label{fig_1D_plots_mini}
\end{center}
\vskip -0.2in
\end{figure}

\section{Censored Quantile Regression and Neural Networks}
\label{sec_cqrnn_overview}

This section considers applying the sequential grid algorithm to NNs, explaining why this is inefficient. It then describes the CQRNN algorithm, our main contribution, which avoids these pitfalls.

\subsection{The Sequential Grid Algorithm and Neural Networks}
\label{sec_cqrnn_overview_seq_grid}

\looseness=-1
To the best of our knowledge the sequential grid algorithm has been proposed and implemented only on linear models. Our first contribution is showing that, with three adjustments, it can be used in NNs. 


\looseness=-1
Adj. 1) The $M$ models producing estimates, $\hat{y}_{j, \tau}$, are chosen be NNs, of any architecture with a single output.
Adj. 2) While the linear version of the algorithm minimises Eq. \ref{eq_loss_reweight} analytically with a parametric programming approach, for NNs this is instead minimised with stochastic gradient descent.
Adj. 3) \cite{Portnoy2003} proposed using a lengthy process for the first quantile. They set $\mathbf{\hat{w}} = \mathbf{1}$ and optimise Eq. \ref{eq_loss_reweight}. After optimisation, if any censored datapoints lie below the first quantile, these are removed from the dataset, and the procedure repeats until there are none. This repeated optimisation could go on for many iterations for large datasets and is costly for NNs. Instead, we simply set $\mathbf{\hat{q}} = \mathbf{0}$, and optimise the first quantile once only. We show this is roughly equivalent in Appendix \ref{sec_app_theory_first_iter}.

Unfortunately, this adaption is still inefficient.
Inspecting Algorithm \ref{alg_seq_grid}, one notes that $M$ NNs must be trained sequentially, meaning time complexity (both at training and at test) and memory complexity is $\mathcal{O}(M)$. This is not the case when doing quantile regression with NNs \textit{without} censoring -- a single NN typically uses multiple outputs, one per quantile, and their losses can be minimised simultaneously \citep{Taylor2000}. For modest numbers of quantiles, this results in little overhead relative to a single NN, since the extra quantile outputs require only linear heads, while the bulk of the parameter count and computation time arise from the shared trunk of the NN.

Note that these inefficiencies are less problematic for linear models -- optimisation of Eq. \ref{eq_loss_reweight} can be performed quite rapidly, and memory complexity cannot obviously be improved beyond $\mathcal{O}(M)$. As such, it is the use of NNs that motivates the search for an alternative algorithm.





\subsection{CQRNN Algorithm}
\label{sec_cqrnn_overview_cqrnn_alg}

Our proposed algorithm uses a single NN outputting a grid of quantile estimates that can be optimised simultaneously. This results in a more efficient algorithm in terms of training time, test time, and memory.
Concretely, if the NN is a single hidden-layer NN with $H$ neurons, then the model is defined, $\psi_\tau(\x_j,\theta) = \theta_{1,\tau}^\intercal \phi(\theta_0^\intercal \x_i)$, with $\theta_0 \in \mathbb{R}^{D \times H}$ shared for all estimates and $\theta_{1,\tau} \in \mathbb{R}^{H \times 1}$ for the specific output head, and non-linearity, $\phi$.

The key insight behind our algorithm is that estimates of the censored weights, $\mathbf{\hat{w}}$, can be bootstrapped from the model \textit{while} it learns, without requiring sequential optimisation. Intuitively, as the model trains, the estimates of the censored quantiles, $\mathbf{\hat{q}}$, and hence the censored weights, $\mathbf{\hat{w}}$, improve, and these feedback to create a more accurate loss function, allowing further improvement of the model. 



We summarise the method in Algorithm \ref{alg_cqrnn_loss}.
The quantile of a censored datapoint is estimated as whichever NN output is closest under the current model parameters. 
Whilst the algorithm is written in two steps, getting up-to-date estimates of censored quantiles, $\hat{q}_j$, requires only a forward pass, which is done at the point of loss optimisation anyway, so actually comes `for free'.
In our implementations, the algorithm is extended to leverage mini-batches and modern optimisers (e.g. Adam).



\begin{algorithm}[t]
\caption{CQRNN algorithm.}\label{alg_cqrnn_loss}
\begin{algorithmic}
\Require Dataset $\mathcal{D}$, one parametric model $\psi(\cdot)$ with randomly initialised parameters $\theta$ that can output $M$ quantile predictions, quantiles to be estimated $\operatorname{grid}_\tau$, learning rate $\alpha$, pseudo y value $y^*$.
\newline
\State $\mathcal{S}_\text{censored} \gets$  $\{i \in \{0,1, \dots , N\} : \Delta_i = 0 \}, \mathcal{S}_\text{observed} \gets$  $\{i \in \{0,1, \dots , N\} : \Delta_i = 1 \}$ 
\While{\textbf{not} convereged}
        \State \textit{1. Hard expectation step}
        \For{$j \in \mathcal{S}_\text{censored}$} 
        \State $\hat{q}_j \gets \arg \min_\tau |\hat{y}_{j,\tau} -  y_j|$ \Comment{Estimate quantile of censored data under current model}
        \State $\hat{w}_j \gets \max\left( (\tau - \hat{q}_j)/(1 - \hat{q}_j), 0 \right)$ \Comment{Avoid negative weights if $\hat{q}_j>\tau$}
        \EndFor
        \State \textit{2. Partial maximisation step}
        \State $\theta \gets \theta - \alpha \partial\mathcal{L}(\theta_\tau, \mathcal{D}, \tau, \mathbf{\hat{w}}, y^*) / \partial \theta \;\; \forall \tau \in \operatorname{grid}_\tau$ \Comment{Gradient step to minimise Eq. \ref{eq_loss_reweight}} 
\EndWhile
\end{algorithmic}
\end{algorithm}



\section{Analysis of CQRNN}
\label{sec_cqrnn_analysis}

The previous section introduced the CQRNN algorithm, motivating it intuitively, while this section considers its correctness.
Providing precise analytical guarantees is very challenging, both due to the involvement of non-convex NNs, and the bootstrapping nature of the algorithm.
We therefore consider more holistic ways of understanding and providing confidence in the method.
We firstly show that it can be interpreted as a flavour of EM (`generalised hard' EM) --
connecting to this well-studied class of algorithms provides some justification of CQRNN's design.
We secondly validate the way the model bootstraps its own censored quantile estimates by describing a `self-correcting' property.

\subsection{Interpretation as Expectation-Maximisation}


EM is a two-stage iterative optimisation technique for finding maximum likelihood solutions. It can be useful when one has both observed and latent variables, and where knowing the latent variables would allow straightforward optimisation of a model's parameters and vice versa. 

In our problem, we interpret the weights of the censored datapoints, $\mathbf{w}$, as latent variables, observed variables as $\x$ \& $y$, and the NN parameters $\theta$ are to be optimised.
Under our notation, the vanilla EM procedure \citep{Bishop2013} can be written as follows (using a randomly initialising $\theta^\text{old}$),
\begin{enumerate}
    \item Expectation step. Evaluate, $p(\mathbf{w} | \x, y, \theta^\text{old})$, to set up the expectation,
\begin{equation}
    \label{eq_Q_definition_EM}
    \mathcal{Q}(\theta, \theta^\text{old}) = \int_\mathbf{w} p(\mathbf{w} | \x, y, \theta^\text{old}) \log p(y, \mathbf{w}|\x,\theta) d \mathbf{w} = \mathbb{E}_{\mathbf{w}|\x,y,\theta^\text{old}} [\log p(y, \mathbf{w}|\x,\theta)].
\end{equation}
\item Maximisation step. Perform optimisation of the expected likelihood.
\begin{equation}
    \theta^\text{new} = {\arg \max}_\theta \mathcal{Q}(\theta, \theta^\text{old})
\end{equation}
\item Set $\theta^\text{old} \gets \theta^\text{new}$ and repeat until convergence. 
\end{enumerate}



We now show how our CQRNN method in Algorithm \ref{alg_cqrnn_loss} can be interpreted as a flavour of EM, specifically generalised hard EM. The key steps in this interpretation are: 1) The likelihood for datapoints at each quantile are chosen to follow an asymmetric Laplace distribution. 2) The expectation step is treated as a hard assignment under the current NN parameters. 3) The likelihood is only partially maximised at each iteration.


\subsubsection*{Likelihood Form}

A likelihood function is required for both the E \& M steps. Theorem \ref{theorem_likelihood} shows how Eq. \ref{eq_loss_reweight} can be interpreted as a negative log likelihood.

\textbf{Theorem \ref{theorem_likelihood}.} 
\textit{Let the likelihood for each datapoint at each quantile be an asymmetric Laplace distribution with scale, $\lambda = \sqrt{\tau - \tau^2}$, and asymmetry, $k=\tau/\sqrt{\tau - \tau^2}$.
The negative log likelihood is,}
\begin{align}
    -\log p(y |\x,\theta, \mathbf{w}, y^*) &= \sum_{\tau \in \text{grid}_\tau}  \mathcal{L}_\text{Port.}(\theta, y, \x, \tau, \mathbf{w}, y^*) + \operatorname{constant}.
\end{align}

\textit{Proof sketch.} The asymmetric Laplace distribution has been used in Bayesian treatments of quantile regression without censoring \citep{Yu2001}. We extend this to Portnoy's loss with a weighted likelihood form using the censored weights. Appendix \ref{sec_app_theory} provides the full proof.


\subsubsection*{Expectation as a Hard Assignment of Latent Variables}

In some models, such as a Gaussian mixture model, the expectation step can be computed through analytical evaluation of the latent posterior distribution \citep{Bishop2013}. In our case, $p(\mathbf{w} | \x, y, \theta^\text{old})$, is intractable, though Eq. \ref{eq_loss_reweight} depends linearly on each element $w_j$, so we need only consider the expectation, $\mathbb{E}_{w_j | \x, y, \theta^\text{old}}[w_j]$. The algorithm can then be interpreted as making an assignment under the current model parameters,
\begin{align}
\hat{\mathbb{E}}_{w_j | \x_j, y_j, \theta^\text{old}} [w_j] &= \sum_{w_j} w_j \hat{p}(w_j | \x_j, y_j, \theta^\text{old}) = \frac{\tau - \hat{q}_j}{1-\hat{q}_j},
\\
\label{eq_prior_EM_interp}
\hat{p}(w_j | \x_j, y_j, \theta^\text{old}) &=
    \begin{cases}
    1 \text{ if } w_j = \frac{\tau - \hat{q}_j}{1-\hat{q}_j} \text{, where } \hat{q}_j = \arg \min_\tau |\hat{y}_{j,\tau} -  y_j|\\
    0 \text{ else}
    \end{cases} .
\end{align}
This is a `hard' assignment, where latents are assigned to the most likely values under the current model (e.g. datapoints are attributed to the nearest clusters in the K-means algorithm \citep{Bishop2013}), and gives rise to a class of algorithms termed hard EM \citep{Samdani2012}.

Note that the in Eq. \ref{eq_Q_definition_EM} we can rewrite, 
$\log p(y,\mathbf{w} | \x, \theta) = \log p(\mathbf{w})p(y | \x, \theta, \mathbf{w})$. Since we choose only a single setting for $\mathbf{w}$ in the outer expectation, this dependence disappears, and we require maximisation of, $\mathbb{E}_{\mathbf{w}|\x,y,\theta^\text{old}} [\log p(y|\x,\theta, \mathbf{w})]$.




\subsubsection*{Partial Maximisation}

A second departure from the standard EM algorithm, is that Algorithm \ref{alg_cqrnn_loss} performs only a partial maximisation of the likelihood, taking a single gradient step. This has been shown to provide similar guarantees, and is termed `generalised' EM \citep{EMradford}.

One could consider an alternative version of the algorithm that is closer to the standard EM procedure, where the estimated quantiles of censored data points are fixed while the maximisation step of NN parameters is run to convergence over multiple training epochs, before the estimates are updated in the expectation step, and repeating. 
In our case this is less effective
-- the maximisation requires running a forward pass through the NN, and obtaining estimated quantiles under current NN parameters once this is done is trivial, so we can access up-to-date estimates essentially for free.
Appendix Figure \ref{fig_partial_full} empirically demonstrates that CQRNN's partial maximisation approach produces the fastest convergence.





\subsection{Self-Correcting Property}


The CQRNN bootstraps weight estimates from the current model as it trains. It's not immediately obvious why this bootstrapping approach should converge to something sensible -- what if bad initial weight estimates lead to worse ones?
As a second insight into the CQRNN algorithm, Theorem \ref{theorem_self_correct} shows that when a censored weight, $\hat{w}_j$, is estimated incorrectly, the algorithm acts in a way to adjust the estimated quantiles in a favourable way -- we refer to this as `self-correcting'.

\textbf{Theorem \ref{theorem_self_correct}.} 
\textit{
If $\hat{q}_j$ is underestimated, one iteration of the algorithm acts to \textit{increase} the quantile predictions, $\hat{y}_{j,\tau}$, by the same amount, or even higher, than if the weight had been correct. 
If $\hat{q}_j$ is overestimated, $\hat{y}_{j,\tau}$, is decreased by the same amount, or even lower, than with the correct weight.
}

\textit{Proof sketch.} 
Denote ${q}_j$ the true quantile that censored datapoint $j$ is censored in, and ${w}_j$ the corresponding true weight. 
We derive the expression for the gradient wrt the predicted quantiles, $\hat{y}_{j,\tau}$, finding that if $\hat{w}_j$ is underestimated it holds that, $\frac{\partial \mathcal{L}_\text{Port.}(\theta, \mathcal{D}, \tau, \hat{\mathbf{w}}, y^*)}{\partial \hat{y}_{j,\tau}} \leq \frac{\partial \mathcal{L}_\text{Port.}(\theta, \mathcal{D}, \tau, {\mathbf{w}}, y^*)}{\partial \hat{y}_{j,\tau}}$, and hence gradient descent applies an adjustment in the desired direction of equal or greater magnitude than if the true weight had been used. 
The reverse holds for overestimated $\hat{w}_j$. Appendix \ref{sec_app_theory} provides a full proof.



\section{Related Work}
\label{sec_relatedwork}

\textbf{Survival analysis and NNs.}
Following the widespread success of deep NNs over the past decade, there has been a wave of research applying NNs to survival analysis -- for instance by modifying the CoxPH model \citep{Katzman2018}, or framing the task as ordinal classification \citep{LeeDeepHit2018}. Closer to our work are methods that use NNs to output parameters of a distribution such as the Weibull \citep{Martinsson2016}, seek robust training objectives for these models \citep{Avati2019}, or help with their optimisation \citep{tang2022}. Our work stands out as offering a way to directly estimate the target variable at pre-specified quantiles, without enforcing any distributional assumption. A limitation is that the distribution may only be predicted at these quantiles (unless additional assumptions are made to allow interpolation between these).



\textbf{Quantile regression and NNs.}
Quantile regression has proven an attractive option to enable NNs to move beyond point predictions. This allows quantification of a NN's aleatoric uncertainty \citep{Tagasovska2018}. It is attractive due to its straightforward implementation, and avoidance of any distributional assumption. For example, it has found use in reinforcement learning to capture the \textit{distribution} of rewards, rather than just the mean \citep{Dabney2017distRL}.
Since NNs are flexible function approximators, particular attention has been paid to the crossing quantile problem 
\citep{Bondell2010, Zhou2020crossing, Brando2022}. 
CQRNN borrows ideas from this line of work, and further tackles the challenge of learning quantiles under censored data.

\textbf{Censored quantile linear regression.} 
There is much work on censored quantile regression methods for linear models. 
\cite{Powell1986} developed an estimator under fixed-value censoring which can be implemented with an algorithm from \cite{Fitzenberger1997}.
\cite{Portnoy2003} developed an estimator under random censoring based on the KM estimator, while \cite{Peng2008a} developed an alternative based on the Nelson–Aalen (NA) estimator. KM and NA are closely related, and Portnoy and Peng's methods have been reported to offer similar empirical performance \citep{Koenker2008}.
\cite{Koenker2005} provides all above methods in the popular `\Verb"quantreg"' R package.
Other notable methods include \cite{Yang2018}, based on the data augmentation algorithm, and \cite{Wang2009a}, whose estimator is similar to Portnoy's but utilises local estimates of the KM, computed with a kernel method. 
See \cite{Peng2021} for a review of the broader area.
Our work allows modelling of flexible non-linear quantile functions, leveraging the powerful representation learning abilities of NNs. Although, this sacrifices the interpretability of coefficients of linear models.

\textbf{Censored quantile regression and NNs.} Only a small amount of work has been done in this area. The `\Verb"qrnn"' R package \citep{Alex2019} offers the ability to train NNs under fixed-value left censoring, adopting an idea from the linear setting \citep{Friederichs2007, Cannon2011}. 
\cite{Huttel2022} explore our objective but assume censoring times, $c_i$, are available for both censored and uncensored data points (i.e. the censoring distribution is known). We do not require this assumption.
DeepQuantReg \citep{Jia2022} tackles the same objective as this paper. They showed that improvements in quantile estimation can be obtained relative to naive methods. Their work differs from ours significantly -- they base their method around an estimator for the median from \cite{huang2005}, requiring an assumption that the censoring distribution is independent of covariates. 

\section{Experiments}
\label{sec_experiments}



This section empirically investigates several questions. Q1) \textit{How does the proposed CQRNN method compare with existing methods?} This is done qualitatively on synthetic 1D functions in Section \ref{sec_exp_qualitative_1D} and quantitatively on synthetic and real datasets in Section \ref{sec_exp_benchmarking}. Q2) \textit{How does the sequential grid algorithm compare to the CQRNN algorithm, both in terms of predictive accuracy and efficiency?} Explored in Section \ref{sec_exp_comparison_neo} Q3) \textit{How is the CQRNN algorithm affected by its hyperparameters?} We investigate the impact of grid fidelity and $y^*$ in Section \ref{sec_exp_hyperparam}.

All experiments use fully-connected NNs with two hidden layers of 100 neurons, except for SurvMNIST, when three convolutional layers are used. 
Grid size $M$ is set to 5, 9 or 19 depending on dataset size.
Appendix \ref{sec_app_exp_details} contains further details on hyperparameter settings, metrics, and datasets. Appendix \ref{sec_app_further_results} presents further results.

\looseness=-1
\textbf{Datasets.} We use three types of dataset. Type 1) \textit{Synthetic target data with synthetic censoring.} Type 2) \textit{Real target data with synthetic censoring.} Type 3) \textit{Real target data with real censoring.}
Whilst type 3 captures the kind of datasets we care about most, evaluation of quantiles is challenging since in real-world survival data the target conditional quantiles are not obtainable even at test time \citep{2017lipeng}.
In contrast, type 1 offers access to the ground truth quantiles for clean evaluation, but properties of these datasets may be less realistic.
We introduce type 2 as a middle ground, which takes a real-world dataset without censoring, and synthetically overlays a censoring distribution to create training data. At test time, we have the option of not applying the censoring, providing samples from the clean target distribution, allowing clearer evaluation.
We summarise all datasets in Table \ref{tab_dataset_info}. Full details are in Appendix \ref{sec_app_dataset_details}. Our experiments exceed the number and variety of datasets used in popular recent works \citep{Goldstein2020, Zhong2021}




\textbf{Metrics.} 
Our objective is to measure how closely a model's predicted quantiles match those of the ground truth target distribution. 
We favour different metrics for each dataset type.


For type 1 datasets, since targets are generated synthetically it is possible to analytically compute the ground truth target quantile for an input $\x_i$, which we denote $y_{i,\tau}$. We compute the mean squared error (MSE) between the predictions and the ground truths across three quantiles, $\tau \in [0.1,0.5,0.9]$, (we ensure these are always present in $\operatorname{grid}_\tau$). This is our first-choice metric when available. $\text{True quantile MSE (TQMSE)} \coloneqq \frac{1}{N} \sum_{\tau \in [0.1,0.5,0.9]} \sum_{i=1}^N (\hat{y}_{i,\tau} - y_{i,\tau})^2.$


In type 2 datsets \textit{samples} from the uncensored target distribution are available but not the synthetic generating function. Our preferred metric is the checkmark loss across the three quantiles. $\text{Uncensored quantile loss (UQL)} \coloneqq \frac{1}{N} \sum_{\tau \in [0.1,0.5,0.9]} \sum_{i=1}^N \rho_\tau(y_i,\hat{y}_{i,\tau})$. 

For datasets of type 3, we use two metrics. The concordance index (C-index) is computed using the median ($\tau=0.5$). 
But this may not reveal anything about systematic bias of different models, nor about other quantiles $\tau \neq 0.5$. 
We secondly use censored D-Calibration (CensDCal) \citep{Haider2020} which measures whether the empirical proportion of datapoints falling between pairs of consecutive quantiles, matches the target proportion, $\tau_{j+1} - \tau_j$. MSE of the deviation is then computed. 
There is also an uncensored version (UnDCal), which we can compute for type 1 \& 2 datasets.
Appendix \ref{sec_app_exp_details_metrics} gives further details.

\textbf{Baselines.} We compare against three methods that can be used to predict the quantiles of a target distribution using a NN.
\textbf{Excl. Censor} -- a method that naively excludes censored datapoints from the training data, optimising the loss in Eq. \ref{eq_loss_checkmark}. 
\textbf{DeepQuantReg} -- the only existing method in the literature proposing explicit output of quantiles from a NN on censored data \citep{Jia2022}. 
\textbf{LogNorm MLE} -- A NN outputting parameters of a lognormal distribution, that’s trained via maximum likelihood estimation (MLE) (details in Appendix \ref{sec_app_exp_details_baselines}), and is a standard baseline in related work (e.g. \citep{Avati2019, Goldstein2020}) since this distribution often suits properties of real-world time-to-event survival data \citep{kleinbaum2012}.


\subsection{Qualitative 1D Analysis}
\label{sec_exp_qualitative_1D}

Figure \ref{fig_1D_plots_mini} visualises the quantiles predicted by CQRNN and baseline methods for 1D datasets (Table \ref{tab_dataset_info} describes dataset functions, Figure \ref{fig_1D_plots} visualises further 1D datasets). Each dataset contains 500 datapoints drawn from, $\x \sim \mathcal{U}(0,2)$ and $\operatorname{grid}_\tau \in \{0.1,0.3,0.5,0.7,0.9\}$.
CQRNN usually learns quantiles that are closer to the ground truth than baseline methods, particularly at higher quantiles. Excl. censor systemically makes underpredictions, and this worsens at higher quantiles since larger values of $y$ are more likely to be censored and excluded. DeepQuantReg avoids this systemic underprediction, but appears to introduce bias of its own. LogNorm MLE provides variable results -- on Norm uniform it fails with excessive variance (an issue also observed by \cite{Avati2019}), but on LogNorm, since the target distribution matches the distribution output by the NN, it performs well.

\begin{table}[t]
\centering
\caption{Performance of our proposed CQRNN algorithm (Algorithm \ref{alg_cqrnn_loss}) compared to the sequential grid method for NNs (Algorithm \ref{alg_seq_grid}), over 200 random seeds.}
\resizebox{0.95\textwidth}{!}{\begin{tabular}{lrrrrrrrccc}

\toprule
   Dataset &  Number & Training time &  Test time &  Parameter & TQMSE difference  & CQRNN is & Seq. grid is & No statistical \\
  
    & quantiles & speed up &  speed up & saving &  Seq. grid $-$ CQRNN &  sig. better & sig. better & significant \\
    & & & & & 95\% conf. interval & in TQMSE? & in TQMSE? & difference \\
\midrule
  Norm linear &           9 &   14.7$\times$ &   11.2$\times$ &  8.3$\times$ &          -2.189 $\pm$ 0.245 & \checkmark & & \\
 Norm non-lin &           9 &   12.4$\times$ &   9.6$\times$ &  8.3$\times$ &         -0.003 $\pm$ 0.001 & \checkmark & & \\
  Exponential &           9 &   12.4$\times$ &   8.0$\times$ &  8.3$\times$ &           -0.093 $\pm$ 0.176 &  & & \checkmark\\
      Weibull &           9 &   12.5$\times$ &   8.7$\times$ &  8.3$\times$ &         -0.014 $\pm$ 0.016 &  & & \checkmark \\
      LogNorm &           9 &   12.9$\times$ &   8.8$\times$ &  8.3$\times$ &         -0.039 $\pm$ 0.028 & \checkmark & & \\
 Norm uniform &           9 &   12.6$\times$ &   8.1$\times$ &  8.3$\times$ &          0.175 $\pm$ 0.033 & & \checkmark  & \\
   Norm heavy &           19 &  31.0$\times$ &        18.5$\times$ &        16.2$\times$ &           0.108 $\pm$ 0.211 & & & \checkmark\\
    Norm med. &           19 &  34.1$\times$ &        17.9$\times$ &        16.2$\times$ &          -0.046 $\pm$ 0.004 & \checkmark & & \\
   Norm light &           19 &  31.7$\times$ &            19.2$\times$ &     16.2$\times$ &         -0.035 $\pm$ 0.003 & \checkmark & & \\
    Norm same &           19 &  34.3$\times$ &            21.2$\times$ &     16.2$\times$ &         -0.390 $\pm$ 0.049 & \checkmark & & \\
LogNorm heavy &           19 &  33.6$\times$ &        19.9$\times$ &        16.2$\times$ &          0.005 $\pm$ 0.001 & & \checkmark &  \\
 LogNorm med. &           19 &  31.3$\times$ &        13.1$\times$ &        16.2$\times$ &          0.019 $\pm$ 0.002 & & \checkmark &  \\
LogNorm light &           19 &  31.3$\times$ &        20.2$\times$ &        16.2$\times$ &         -0.033 $\pm$ 0.004 & \checkmark & & \\
 LogNorm same &           19 &  29.8$\times$ &        18.4$\times$ &        16.2$\times$ &         -0.371 $\pm$ 0.045 & \checkmark & & \\
\midrule
 Total: & & & & & & 8/14 & 3/14 & 3/14 \\
\bottomrule

\end{tabular}}
\label{tab_grid_cqrnn_compare}
\end{table}


\begin{wrapfigure}{R}{0.5\textwidth} 
\begin{center}
\vspace{-0.4in}

Type 1 datasets 

\includegraphics[width=0.49\columnwidth]{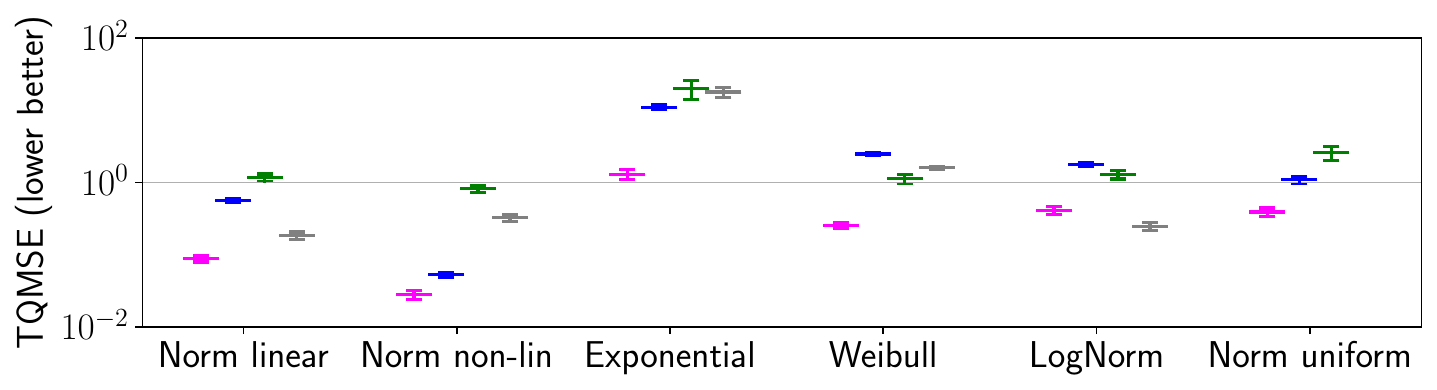}
\put(-0,8){\transparent{1.}\includegraphics[width=0.13\columnwidth]{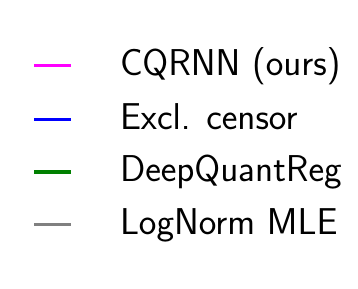}}


\includegraphics[width=0.49\columnwidth]{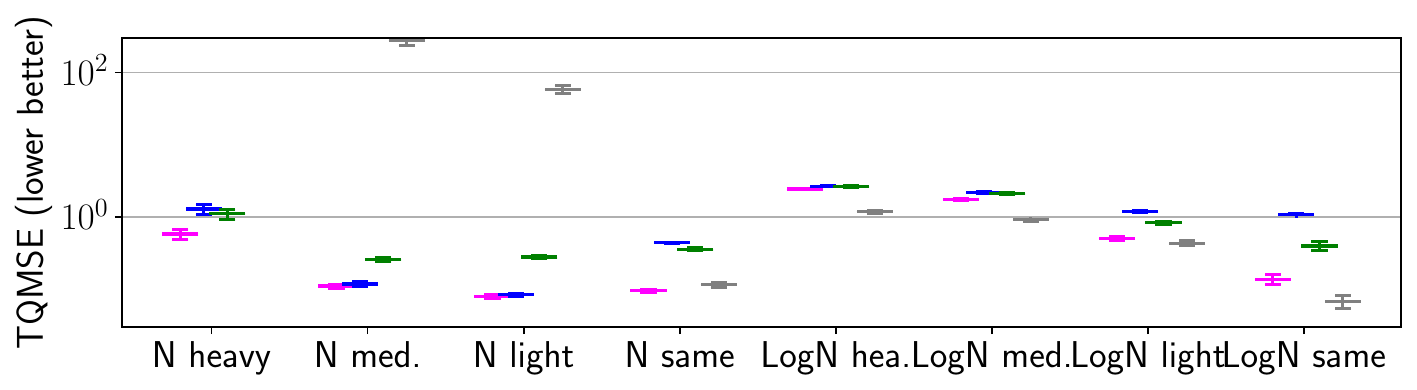}

Type 2 datasets

\includegraphics[width=0.49\columnwidth]{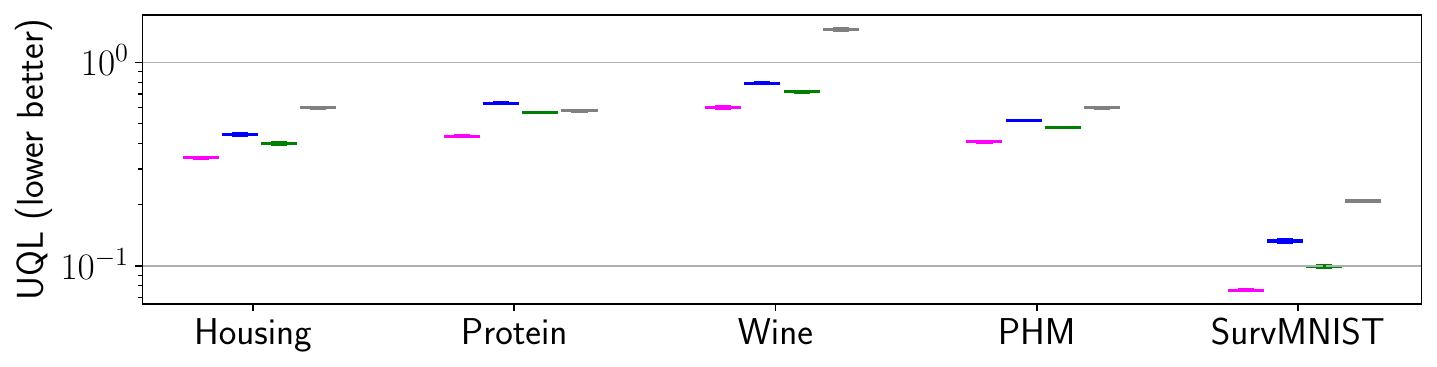}

Type 3 datasets

\includegraphics[width=0.49\columnwidth]{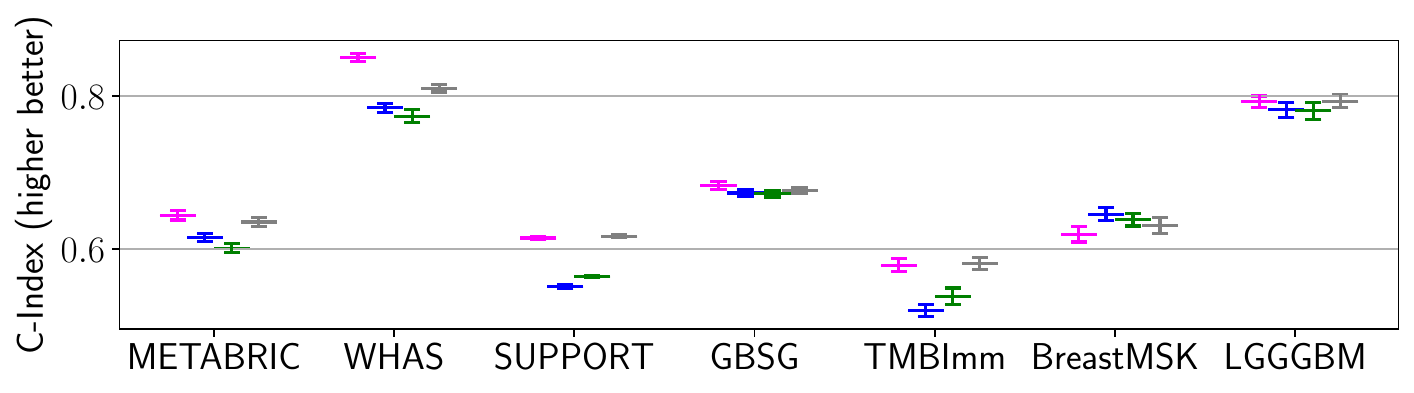}

\includegraphics[width=0.49\columnwidth]{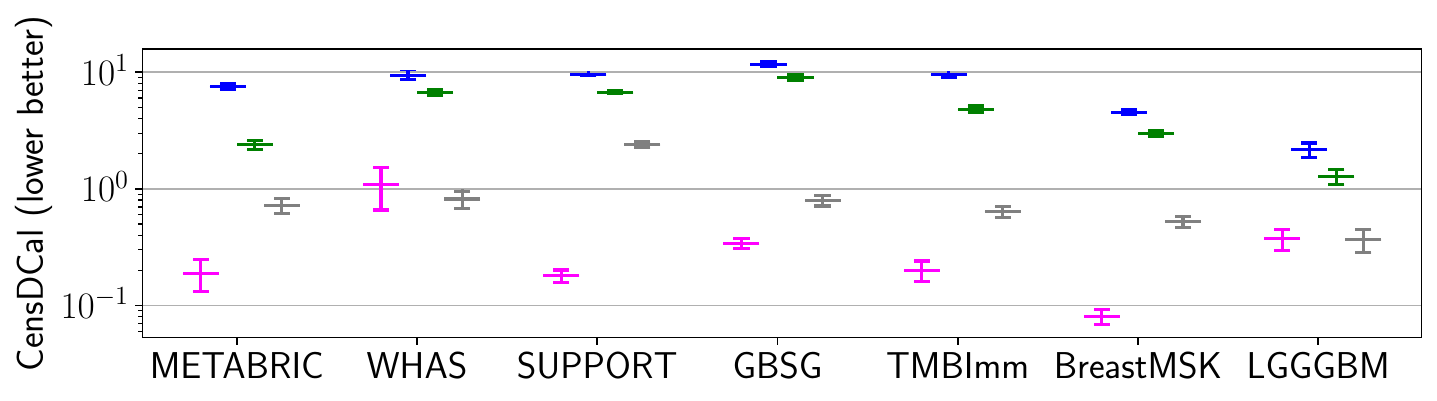}

\caption{Main benchmarking results. Mean $\pm$ one standard error over ten runs.} 
\label{fig_metric_plots}
\end{center}
\vskip -0.6in
\end{wrapfigure}
\subsection{Benchmarking}
\label{sec_exp_benchmarking}

Our main experiment benchmarks CQRNN against all baselines across a wide variety of datasets, covering various domains, sizes, dimensionalities and censoring proportions (see Table \ref{tab_dataset_info}). 
Some hyperparameter tuning was performed for each method (Appendix \ref{sec_app_exp_details_hyperparam}).
Figure \ref{fig_metric_plots} plots the preferred metric for each dataset type, though patterns were consistent across metrics -- Table \ref{tab_full_resultsa} provides a full breakdown.

For type 1 (synthetic) datasets, the quantitative results follow our qualitative observations, with CQRNN producing lowest TQMSE on all but the datasets generated from a Log Normal distribution, when it sometimes placed second behind LogNorm MLE.

CQRNN also produces the lowest UQL on all datasets of type 2. LogNorm MLE performs poorly on these, presumably since they are not typical time-to-event datasets which are well captured by the Log Normal distribution. 

For type 3 datasets, CQRNN and LogNorm MLE usually perform best in terms of C-index, with error bars tending to overlap. In CensDCal, CQRNN consistently perform best, matched only by LogNorm MLE on two datasets, when error bars overlap. 
Excl. censor is the weakest method, with DeepQuantReg midway between it and CQRNN.



\subsection{Comparison of Sequential Grid and CQRNN}
\label{sec_exp_comparison_neo}

We ran a head-to-head comparison of CQRNN and the sequential grid algorithm on all type 1 synthetic datasets, and all type 3 real datasets.
Table \ref{tab_grid_cqrnn_compare} compares on type 1 datasets, along with statistical tests of significance for TQMSE differences (200 random seeds). Appendix Table \ref{tab_grid_cqrnn_compare_real} provides results for the type 3 datsets, reporting C-index and CensDCal (50 random seeds). Significance tests are described in Appendix \ref{sec_app_exp_details_hyperparam}.

CQRNN delivers benefits in speed up both at training and test time, running an order of magnitude faster than the sequential grid algorithm. CQRNN also dramatically reduces model size. The magnitude of these benefits is largely determined by the size of the quantile grid, $M$, which is explicit in the sequential algorithm's time and space complexity, $\mathcal{O}(M)$, but largely avoided in CQRNN.

\looseness=-1
Differences in the quality of quantiles of the two algorithms is usually slight (after-all both leverage the same estimator), though CQRNN shows statistically significant gains on 8/14 type 1 datasets in terms of TQMSE, 4/7 type 3 datsets in terms of C-Index, and 7/7 type 3 datasets in terms of CensDCal. We hypothesise that having a single NN output all quantiles provides a helpful inductive bias, encouraging similarity between adjacent quantiles, that's not present in independently trained NNs.







\subsection{Hyperparameter Investigation}
\label{sec_exp_hyperparam}


The CQRNN algorithm includes two hyperparameters -- the grid of quantiles, $\operatorname{grid}_\tau$, and the pseudo y value, $y^*$. We ran ablations to empirically investigate the effect of these, as well as remedies for the `crossing-quantile' problem, and a comparison of partial vs. full optimisation. Here we summarise our findings, Appendix \ref{sec_app_further_results_hyper} provides full details.

We tested grid sizes, $M \in \{9,19,39\}$, on several type 1 datasets.
In general, a finer grid (larger $M$) is slightly beneficial, though there is variance between datasets and gains are sometimes only seen for larger datasets (>5,000 datapoints). 
Selection of $y^*$ requires some care for CQRNN. We defined it in terms of the maximum $y$ value in the training set, $y^* = c_{y^*} \max_i y_i$, for a hyperparemeter, $c_{y^*}>1$. Performance can be improved by tuning $c_{y^*}$ for each dataset, but using $c_{y^*}=1.2$ provided consistently reasonable results -- this was the value used in our benchmarking experiments. 
We trialled two methods for combating the crossing-quantile problem. 1) Adding a crossing loss penalty. 2) Constraining the NN architecture to output quantiles adding to the previous prediction. Neither of these methods significantly affected performance. 





\section{Discussion \& Conclusion}
\label{sec_discuss_conclude}

This paper has taken a popular idea from survival analysis, Portnoy's censored quantile regression estimator, and shown how it can be efficiently combined with NNs in a new algorithm, CQRNN. We provided theoretical insight by interpreting it as a flavour of EM. 
Empirically the method outperformed existing approaches, consistently producing more accurate quantile estimates across a range of synthetic and real datasets. For example, across datasets of type 2 and 3, CQRNN was best calibrated in 10 out of 12 instances (see CensDCal in Table \ref{tab_full_resultsa}).

\textbf{Limitations.} Firstly, our theoretical results have not said that solutions will converge on a global optimum. 
Secondly, we have only tested CQRNN on a modest number of real-world datasets, many drawn from the biomedical domain. It's possible that some datasets may cause CQRNN problems, as we found on BreastMSK (Figure \ref{fig_metric_plots}). In particular, cases where higher quantiles are undefined should be handled with care.



\textbf{Conclusion.} To summarise, our work contributes toward unlocking the benefit that modern machine learning could bring to important domains such as healthcare and machinery prognostics. 
By outputting quantiles, CQRNN naturally communicates a measure of uncertainty in its predictions, which makes it particularly suitable to these high-stakes applications, and a valuable addition to the toolkit combining deep NNs with survival analysis.

\begin{table}[h!]
\centering
\caption{Summary of all datasets used.}
\resizebox{0.9\textwidth}{!}{\begin{tabular}{lrrrrcc}
\toprule
    
     \multicolumn{1}{l}{Dataset} &  \multicolumn{1}{c}{Feats} &  \multicolumn{1}{c}{Train}  &  \multicolumn{1}{c}{Test} &  Prop.  &  Target sampling & Censoring sampling  \\
     
     & & \multicolumn{1}{c}{data} & \multicolumn{1}{c}{data} & censored & distribution  & distribution  \\
\midrule
\multicolumn{7}{c}{\textbf{Type 1 -- Synthetic datasets with synthetic censoring}}\\
     Norm linear &        1 &      500 &    1000 &             0.20 &    $\mathcal{N}(2x+10,(x+1)^2)$ & $\mathcal{N}(4x+10,(0.8x+0.4)^2)$ \\
 Norm non-linear &        1 &      500 &    1000 &             0.24 &    $\mathcal{N}(x\sin(2x)+10,(0.5x+0.5)^2)$ & $\mathcal{N}(2x+10,2^2)$ \\
     Exponential &        1 &      500 &    1000 &             0.30 &    $\operatorname{Exp}(2x+4)$ & $\operatorname{Exp}(-3x+15)$ \\
         Weibull &        1 &      500 &    1000 &             0.22 &    $\operatorname{Weibull}(4x \sin(2(x-1))+10, 5)$ & $\operatorname{Weibull}(-3x+20, 5)$ \\
         LogNorm &        1 &      500 &    1000 &             0.21 &    $\operatorname{Lognorm}((x-1)^2, x^2)$ & $\mathcal{U}(0, 10)$ \\
    Norm uniform &        1 &      500 &    1000 &             0.62 &    $\mathcal{N}(2x \cos(2x)+13,(x^2+0.5)^2)$ & $\mathcal{U}(0, 18)$\\
      Norm heavy &        4 &     2000 &    1000 &             0.80 &   \footnotesize{$\mathcal{N}(3x_0 + x_1^2 - x_2^2 + 2\sin(x_2 x_3)+6,(x^2+0.5)^2)$} & $\mathcal{U}(0, 12)$ \\
      Norm med. &        4 &     2000 &    1000 &             0.49 &    --- \raisebox{-0.5ex}{''} --- & $\mathcal{U}(0, 20)$ \\
      Norm light &        4 &     2000 &    1000 &             0.25 &    --- \raisebox{-0.5ex}{''} --- & $\mathcal{U}(0, 40)$ \\
      Norm same &        4 &     2000 &    1000 &             0.50 &    --- \raisebox{-0.5ex}{''} --- & Equal to target dist.  \\
  LogNorm heavy &        8 &     4000 &    1000 &             0.75 &    $\operatorname{Lognorm}(\sum_i^8 \beta_i x_i, 1)/10$ & $\mathcal{U}(0, 0.4)$ \\
    LogNorm med. &        8 &     4000 &    1000 &             0.52 &    --- \raisebox{-0.5ex}{''} --- & $\mathcal{U}(0, 1.0)$\\
  LogNorm light &        8 &     4000 &    1000 &             0.23 &    --- \raisebox{-0.5ex}{''} --- & $\mathcal{U}(0, 3.5)$\\
    LogNorm same &        8 &     4000 &    1000 &             0.50 &    --- \raisebox{-0.5ex}{''} --- & Equal to target dist.  \\
    
        \multicolumn{7}{c}{ \textbf{Type 2 -- Real datasets with synthetic censoring} }\\
         
         Housing &        8 &    16512 &    4128 &             0.50 &    Real & $\mathcal{U}(0, c_1)$ \\
         Protein &        9 &    36584 &    9146 &             0.44 &    Real & $\mathcal{U}(0, c_2)$\\
            Wine &       11 &     5197 &    1300 &             0.69 &    Real & $\mathcal{U}(0, c_3)$\\
             PHM &       21 &    36734 &    9184 &             0.52 &    Real & $\mathcal{U}(0, c_4)$\\
      SurvMNIST &        28$\times$28 &    48000 &   12000 &             0.53  & One Gamma dist. per MNIST class & $\mathcal{U}(0, c_5)$ \\
       
        \multicolumn{7}{c}{ \textbf{Type 3 -- Real datasets with real censoring} }\\
        METABRIC &        9 &     1523 &     381 &             0.42 &    Real & Real \\
            WHAS &        6 &     1310 &     328 &             0.57 &    Real & Real\\
         SUPPORT &       14 &     7098 &    1775 &             0.32 &    Real & Real\\
            GBSG &        7 &     1785 &     447 &             0.42 &    Real & Real\\
      TMBImmuno &        3 &     1328 &     332 &             0.49 &    Real & Real\\
      BreastMSK &        5 &     1467 &     367 &             0.77 &    Real & Real\\
          LGGGBM &        5 &      510 &     128 &             0.60 &    Real & Real\\
\bottomrule
\end{tabular}}
\label{tab_dataset_info}
\end{table}

\vspace{-0.2in}

\section*{Acknowledgement}
\looseness=-1
This work was supported by the National Key Research and Development
Program of China (2020AAA0106302); NSF of China Projects (Nos. 62061136001, U19B2034, U1811461, U19A2081); Beijing NSF Project (No.
JQ19016); a grant from Tsinghua Institute for Guo Qiang; and the High Performance Computing Center, Tsinghua University. J.Z was also supported by the XPlorer Prize.

\FloatBarrier
\newpage

\bibliography{library.bib}


\FloatBarrier
\newpage

\section*{Checklist}

\begin{enumerate}

\item For all authors...
\begin{enumerate}
  \item Do the main claims made in the abstract and introduction accurately reflect the paper's contributions and scope?
    \answerYes{}
  \item Did you describe the limitations of your work?
    \answerYes{See Section \ref{sec_discuss_conclude}.}
  \item Did you discuss any potential negative societal impacts of your work?
    \answerNA{We don't foresee direct negative societal impacts from our algorithm.}
  \item Have you read the ethics review guidelines and ensured that your paper conforms to them?
    \answerYes{}
\end{enumerate}

\item If you are including theoretical results...
\begin{enumerate}
  \item Did you state the full set of assumptions of all theoretical results?
    \answerYes{}
        \item Did you include complete proofs of all theoretical results?
    \answerYes{See Appendix \ref{sec_app_theory}.}
\end{enumerate}

\item If you ran experiments...
\begin{enumerate}
  \item Did you include the code, data, and instructions needed to reproduce the main experimental results (either in the supplemental material or as a URL)?
    \answerYes{Code is available at \textcolor{blue}{\url{https://github.com/TeaPearce/Censored_Quantile_Regression_NN}}.}
  \item Did you specify all the training details (e.g., data splits, hyperparameters, how they were chosen)?
    \answerYes{See Appendix \ref{sec_app_exp_details_hyperparam}.}
        \item Did you report error bars (e.g., with respect to the random seed after running experiments multiple times)?
    \answerYes{Standard errors over ten random seeds are included in main results in Figure \ref{fig_metric_plots} and Table \ref{tab_full_resultsa}.}
        \item Did you include the total amount of compute and the type of resources used (e.g., type of GPUs, internal cluster, or cloud provider)?
    \answerYes{See Appendix \ref{sec_app_exp_details}.}
\end{enumerate}

\item If you are using existing assets (e.g., code, data, models) or curating/releasing new assets...
\begin{enumerate}
  \item If your work uses existing assets, did you cite the creators?
    \answerYes{Data details are given in Appendix \ref{sec_app_dataset_details}.}
  \item Did you mention the license of the assets?
    \answerYes{All datasets used are opensourced.}
  \item Did you include any new assets either in the supplemental material or as a URL?
    \answerYes{We don't create new assets. Links to datasets used are provided in Appendix \ref{sec_app_dataset_details}.}
  \item Did you discuss whether and how consent was obtained from people whose data you're using/curating?
    \answerNA{}
  \item Did you discuss whether the data you are using/curating contains personally identifiable information or offensive content?
    \answerYes{Data details are given in Appendix \ref{sec_app_dataset_details}.}
\end{enumerate}

\item If you used crowdsourcing or conducted research with human subjects...
\begin{enumerate}
  \item Did you include the full text of instructions given to participants and screenshots, if applicable?
    \answerNA{}
  \item Did you describe any potential participant risks, with links to Institutional Review Board (IRB) approvals, if applicable?
    \answerNA{}
  \item Did you include the estimated hourly wage paid to participants and the total amount spent on participant compensation?
    \answerNA{}
\end{enumerate}

\end{enumerate}

\FloatBarrier
\newpage
\appendix

\FloatBarrier
\newpage
\section{Analytical Results}
\label{sec_app_theory}

\subsection{Proofs}

\begin{theorem}
\label{theorem_likelihood}
Let the likelihood for each datapoint at each quantile be an asymmetric Laplace distribution with scale, $\lambda = \sqrt{\tau - \tau^2}$, and asymmetry, $k=\tau/\sqrt{\tau - \tau^2}$.
The negative log likelihood is,
\begin{align}
    -\log p(y |\x,\theta, \mathbf{w}, y^*) &= \sum_{\tau \in \text{grid}_\tau}  \mathcal{L}_\text{Port.}(\theta, y, \x, \tau, \mathbf{w}, y^*) + \operatorname{constant}.
\end{align}
\end{theorem}

\begin{proof}
Define the likelihood over all quantiles of interest, and split censored datapoints into two pseudo datapoints, one at the censoring location and one at the large pseudo value $y^*$, to give a weighted likelihood,
\begin{align}
    p(y |\x,\theta, \mathbf{w}, y^*) &= \prod_{\tau \in \text{grid}_\tau} p(y|\x,\theta, \mathbf{w}, y^*,\tau), \\
    \label{eq_prob_model}
    p(y|\x,\theta, \mathbf{w}, y^*,\tau) &= \prod_{i\in\mathcal{S}_\text{observed}}p(y_i|\x_i,\theta) 
    \prod_{j\in\mathcal{S}_\text{censored}}p(y_j|\x_j,\theta)^{w_j}p(y^*|\x_j,\theta)^{1-w_j} .
\end{align}



We write the asymmetric Laplace density with $\hat{y}_{j,\tau}$ as the location parameter and scale $\lambda$ and asymmetry $k$,
\begin{equation}
    f(y_{j}; \hat{y}_{j,\tau},\lambda,k) = \frac{\lambda}{k+1/k} 
    \begin{cases}
        \exp ((\lambda/k) (y_{j}-\hat{y}_{j,\tau}))   \;\text{ if } \; \hat{y}_{j,\tau}>y_{j}\\
        \exp (-\lambda k (y_{j}-\hat{y}_{j,\tau})) \;\;\, \text{ else}
    \end{cases} .
\end{equation}
Setting $\lambda = \sqrt{\tau - \tau^2}$ and $k=\tau/\sqrt{\tau - \tau^2}$ and rearranging,
\begin{align}
    f(y_{j}; \hat{y}_{j,\tau},\lambda,k) 
        &= (\tau - \tau^2)\exp (y_{j}-\hat{y}_{j,\tau})  (-\tau + \mathbb{I}[\hat{y}_{j,\tau}>y_{j}]) \\
        \label{eq_laplace_simplify}
        \log f(y_{j}; \hat{y}_{j,\tau},\lambda,k) 
        &= - \rho_\tau(y_{j}, \hat{y}_{j,\tau}) + \operatorname{constant}
\end{align}
Taking the logarithm of Eq. \ref{eq_prob_model} we have,
\begin{align}
    \log p(y |\x,\theta, \mathbf{w}, y^*) \nonumber  &=\\  \sum_{\tau \in \text{grid}_\tau} & \sum_{i\in\mathcal{S}_\text{observed}} \log p(y_i|x_i,\theta) + 
    \sum_{j\in\mathcal{S}_\text{censored}} \log p(y_j|x_j,\theta)^{w_j} + \log p(y^*|x_j,\theta)^{1-w_j} .
\end{align}
Eq. \ref{eq_laplace_simplify} may be substituted into this if all likelihoods of observed and censored pseudo datapoints are chosen to follow asymmetric Laplace distributions. Taking the negative then recovers the theorem's result.

\end{proof}

\begin{theorem}
\label{theorem_self_correct}
If $\hat{q}_j$ is underestimated, one iteration of the algorithm acts to \textit{increase} the quantile predictions, $\hat{y}_{j,\tau}$, by the same amount, or even higher, than if the weight had been correct. 
If $\hat{q}_j$ is overestimated, $\hat{y}_{j,\tau}$, is decreased by the same amount, or even lower, than with the correct weight.
\end{theorem}

\begin{remark}
Note that if a quantile $\hat{q}_j$ is underestimated, it's desirable to increase the quantiles relating to that datapoint, for input $\x_j$ and predictions $\hat{y}_{j,\tau}$. If a quantile $\hat{q}_j$ is instead overestimated, it's desirable to decrease the quantiles for input $\x_j$ and predictions $\hat{y}_{j,\tau}$.
We refer to this desired behaviour as `self-correcting'.
\end{remark}

\begin{proof}

Denote ${q}_j$ the true quantile that censored datapoint $j$ is censored in, and ${w}_j = (\tau-{q}_j)/(1-{q}_j)$ the corresponding true weight. We now consider one iteration of the algorithm in the case that the estimated quantile is underestimated.
\begin{enumerate}
    \item Model underpredicts the censored quantile, $\hat{q}_j = {q}_j - \epsilon$, for some $\epsilon>0$.
    \item The corresponding weight is also underestimated, $\hat{q}_j < {q}_j \implies \hat{w}_j<{w}_j$, shown in lemma \ref{lemma_weights_larger}.
    
    \item Lemma \ref{lemma_partial_deriv} shows that, for censored datapoint $j$, the gradient of Eq. \ref{eq_loss_reweight} wrt the quantile prediction $\hat{y}_{j,\tau}$ is, 
    \begin{align}
        \frac{\partial \mathcal{L}_\text{Port.}(\theta, \mathcal{D}, \tau, \mathbf{w}, y^*)}{\partial \hat{y}_{j,\tau}} = \begin{cases}
        - \tau &\text{ if, } \hat{y}_{j,\tau} < y_j \\
        w_j - \tau &\text{ if, } y_j \leq \hat{y}_{j,\tau} < y^* \\
        1 - \tau &\text{ if, } y^* \leq \hat{y}_{j,\tau} \\
        \end{cases}.
    \end{align}
    Hence, if $\hat{w}_j$ is underestimated it holds that, $\frac{\partial \mathcal{L}_\text{Port.}(\theta, \mathcal{D}, \tau, \hat{\mathbf{w}}, y^*)}{\partial \hat{y}_{j,\tau}} \leq \frac{\partial \mathcal{L}_\text{Port.}(\theta, \mathcal{D}, \tau, {\mathbf{w}}, y^*)}{\partial \hat{y}_{j,\tau}}$. 
    \item When this gradient is used for optimisation, this has the effect of increasing the quantile prediction, $\hat{y}_{j,\tau}$, by either the same amount, or even higher, than if the weight had been correct. 
    
\end{enumerate}




Similar (reversed) logic applies if the quantile is initially overestimated, $\hat{q}_j = {q}_j + \epsilon$, (e.g. lemma \ref{lemma_weights_smaller}) which encourages decreasing the quantile predictions, $\hat{y}_{j,\tau}$, by the same amount or lower than with correct weights. Hence, weight estimates will be improved in future iterations of the algorithm.
Note that this applies to all quantiles, $\tau \in \operatorname{grid}_\tau$.


\end{proof}

\begin{lemma}
\label{lemma_weights_larger}
Let, $\hat{q}_j = {q}_j - \epsilon$, and $\epsilon>0$.
It holds that,
$\hat{q}_j < {q}_j  \implies \hat{w}_j<{w}_j$, for, $\tau \in (0,1)$, and, $\hat{q}_j, {q}_j \in (0,1)$. 
\end{lemma}

\begin{proof}
From the definition of the weights, 
$ \hat{w}_j<{w}_j  \implies (\tau-{q}_j - \epsilon)/(1-{q}_j - \epsilon) < (\tau-{q}_j)/(1-{q}_j)$.

Let $a \coloneqq \tau-{q}_j$ and $b \coloneqq 1-{q}_j$. Note that, $a<b$, since by assumption, $\tau<1$.
We must show that,
\begin{align}
\frac{a-\epsilon}{b-\epsilon}&<\frac{a}{b}\\
    \frac{a-\epsilon}{b-\epsilon} \frac{b}{a} &< 1 \\
    \frac{ab-b\epsilon}{ab-a\epsilon} &< 1, \text{ which holds since, } a<b.
\end{align}
\end{proof}

\begin{lemma}
\label{lemma_weights_smaller}
Let, $\hat{q}_j = {q}_j + \epsilon$, and $\epsilon>0$.
It holds that,
$\hat{q}_j > {q}_j  \implies \hat{w}_j>{w}_j$, for, $\tau \in (0,1)$, and, $\hat{q}_j, {q}_j \in (0,1)$. 
\end{lemma}

\begin{proof}
This proof follows lemma \ref{lemma_weights_larger}.
We now have,
\begin{align}
    \frac{ab+b\epsilon}{ab+a\epsilon} &> 1, \text{ which holds since, } a<b.
\end{align}
\end{proof}

\begin{lemma}
\label{lemma_partial_deriv}
The partial derivative of Portnoy's loss wrt the predicted quantile, for censored datapoint $j$, is given by,
    \begin{align}
        \frac{\partial \mathcal{L}_\text{Port.}(\theta, \mathcal{D}, \tau, \mathbf{w}, y^*)}{\partial \hat{y}_{j,\tau}} = \begin{cases}
        - \tau &\text{ if, } \hat{y}_{j,\tau} < y_j \\
        w_j - \tau &\text{ if, } y_j \leq \hat{y}_{j,\tau} < y^* \\
        1 - \tau &\text{ if, } y^* \leq \hat{y}_{j,\tau} \\
        \end{cases}.
    \end{align}
\end{lemma}

\begin{proof}
Recalling that,
\begin{align}
    \rho_\tau(y_i, \hat{y}_{i,\tau}) = (y_i - \hat{y}_{i,\tau})\left(\tau - \mathbb{I} [ \hat{y}_{i,\tau}>y_i ] \right),
\end{align}
we have,
\begin{align}
\label{eq_rho_derivative}
    \frac{\partial \rho_\tau(y_i, \hat{y}_{i,\tau})}{\partial \hat{y}_{j,\tau}} = \mathbb{I} [ \hat{y}_{i,\tau}>y_i ] - \tau .
\end{align}
We are interested in the derivative for the censored portion of Portnoy's loss in Eq. \ref{eq_loss_reweight},
\begin{align}
  \frac{\partial w_j \rho_\tau(y_j, \hat{y}_{j,\tau}) +  (1-w_j) \rho_\tau(y^*, \hat{y}_{j,\tau}) }{\partial \hat{y}_{j,\tau}}.
\end{align}
Note that we always choose, $y^* > y_j $. Hence $\hat{y}_{j,\tau} < y_j \implies \hat{y}_{j,\tau} < y^*$, so there are three cases to consider. Case 1, $\hat{y}_{j,\tau} < y_j$, case 2, $y_j \leq \hat{y}_{j,\tau} < y^*$, case 3 $y^* \leq \hat{y}_{j,\tau}$.

\begin{align}
\intertext{For case 1,} 
  \frac{\partial w_j \rho_\tau(y_j, \hat{y}_{j,\tau}) +  (1-w_j) \rho_\tau(y^*, \hat{y}_{j,\tau}) }{\partial \hat{y}_{j,\tau}} &=  w_j (- \tau) + (1-w_j)(-\tau) =-\tau. \\
\intertext{For case 2,} 
  \frac{\partial w_j \rho_\tau(y_j, \hat{y}_{j,\tau}) +  (1-w_j) \rho_\tau(y^*, \hat{y}_{j,\tau}) }{\partial \hat{y}_{j,\tau}} &=  w_j (1 - \tau) + (1-w_j)(-\tau) = w_j - \tau. \\
\intertext{For case 3,} 
  \frac{\partial w_j \rho_\tau(y_j, \hat{y}_{j,\tau}) +  (1-w_j) \rho_\tau(y^*, \hat{y}_{j,\tau}) }{\partial \hat{y}_{j,\tau}} &=  w_j (1 - \tau) + (1-w_j)(1 -\tau) = 1 - \tau.
\end{align}
\end{proof}

\subsection{First Iteration of the Sequential Grid Algorithm}
\label{sec_app_theory_first_iter}


In this section we compare the procedure proposed by \cite{Portnoy2003} to find the first quantile predicted at, $\tau_0 \coloneqq \operatorname{grid}_\tau[0]$, with the procedure we propose in the sequential grid algorithm for NNs (Algorithm \ref{alg_seq_grid}). We show that these two procedures produce equivalent gradients.

\textbf{Portnoy's procedure.}
\cite{Portnoy2003} require that no censored datapoints lie below the first quantile, and propose deleting these from the dataset when this does occur, as follows.
\begin{enumerate}
    \item Set $\mathbf{w} = \mathbf{1}$ for all censored datapoints in dataset.
    \item Optimise $\mathcal{L}_\text{Port.}$ (Eq. \ref{eq_loss_reweight}) for $\tau_0$.
    \item Find all censored datapoints below $\tau_0$, $B \gets \{j \in \mathcal{S}_\text{censored} : y_j < \hat{y}_{j, \tau_0}\}$.
    \item If $B$ is empty then exit. 
    \item Exclude all elements of $B$ from the dataset and repeat.
\end{enumerate}
This optimisation could be repeated many times. We'd like to avoid this since training NNs on potentially large datasets can be costly.


\textbf{Sequential grid for NNs procedure.}
Algorithm \ref{alg_seq_grid} instead simply sets $\mathbf{q} = \mathbf{0}$, and optimises the first quantile once only. 
\begin{enumerate}
    \item Set $\mathbf{q} = \mathbf{0}$ for all censored datapoints in dataset.
    \item Optimise $\mathcal{L}_\text{Port.}$ (Eq. \ref{eq_loss_reweight}) for $\tau_0$.
\end{enumerate}

\textbf{Justification.}
We now justify why this is a reasonable approximation.
First note that when $\hat{q}_j=0$ we have, $\hat{w}_i = \frac{\tau-\hat{q}_i}{1-\hat{q}_i} = \tau$.
Using lemma \ref{lemma_partial_deriv} we can compare the gradients for each procedure.
\begin{align}
\intertext{For Portoy's procedure, when $\hat{w}_j=1$,}
    \frac{\partial \mathcal{L}_\text{Port.}(\theta, \mathcal{D}, \tau, \mathbf{\hat{w}}, y^*)}{\partial \hat{y}_{j,\tau}} &= \begin{cases}
    - \tau &\text{ if, } \hat{y}_{j,\tau} < y_j \\
    1 - \tau &\text{ if, } y_j \leq \hat{y}_{j,\tau} < y^* \implies \text{set to 0 in next iteration} \\
    1 - \tau &\text{ if, } y^* \leq \hat{y}_{j,\tau} \\
    \end{cases} .
    \\
\intertext{For Algorithm \ref{alg_seq_grid}, when $\hat{q}_j=0$,}
    \frac{\partial \mathcal{L}_\text{Port.}(\theta, \mathcal{D}, \tau, \mathbf{\hat{w}}, y^*)}{\partial \hat{y}_{j,\tau}} &= \begin{cases}
    - \tau &\text{ if, } \hat{y}_{j,\tau} < y_j \\
    0 &\text{ if, } y_j \leq \hat{y}_{j,\tau} < y^* \\
    1 - \tau &\text{ if, } y^* \leq \hat{y}_{j,\tau} \\
    \end{cases}   .
\end{align}
At first look, the gradients appear to differ in the case $y_j \leq \hat{y}_{j,\tau} < y^*$. But when a datapoint triggers this criteria in Portnoy's procedure, it will be excluded and the model retrained, in which case its gradient becomes 0. As such the gradients for a censored datapoint in both procedures are equivalent.





\FloatBarrier
\newpage
\section{Experimental Details}
\label{sec_app_exp_details}

This section provides further details about all experiments run. Our code base uses the $\operatorname{PyTorch}$ framework.
Hyperparameters in Appendix \ref{sec_app_exp_details_hyperparam}. 
Metrics in Appendix \ref{sec_app_exp_details_metrics}. 
Baselines in Appendix \ref{sec_app_exp_details_baselines}.
Datasets in Appendix \ref{sec_app_dataset_details}.

\textbf{Hardware.} We used an internal cluster for experiments, utilising machines with four GPUs and 14 CPU cores. Most of our datasets used fully-connected NNs, which were trained on CPU, while GPUs were used for the SurvMNIST experiments. 


\subsection{Full Hyperparameter Details}
\label{sec_app_exp_details_hyperparam}
Below we list hyperparameter settings used and where applicable the tuning protocols followed. 

\subsubsection{Qualitative 1D Analysis}
Section \ref{sec_exp_qualitative_1D} experiment.
All methods used the same optimisation procedure and NN architecture, without tuning.
\begin{itemize}
  \setlength{\itemsep}{0pt}
  \setlength{\parskip}{0pt}
  \setlength{\parsep}{0pt}
    \item Training dataset size: 500, where, $\x \sim \mathcal{U}(0,2)$
    \item Epochs: 100
    \item Optimiser: Adam
    \item Learning rate: 0.01 (decreased 70\% and 90\% of the way through training)
    \item Batch size: 128
    \item Weight decay: 0.0001
    \item NN architecture: Fully-connected, two hidden layers of 100 hidden nodes, GeLU activations
    \item $y^* = 1.2 \times \max_i y_i$
    \item $\operatorname{grid}_\tau \in \{ 0.1, 0.3, 0.5, 0.7, 0.9\}$
\end{itemize}

\subsubsection{Benchmarking}
Section \ref{sec_exp_benchmarking} experiment. Experiments were repeated over 10 random seeds.
Hyperparameter settings.
\begin{itemize}
  \setlength{\itemsep}{0pt}
  \setlength{\parskip}{0pt}
  \setlength{\parsep}{0pt}
    \item Training dataset size: various -- see Table \ref{tab_dataset_info}
    \item Test dataset size: various -- see Table \ref{tab_dataset_info}
    \item Epochs: $\in \{10, 20,50,100\}$
    \item Optimiser: Adam
    \item Learning rate: 0.01 for fc NN, 0.001 for CNN (decreased 70\% and 90\% of the way through training)
    \item Batch size: 128
    \item Weight decay: 0.0001
    \item Default NN architecture: Fully-connected, two hidden layers of 100 hidden nodes, ReLU activations 
    \item CNN architecture for SurvMNIST following \cite{Goldstein2020}: Conv2D[64, (5$\times$5)] $\to$ ReLU $\to$ Dropout(0.2) $\to$ AveragePool(2$\times$2) $\to$ Conv2D[128, (5$\times$5)] $\to$ ReLU $\to$ Dropout(0.2) $\to$ AveragePool(2$\times$2) $\to$ Conv2D[256, (2$\times$2)] $\to$ ReLU $\to$ Linear
    \item $y^* = 1.2 \times \max_i y_i$
    \item Grid size $M \in \{6, 10, 20\}$
    \item Dropout $\in \{\operatorname{True}, \operatorname{False}\}$
\end{itemize}

Tuning process for epochs and dropout:
\begin{itemize}
\item For the real type 2 and type 3 datasets, we tuned number of epochs $\in \{10, 20,50,100\}$ and dropout $\in \{\operatorname{True}, \operatorname{False}\}$ for each dataset for each method. We used three random splits as a validation (but not overlapping with the random seeds used in the final test run).
\item Epochs was fixed to 100 and dropout disabled for: Norm linear, Norm non-linear, Exponential, Weibull, LogNorm, Norm uniform. 
\item Epochs was fixed to 20 and dropout disabled for: Norm heavy, Norm medium, Norm light, Norm same. 
\item Epochs was fixed to 10 and dropout disabled for: LogNorm heavy, LogNorm medium, LogNorm light, LogNorm same.
\end{itemize}

We fixed the grid size according to an estimate of how densely the datapoints covered the input space (a rough consideration of dataset size and number of features):
\begin{itemize}
    \item Grid size $M=5$ for smaller datasets with more features: WHAS, SUPPORT, GBSG, TMBImmuno, BreastMSK, LGGGBM, METABRIC.
    \item Grid size $M=9$ for medium datasets or smaller datasets with less features: Norm linear, Norm non-linear, Exponential, Weibull, LogNorm, Norm uniform.
    \item Grid size $M=19$ for larger datasets or those with less features: Norm heavy, Norm medium, Norm light, Norm same, LogNorm heavy, LogNorm medium, LogNorm light, LogNorm same, Housing, Protein, Wine, PHM, SurvMNIST.
\end{itemize}

\subsubsection{Comparison of Sequential Grid and CQRNN}
Section \ref{sec_exp_comparison_neo} experiment. This was carried out under the same protocol as for the main benchmarking, but repeated over a larger number of seeds (200 for type 1 datasets, 50 for type 3 datasets).
To obtain 95\% confidence intervals, we compute the standard error of the difference between means, and multiply it by a two-sided t-statistic, with (number of seeds$-1$) degrees of freedom at the $\alpha = 0.05$ significance level. 
If zero falls within this confidence interval, the difference is deemed not significant.

\subsubsection{Hyperparameter Investigation}
Section \ref{sec_exp_hyperparam} experiment.
Hyperparameters are as for the main benchmarking except $M$ and $y^*$ were varied as stated in the text. Epochs were set via the formula epochs $= 500\times200/N$ ensuring the same number of gradient updates were made on each run. Experiments were repeated over 100 random seeds for the grid investigation, and ten random seeds for the $y^*$ investigation.

\subsection{Metrics}
\label{sec_app_exp_details_metrics}
This section briefly expands upon the metrics introduced in Section \ref{sec_experiments}. Table \ref{tab_dataset_metrics} summarises the availability on each metric for each dataset type. The computation of each is also detailed below.

\begin{table}[bt]
\centering
\caption{Availability of metrics for each dataset type.}
\resizebox{0.9\textwidth}{!}{\begin{tabular}{lccccccc}
\toprule
Dataset type & Target distribution & Censoring distribution & TQMSE & UQL & UnDCal & CensDCal & C-Index \\
\midrule
Type 1 & Synthetic & Synthetic & \cmark & \cmark & \cmark & \cmark & \cmark \\
Type 2 & Real & Synthetic & \xmark & \cmark & \cmark & \cmark & \cmark \\
Type 3 & Real & Real & \xmark & \xmark & \xmark & \cmark & \cmark \\
\bottomrule
\end{tabular}}
\label{tab_dataset_metrics}
\end{table}



\begin{align}
\text{True quantile MSE (TQMSE)} &\coloneqq \frac{1}{N} \sum_{\tau \in [0.1,0.5,0.9]} \sum_{i=1}^N (\hat{y}_{i,\tau} - y_{i,\tau})^2 \\
\text{Uncensored quantile loss (UQL)} &\coloneqq \frac{1}{N} \sum_{\tau \in [0.1,0.5,0.9]} \sum_{i=1}^N \rho_\tau(y_i,\hat{y}_{i,\tau}) \\
    \text{Uncensored D-Calibration (UnDCal)} &\coloneqq 100 \times \sum_{j = 1}^{M-1} \left( \left(\tau_{j+1} - \tau_j\right) - \frac{1}{N} \sum_{i=1}^N \mathbb{I} [\hat{y}_{i,\tau_j} < y_i \leq \hat{y}_{i,\tau_{j+1}}] \right)^2 \\
    \text{Censored D-Calibration (CensDCal)} &\coloneqq 100 \times \sum_{j = 1}^{M-1} \left( 
    \left(\tau_{j+1} - \tau_j\right) - 
        \frac{1}{N} \xi
    \right)^2 \\
\intertext{where, \cite{Goldstein2020} defines,}
    \xi  =   
        \sum_{i \in \mathcal{S}_\text{observed}}\mathbb{I} [\hat{y}_{i,\tau_j} < y_i \leq \hat{y}_{i,\tau_{j+1}}]  
        +& \sum_{i \in \mathcal{S}_\text{censored}}
        \frac{(\tau_{j+1}-q_i) \mathbb{I} [\hat{y}_{i,\tau_j} < y_i \leq \hat{y}_{i,\tau_{j+1}}]}{1-q_i}
        + \frac{(\tau_{j+1} - \tau_j)\mathbb{I}[q_i<\tau_j]}{1-q_i}.
\end{align}
We increase the magnitude of DCal metrics by $100 \times$ to make the numbers of similar order to TQMSE and UQL.


\subsection{Baselines}
\label{sec_app_exp_details_baselines}
This section provides some extra detail about the LogNorm MLE baseline. For this method, we use a NN with two outputs, representing the mean, $\mu$, and standard deviation, $\sigma$, of a Log Normal distribution, i.e. $\log y \sim \mathcal{N}(\hat{\mu}, \hat{\sigma}^2)$. We pass the output representing the standard deviation prediction through a $\operatorname{SoftPlus}$ to ensure it is always positive and differentiable.

The maximum likelihood estimation loss is then,
\begin{align}
    - \mathcal{L}_\text{MLE}(\theta, \mathcal{D}) \coloneqq \sum_{i \in \mathcal{S}_\text{observed}} \log p(y_i|\x_i, \theta) + \sum_{j \in \mathcal{S}_\text{censored}} \log (1 - \operatorname{CDF}(y_j|\x_j, \theta)),
\end{align}
where the likelihood and CDF follow the analytical expressions for the Log Normal distribution.
At evaluation time, we use the $\operatorname{SciPy}$ package to compute the quantiles from the predicted Log Normal distribution that correspond to those in $\operatorname{grid}_\tau$. This allows a like-for-like comparison with our other baselines.

\subsection{Dataset Details}
\label{sec_app_dataset_details}

This section provides further detail about the source of each dataset used.
All real-world datasets were taken from open-access repositories, and had already been anonymised.
For real-world datasets, we do not follow any previous test/train splits, rather we randomly shuffle the data for each run, selecting 80\% for training and 20\% for testing.



\subsubsection{Type 1 Datasets}

Table \ref{tab_dataset_info} details the generating target and censoring distributions used, as well as numbers of test and train datapoints. Inputs were always generated uniformly via, $\x \sim \mathcal{U}(0,2)^D$ for $D$ features.

\subsubsection{Type 2 Datasets}

Table \ref{tab_dataset_info} details dataset sizes and number of features. All datasets used a uniform censoring distribution, $c_i \sim \mathcal{U}(0,c)$, where $c$ was selected to be equal to the 90th percentile of the target distribution for SurvMNIST, and $c = 1.5 \max_i y_i$ for the other type 2 datasets.
Housing was sourced from \Verb'Scikit Learn' datasets, while Protein, Wine and PHM were sourced through \Verb'OpenML' \textcolor{blue}{\url{https://www.openml.org/}}.
\begin{itemize} 
    \item \textbf{Housing} Target is median house prices. 
    Retrieved from \textcolor{blue}{\url{https://scikit-learn.org/stable/datasets/real_world.html\#california-housing-dataset}}.
    \item \textbf{Protein} Target is RMSD.
    OpenML lookup ID is \Verb'physicochemical-protein'.
    \item \textbf{Wine} Target is quality of wine.
    OpenML lookup ID is \Verb'wine quality'.
    \item \textbf{PHM} (prognostics health management) target is breakdown time of simulated machines.
    OpenML lookup ID is \Verb'NASA PHM2008'.
\end{itemize}
Our final type 2 dataset is slightly different, since we don't use the provided labels directly.
\begin{itemize} 
    \item \textbf{SurvMNIST} appeared in \cite{Goldstein2020}, who adapted it from Sebastian Pölsterl's blog: \textcolor{blue}{\url{https://k-d-w.org/blog/2019/07/ survival-analysis-for-deep-learning/}}. It uses the standard MNIST dataset \textcolor{blue}{\url{http://yann.lecun.com/exdb/mnist/}}, but targets are drawn from a Gamma distribution, with different parameters per class. \cite{Goldstein2020} used a small variance fixed across classes, with means $\in [11.25, 2.25, 5.25, 5.0, 4.75, 8.0, 2.0, 11.0, 1.75, 10.75]$. For our purposes, we are most interested in how well methods capture variance, so we vary it, 
    variance $\in [0.1, 0.5, 0.1, 0.2, 0.2, 0.2, 0.3, 0.1, 0.4, 0.6]$.
\end{itemize}

\subsubsection{Type 3 Datasets}

We provide a brief overview of each dataset. 
Four datasets -- GBSG, METABRIC, SUPPORT, WHAS -- were all retrieved from \textcolor{blue}{\url{https://github.com/jaredleekatzman/DeepSurv/tree/master/experiments/data}}. \cite{Katzman2018} provides a detailed introduction to these datasets. The other three datasets -- TMBImmnuo, BreastMSK, LGGGBM -- were all sourced from the cBioPortal, \textcolor{blue}{\url{https://www.cbioportal.org/}}, for cancer genomics.
\begin{itemize} 
    \item \textbf{GBSG} (Rotterdam \& German Breast Cancer Study Group) requires prediction of survival time for breast cancer patients.
    \item \textbf{METABRIC} (Molecular Taxonomy of Breast Cancer International Consortium) requires prediction of survival time for breast cancer patients. 
    Covariates include expressions for four genes as well as clinical data.
    \item \textbf{SUPPORT} (Study to Understand Prognoses Preferences Outcomes and Risks of Treatment) requires prediction of survival time in seriously ill hospitalised patients. Covariates include demographic and basic diagnosis information.
    \item \textbf{WHAS} (Worcester Heart Attack Study) requires prediction of acute myocardial infraction survival. 
    \item \textbf{TMBImmuno} (Tumor Mutational Burden and Immunotherapy)
    requires prediction of survival time for patients with various cancer types using clinical data. 
    Covariates include age, sex, and number of mutations. 
    Retrieved from \textcolor{blue}{\url{https://www.cbioportal.org/study/clinicalData?id=tmb_mskcc_2018}}.
    \item \textbf{BreakMSK} 
    requires prediction of survival time for patients with breast cancer using tumour information. 
    Covariates include ER, HER, HR, mutation count, TMB.
    Retrieved from \textcolor{blue}{\url{https://www.cbioportal.org/study/clinicalData?id=breast_msk_2018}}.
    \item \textbf{LGGGBM.} 
    requires prediction of survival time for cancer patient from clinical data. 
    Covariates include age, sex, purity, mutation count, TMB.
    Retrieved from \textcolor{blue}{\url{https://www.cbioportal.org/study/clinicalData?id=lgggbm_tcga_pub}}
\end{itemize}





\FloatBarrier
\newpage
\section{Further Results}
\label{sec_app_further_results}

\begin{figure}[h!]
\begin{center}
\vspace{-0.02in}

\hspace{0.2in} \small CQRNN   \hspace{0.4in}  Sequential grid  \hspace{0.4in}  Excl. censor \hspace{0.4in}  DeepQuantReg \hspace{0.35in}  LogNorm MLE

\includegraphics[width=0.195\columnwidth]{images/1D_plots/1Dinput_Gaussian_uniform_cqrnn_v01.pdf}
\put(-85,0){\rotatebox{90}{\small Norm uniform}}
\includegraphics[width=0.195\columnwidth]{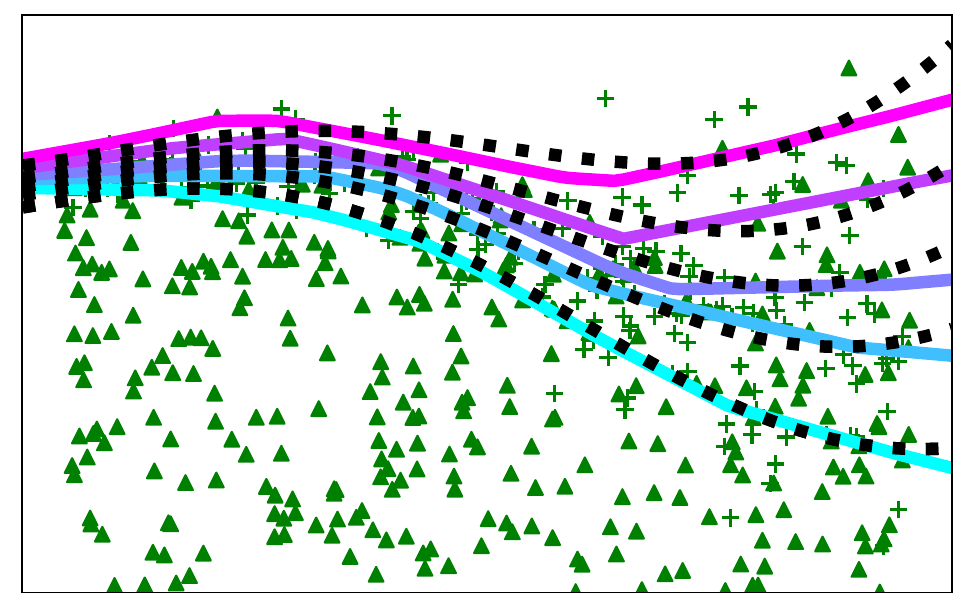}
\includegraphics[width=0.195\columnwidth]{images/1D_plots/1Dinput_Gaussian_uniform_excl_censor_v01.pdf}
\includegraphics[width=0.195\columnwidth]{images/1D_plots/1Dinput_Gaussian_uniform_deepquantreg_v01.pdf}
\includegraphics[width=0.195\columnwidth]{images/1D_plots/1Dinput_Gaussian_uniform_lognorm_v01.pdf}
\put(-0,4){\transparent{1.}\includegraphics[width=0.1\columnwidth]{images/1D_plots/1Dinput_legend_01.pdf}}

\includegraphics[width=0.195\columnwidth]{images/1D_plots/1Dinput_Gaussian_linear_cqrnn_v01.pdf}
\put(-85,0){\rotatebox{90}{\small Norm linear}}
\includegraphics[width=0.195\columnwidth]{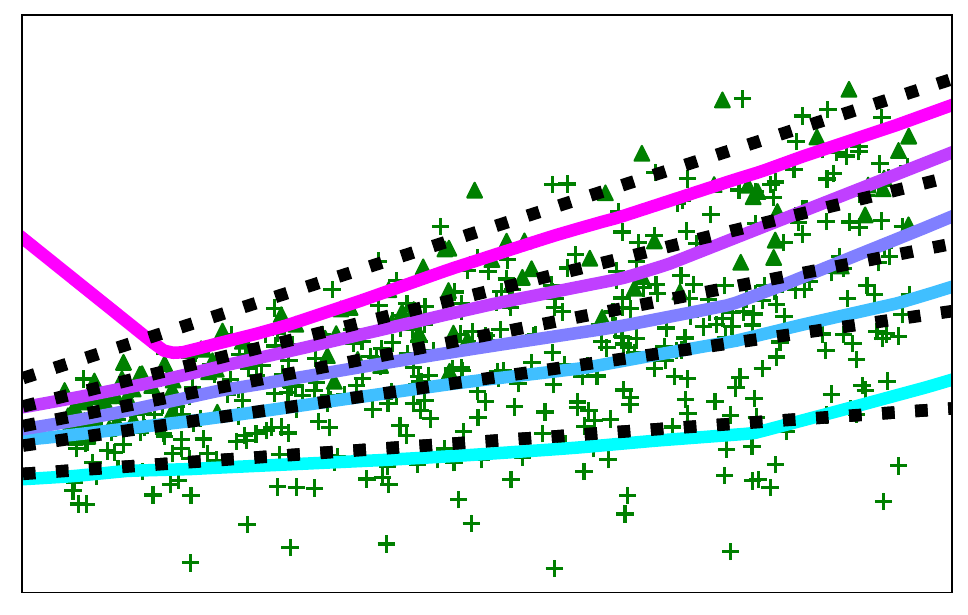}
\includegraphics[width=0.195\columnwidth]{images/1D_plots/1Dinput_Gaussian_linear_excl_censor_v01.pdf}
\includegraphics[width=0.195\columnwidth]{images/1D_plots/1Dinput_Gaussian_linear_deepquantreg_v01.pdf}
\includegraphics[width=0.195\columnwidth]{images/1D_plots/1Dinput_Gaussian_linear_lognorm_v01.pdf}

\includegraphics[width=0.195\columnwidth]{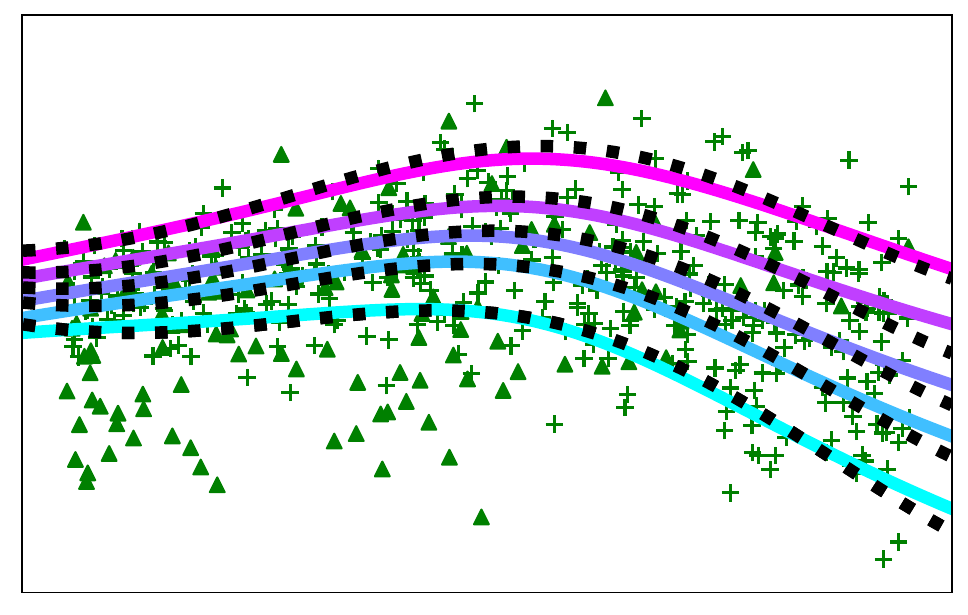}
\put(-85,0){\rotatebox{90}{\small Norm non-lin.}}
\includegraphics[width=0.195\columnwidth]{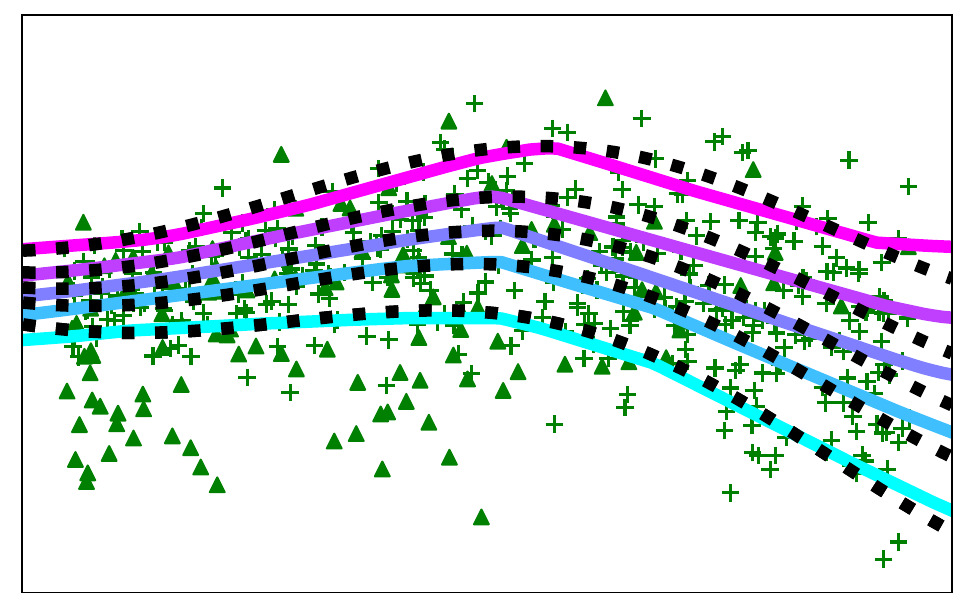}
\includegraphics[width=0.195\columnwidth]{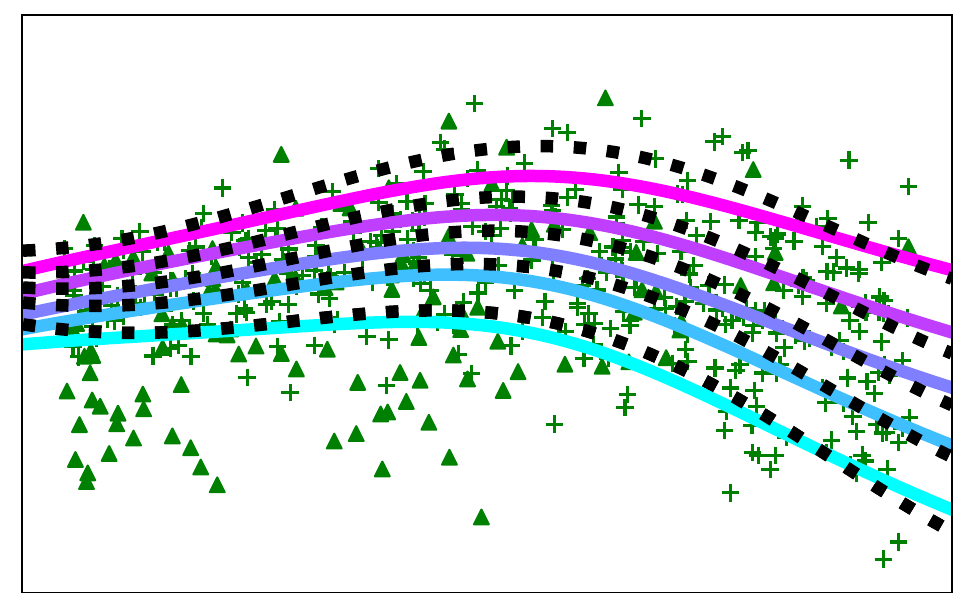}
\includegraphics[width=0.195\columnwidth]{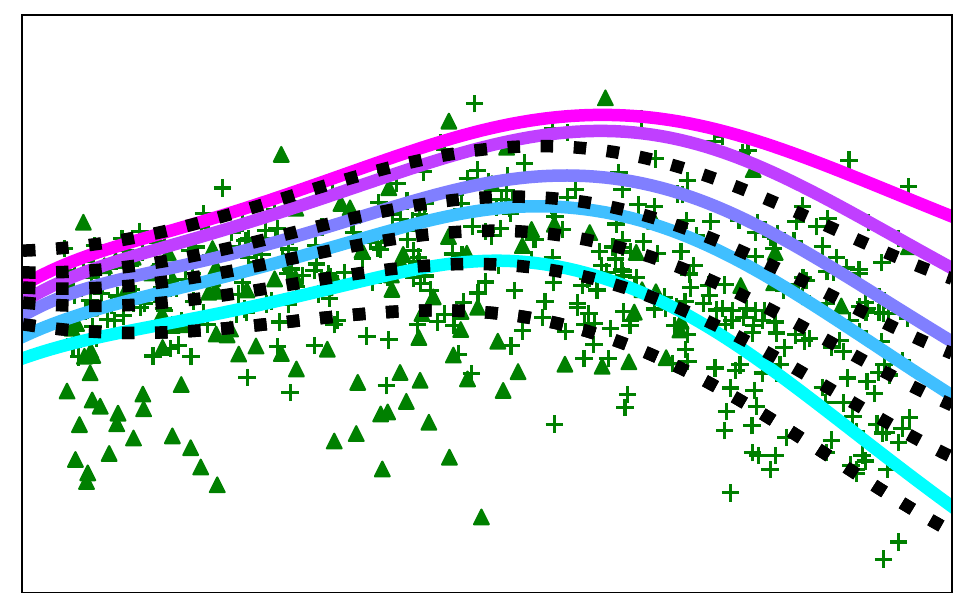}
\includegraphics[width=0.195\columnwidth]{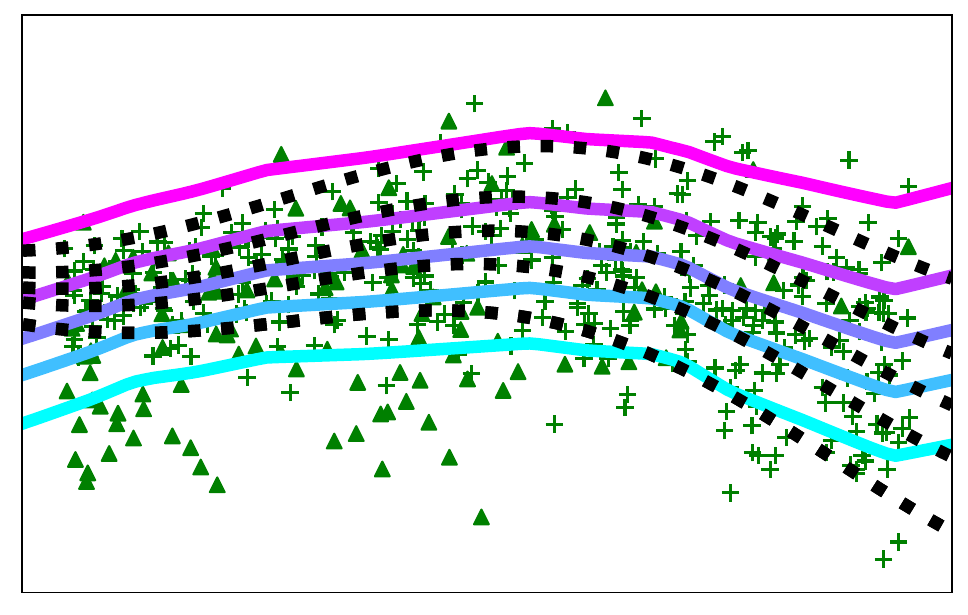}

\includegraphics[width=0.195\columnwidth]{images/1D_plots/1Dinput_Exponential_cqrnn_v01.pdf}
\put(-85,0){\rotatebox{90}{\small Exponential}}
\includegraphics[width=0.195\columnwidth]{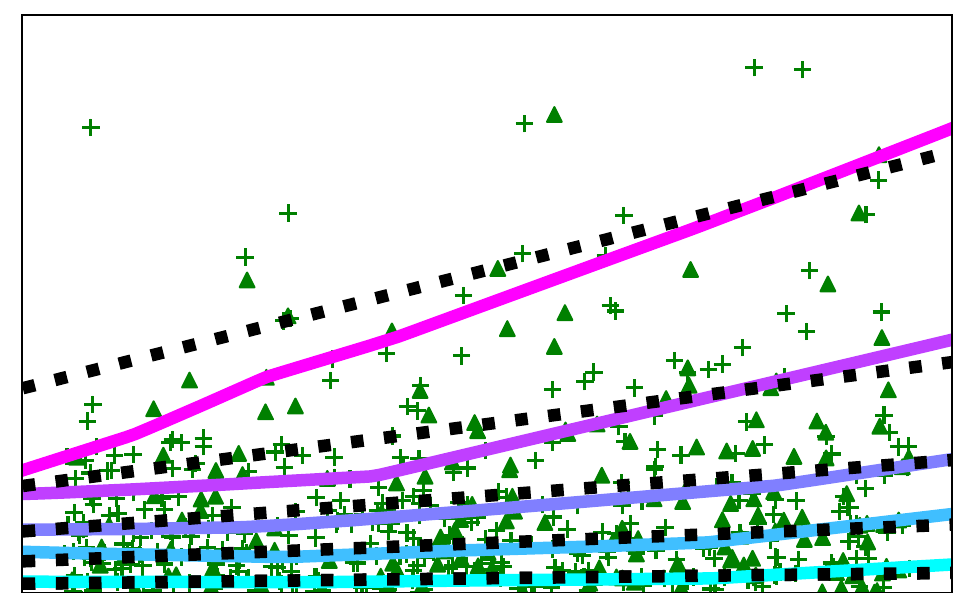}
\includegraphics[width=0.195\columnwidth]{images/1D_plots/1Dinput_Exponential_excl_censor_v01.pdf}
\includegraphics[width=0.195\columnwidth]{images/1D_plots/1Dinput_Exponential_deepquantreg_v01.pdf}
\includegraphics[width=0.195\columnwidth]{images/1D_plots/1Dinput_Exponential_lognorm_v01.pdf}

\includegraphics[width=0.195\columnwidth]{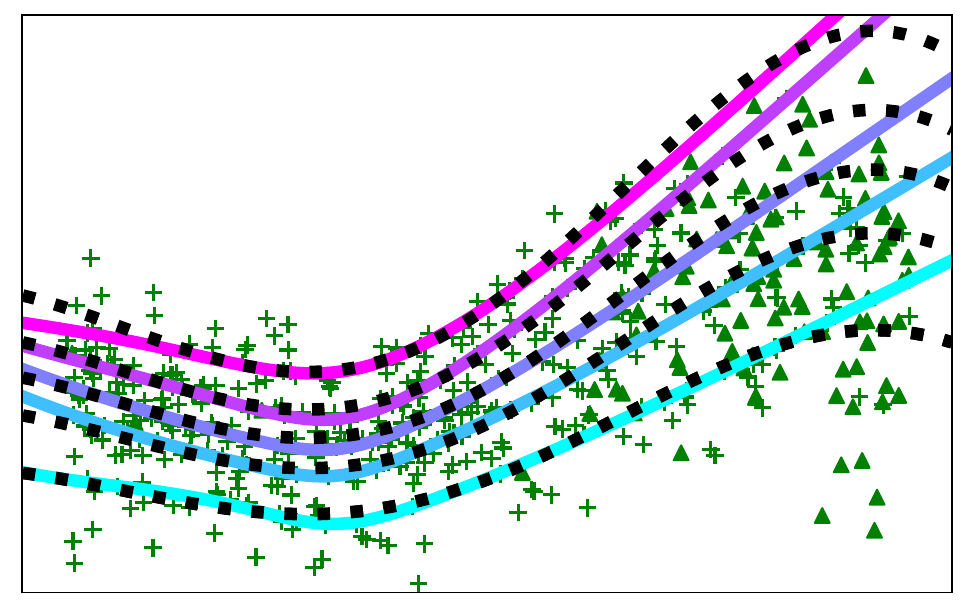}
\put(-85,0){\rotatebox{90}{\small Weibull}}
\includegraphics[width=0.195\columnwidth]{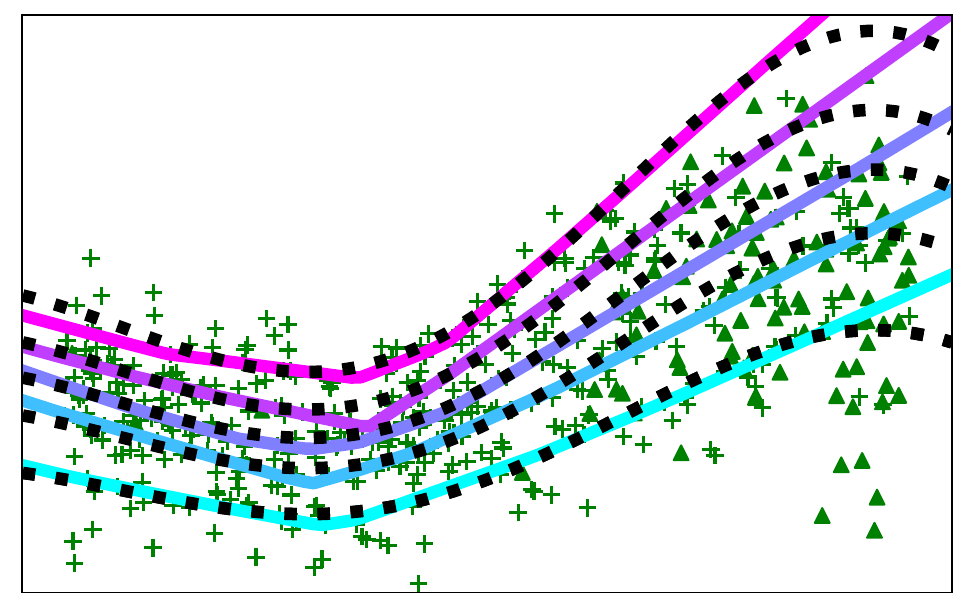}
\includegraphics[width=0.195\columnwidth]{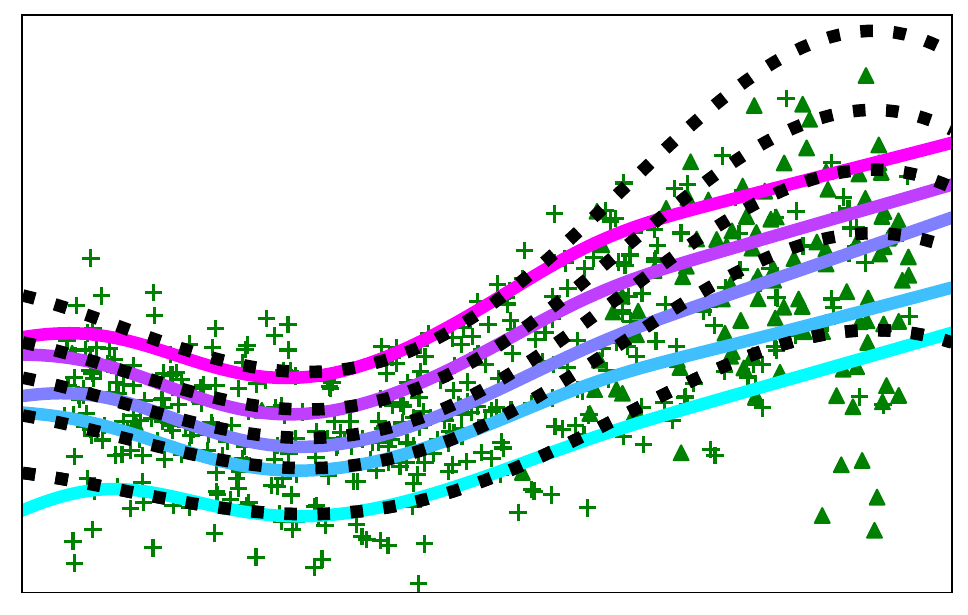}
\includegraphics[width=0.195\columnwidth]{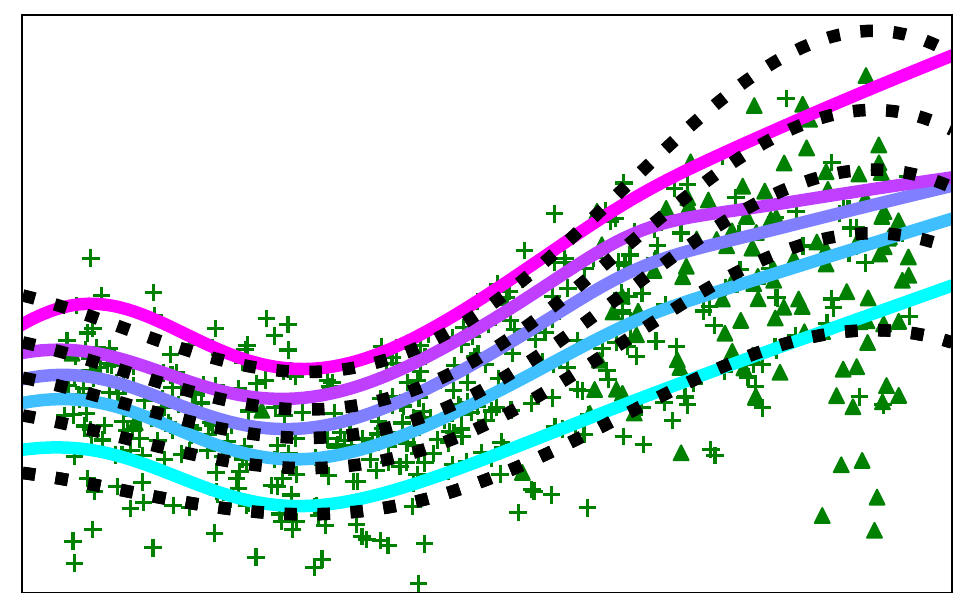}
\includegraphics[width=0.195\columnwidth]{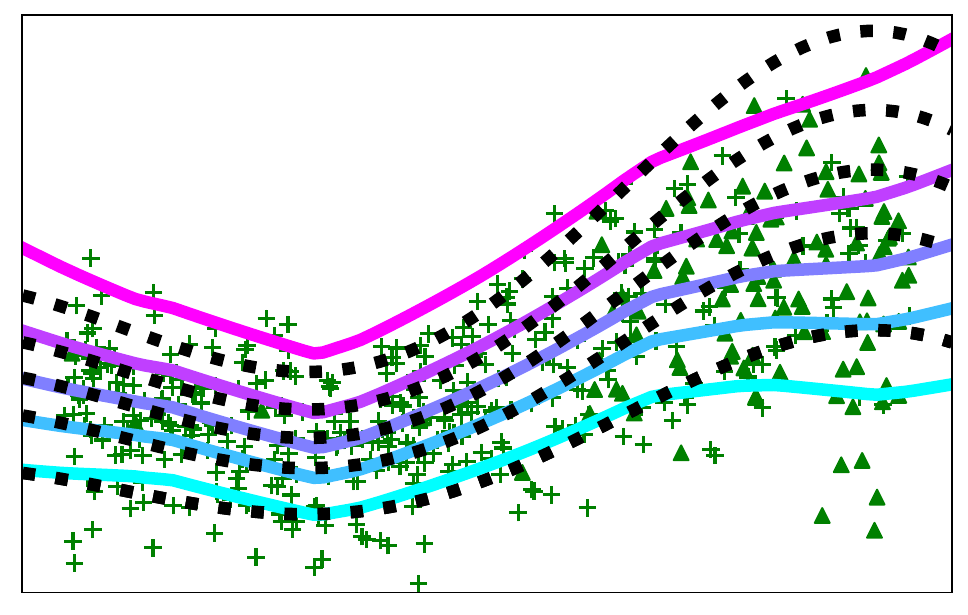}

\includegraphics[width=0.195\columnwidth]{images/1D_plots/1Dinput_LogNorm_cqrnn_v01.pdf}
\put(-85,0){\rotatebox{90}{\small LogNorm}}
\includegraphics[width=0.195\columnwidth]{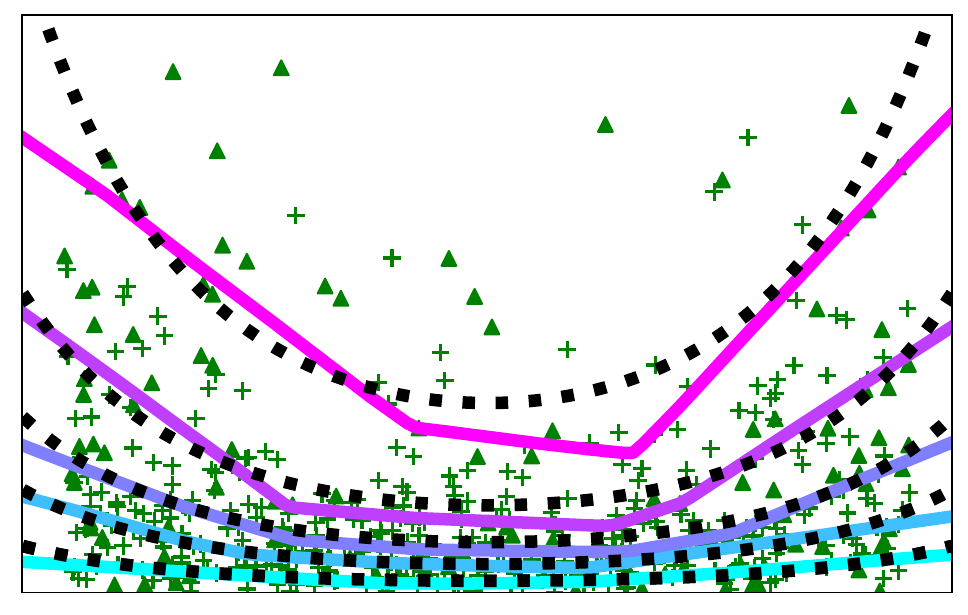}
\includegraphics[width=0.195\columnwidth]{images/1D_plots/1Dinput_LogNorm_excl_censor_v01.pdf}
\includegraphics[width=0.195\columnwidth]{images/1D_plots/1Dinput_LogNorm_deepquantreg_v01.pdf}
\includegraphics[width=0.195\columnwidth]{images/1D_plots/1Dinput_LogNorm_lognorm_v01.pdf}

\caption{This figure shows estimated quantiles (blue through pink) compared to ground truth quantiles (dashed black lines).
It shows 1D synthetic datasets of varying functions and noise distributions (rows), fitted by various methods (columns).
}
\label{fig_1D_plots}
\end{center}
\vskip -0.2in
\end{figure}

\begin{table}[ht]
\centering
\caption{Full results table for all datasets, methods and metrics. Mean $\pm$ 1 standard error for test set over 10 runs.} 
\resizebox{1.\textwidth}{!}{\begin{tabular}{lllllll}
\toprule
     Dataset &        Method & MSE to true quantile & Uncensored quantile loss & Uncensored D-Calibration & Concordance-Index & Censored D-Calibration  \\
     
      &     &  (lower better) & (lower better) & (lower better) & (higher better) & (lower better)  \\

\midrule
&\multicolumn{6}{c}{\textbf{Type 1 -- synthetic data, synthetic censoring}}\\

   Norm linear &            CQRNN &                      \textbf{0.088 $\pm$ 0.009} &                            \textbf{1.517 $\pm$0.01} &          0.327 $\pm$0.057 &                     \textbf{0.663 $\pm$0.002} &                         \textbf{0.211 $\pm$0.034} \\
   Norm linear &  Sequential grid &                      2.779 $\pm$ 0.506 &                            1.623 $\pm$0.02 &           \textbf{0.303 $\pm$0.06} &                     0.662 $\pm$0.002 &                         0.224 $\pm$0.024 \\
   Norm linear &     Excl. censor &                      0.566 $\pm$ 0.039 &                            1.62 $\pm$0.016 &          3.078 $\pm$0.244 &                     0.662 $\pm$0.002 &                         2.349 $\pm$0.102 \\
   Norm linear &     DeepQuantReg &                      1.172 $\pm$ 0.133 &                           1.695 $\pm$0.026 &           1.94 $\pm$0.192 &                     \textbf{0.663 $\pm$0.002} &                         1.241 $\pm$0.138 \\
   Norm linear &      LogNorm MLE &                      0.184 $\pm$ 0.021 &                           1.522 $\pm$0.008 &          0.315 $\pm$0.023 &                     0.662 $\pm$0.002 &                          0.232 $\pm$0.04 \\
   
   \midrule
   
  Norm non-lin &            CQRNN &                      0.028 $\pm$ 0.004 &                           \textbf{0.759 $\pm$0.005} &          \textbf{0.262 $\pm$0.044} &                     \textbf{0.674 $\pm$0.004} &                         \textbf{0.186 $\pm$0.023} \\
  Norm non-lin &  Sequential grid &                      \textbf{0.027 $\pm$ 0.003} &                           \textbf{0.759 $\pm$0.005} &          0.283 $\pm$0.034 &                     0.673 $\pm$0.004 &                          0.194 $\pm$0.03 \\
  Norm non-lin &     Excl. censor &                      0.053 $\pm$ 0.004 &                           0.768 $\pm$0.006 &          0.579 $\pm$0.088 &                     0.673 $\pm$0.004 &                         0.449 $\pm$0.054 \\
  Norm non-lin &     DeepQuantReg &                      0.818 $\pm$ 0.093 &                           0.998 $\pm$0.019 &          4.263 $\pm$0.234 &                     0.609 $\pm$0.023 &                         2.593 $\pm$0.206 \\
  Norm non-lin &      LogNorm MLE &                      0.323 $\pm$ 0.035 &                           0.824 $\pm$0.009 &          1.651 $\pm$0.118 &                     0.653 $\pm$0.005 &                         0.971 $\pm$0.077 \\
  
  \midrule
  
   Exponential &            CQRNN &                      \textbf{1.298 $\pm$ 0.201} &                            \textbf{4.057 $\pm$0.04} &          0.404 $\pm$0.035 &                     \textbf{0.559 $\pm$0.003} &                         0.248 $\pm$0.031 \\
   Exponential &  Sequential grid &                      1.702 $\pm$ 0.397 &                           4.061 $\pm$0.039 &          \textbf{0.391 $\pm$0.057} &                     0.558 $\pm$0.004 &                         \textbf{0.226 $\pm$0.025} \\
   Exponential &     Excl. censor &                      11.03 $\pm$ 0.772 &                           4.456 $\pm$0.039 &           2.537 $\pm$0.18 &                     0.535 $\pm$0.007 &                          1.513 $\pm$0.14 \\
   Exponential &     DeepQuantReg &                     20.046 $\pm$ 5.844 &                            4.675 $\pm$0.09 &          2.294 $\pm$0.252 &                     0.544 $\pm$0.004 &                         1.276 $\pm$0.158 \\
   Exponential &      LogNorm MLE &                     17.825 $\pm$ 2.831 &                           4.342 $\pm$0.058 &          1.056 $\pm$0.071 &                     0.552 $\pm$0.003 &                         0.401 $\pm$0.045 \\
     
  \midrule
  
       Weibull &            CQRNN &                      \textbf{0.255 $\pm$ 0.025} &                           1.879 $\pm$0.018 &           \textbf{0.33 $\pm$0.037} &                     \textbf{0.772 $\pm$0.002} &                         0.211 $\pm$0.023 \\
       Weibull &  Sequential grid &                      0.261 $\pm$ 0.018 &                           \textbf{1.877 $\pm$0.017} &          0.352 $\pm$0.037 &                     \textbf{0.772 $\pm$0.002} &                         \textbf{0.209 $\pm$0.023} \\
       Weibull &     Excl. censor &                      2.468 $\pm$ 0.118 &                            2.18 $\pm$0.021 &          1.942 $\pm$0.152 &                     0.769 $\pm$0.002 &                         0.834 $\pm$0.069 \\
       Weibull &     DeepQuantReg &                       1.116 $\pm$ 0.16 &                           2.005 $\pm$0.037 &          1.057 $\pm$0.159 &                     0.769 $\pm$0.002 &                         0.556 $\pm$0.126 \\
       Weibull &      LogNorm MLE &                      1.586 $\pm$ 0.081 &                            2.048 $\pm$0.02 &          0.603 $\pm$0.036 &                     0.771 $\pm$0.002 &                         0.478 $\pm$0.072 \\
         
  \midrule
  
       LogNorm &            CQRNN &                      0.411 $\pm$ 0.057 &                            1.716 $\pm$0.04 &          0.253 $\pm$0.025 &                      0.59 $\pm$0.004 &                         0.161 $\pm$0.015 \\
       LogNorm &  Sequential grid &                      0.328 $\pm$ 0.036 &                           1.716 $\pm$0.041 &           0.28 $\pm$0.033 &                     \textbf{0.591 $\pm$0.003} &                         0.193 $\pm$0.024 \\
       LogNorm &     Excl. censor &                      1.757 $\pm$ 0.114 &                           1.818 $\pm$0.051 &           1.46 $\pm$0.129 &                     0.588 $\pm$0.004 &                         1.119 $\pm$0.062 \\
       LogNorm &     DeepQuantReg &                      1.279 $\pm$ 0.168 &                            1.843 $\pm$0.04 &          1.924 $\pm$0.131 &                      0.58 $\pm$0.004 &                         1.279 $\pm$0.106 \\
       LogNorm &      LogNorm MLE &                      \textbf{0.247 $\pm$ 0.032} &                           \textbf{1.713 $\pm$0.041} &          \textbf{0.115 $\pm$0.014} &                     0.589 $\pm$0.004 &                         \textbf{0.151 $\pm$0.018} \\
         
  \midrule
  
  Norm uniform &            CQRNN &                      0.388 $\pm$ 0.054 &                           1.442 $\pm$0.022 &          2.219 $\pm$0.212 &                     0.789 $\pm$0.003 &                        \textbf{0.094 $\pm$0.008} \\
  Norm uniform &  Sequential grid &                     \textbf{ 0.188 $\pm$ 0.016} &                          \textbf{ 1.409 $\pm$0.017} &         \textbf{ 1.347 $\pm$0.136} &                     0.788 $\pm$0.003 &                          0.462 $\pm$0.02 \\
  Norm uniform &     Excl. censor &                      1.088 $\pm$ 0.112 &                           1.501 $\pm$0.018 &           1.557 $\pm$0.19 &                      \textbf{0.79 $\pm$0.003} &                         0.721 $\pm$0.066 \\
  Norm uniform &     DeepQuantReg &                       2.591 $\pm$ 0.56 &                           1.809 $\pm$0.084 &          4.397 $\pm$0.687 &                     0.598 $\pm$0.064 &                         0.418 $\pm$0.074 \\
  Norm uniform &      LogNorm MLE &                   939.992 $\pm$ 144.76 &                           7.727 $\pm$0.502 &         32.486 $\pm$1.405 &                      0.77 $\pm$0.006 &                         4.259 $\pm$0.144 \\
    
  \midrule
  
    Norm heavy &            CQRNN &                      \textbf{0.579 $\pm$ 0.089} &                           0.999 $\pm$0.041 &          \textbf{6.491 $\pm$1.094} &                     0.922 $\pm$0.002 &                         0.301 $\pm$0.067 \\
    Norm heavy &  Sequential grid &                      0.597 $\pm$ 0.081 &                           \textbf{0.996 $\pm$0.037} &          6.595 $\pm$1.406 &                     \textbf{0.923 $\pm$0.002} &                         {0.198 $\pm$0.014} \\
    Norm heavy &     Excl. censor &                      1.285 $\pm$ 0.205 &                           1.223 $\pm$0.069 &          8.206 $\pm$1.895 &                     \textbf{0.923 $\pm$0.002} &                         0.261 $\pm$0.021 \\
    Norm heavy &     DeepQuantReg &                      1.099 $\pm$ 0.175 &                            1.166 $\pm$0.06 &          6.681 $\pm$1.365 &                     \textbf{0.923 $\pm$0.002} &                         \textbf{0.191 $\pm$0.016} \\
    Norm heavy &      LogNorm MLE &                5568.886 $\pm$ 1737.861 &                          17.077 $\pm$2.442 &         21.283 $\pm$1.057 &                     0.852 $\pm$0.008 &                         1.113 $\pm$0.102 \\
      
  \midrule
  
     Norm med. &            CQRNN &                       \textbf{0.11 $\pm$ 0.007} &                           \textbf{0.789 $\pm$0.006} &          0.633 $\pm$0.138 &                     \textbf{0.896 $\pm$0.001} &                         {0.157 $\pm$0.054} \\
     Norm med. &  Sequential grid &                       0.16 $\pm$ 0.009 &                           0.799 $\pm$0.005 &          0.474 $\pm$0.046 &                     0.895 $\pm$0.001 &                         0.159 $\pm$0.011 \\
     Norm med. &     Excl. censor &                      0.117 $\pm$ 0.008 &                           0.792 $\pm$0.005 &           \textbf{0.247 $\pm$0.02} &                     \textbf{0.896 $\pm$0.001} &                        \textbf{0.136 $\pm$0.015} \\
     Norm med. &     DeepQuantReg &                      0.255 $\pm$ 0.016 &                           0.847 $\pm$0.008 &          0.944 $\pm$0.098 &                     0.892 $\pm$0.001 &                         0.232 $\pm$0.029 \\
     Norm med. &      LogNorm MLE &                   276.276 $\pm$ 41.622 &                           3.974 $\pm$0.228 &         17.969 $\pm$0.548 &                     0.865 $\pm$0.004 &                         4.612 $\pm$0.134 \\
       
  \midrule
  
    Norm light &            CQRNN &                      \textbf{0.079 $\pm$ 0.005} &                           \textbf{0.778 $\pm$0.005} &          0.173 $\pm$0.021 &                     \textbf{0.882 $\pm$0.001} &                         \textbf{0.084 $\pm$0.008} \\
    Norm light &  Sequential grid &                      0.117 $\pm$ 0.005 &                           0.784 $\pm$0.004 &          0.352 $\pm$0.027 &                     \textbf{0.882 $\pm$0.001} &                          0.19 $\pm$0.018 \\
    Norm light &     Excl. censor &                      0.083 $\pm$ 0.005 &                           0.779 $\pm$0.005 &          \textbf{0.159 $\pm$0.017} &                    \textbf{0.882 $\pm$0.001} &                         0.112 $\pm$0.013 \\
    Norm light &     DeepQuantReg &                      0.277 $\pm$ 0.013 &                           0.854 $\pm$0.008 &          1.205 $\pm$0.081 &                     0.878 $\pm$0.001 &                         0.588 $\pm$0.046 \\
    Norm light &      LogNorm MLE &                     58.503 $\pm$ 6.922 &                           2.475 $\pm$0.093 &         13.713 $\pm$0.465 &                     0.861 $\pm$0.003 &                          7.74 $\pm$0.286 \\
      
  \midrule
  
     Norm same &            CQRNN &                      \textbf{0.094 $\pm$ 0.005} &                           \textbf{0.785 $\pm$0.003} &           0.19 $\pm$0.022 &                     0.893 $\pm$0.001 &                         \textbf{0.052 $\pm$0.007} \\
     Norm same &  Sequential grid &                       0.779 $\pm$ 0.16 &                           0.847 $\pm$0.009 &           0.45 $\pm$0.054 &                     0.891 $\pm$0.001 &                         0.102 $\pm$0.012 \\
     Norm same &     Excl. censor &                       0.435 $\pm$ 0.01 &                           0.927 $\pm$0.007 &          4.096 $\pm$0.251 &                     \textbf{0.894 $\pm$0.001} &                          1.398 $\pm$0.07 \\
     Norm same &     DeepQuantReg &                      0.357 $\pm$ 0.013 &                           0.893 $\pm$0.009 &          3.334 $\pm$0.325 &                     0.891 $\pm$0.001 &                         0.983 $\pm$0.091 \\
     Norm same &      LogNorm MLE &                      0.114 $\pm$ 0.008 &                           0.787 $\pm$0.004 &         \textbf{ 0.187 $\pm$0.024} &                     \textbf{0.894 $\pm$0.001} &                         0.059 $\pm$0.006 \\
            
  \midrule
  
 LogNorm heavy &            CQRNN &                      2.424 $\pm$ 0.055 &                           1.123 $\pm$0.021 &          22.493 $\pm$0.36 &                     \textbf{0.782 $\pm$0.005} &                         \textbf{0.036 $\pm$0.004} \\
 LogNorm heavy &  Sequential grid &                       2.42 $\pm$ 0.055 &                           1.121 $\pm$0.021 &         21.938 $\pm$0.299 &                     0.781 $\pm$0.005 &                         0.044 $\pm$0.002 \\
 LogNorm heavy &     Excl. censor &                      2.654 $\pm$ 0.061 &                           1.247 $\pm$0.021 &          35.43 $\pm$0.629 &                     0.772 $\pm$0.005 &                         4.806 $\pm$0.226 \\
 LogNorm heavy &     DeepQuantReg &                       2.639 $\pm$ 0.06 &                           1.236 $\pm$0.022 &         34.132 $\pm$0.719 &                     0.771 $\pm$0.005 &                         3.884 $\pm$0.201 \\
 LogNorm heavy &      LogNorm MLE &                       \textbf{1.17 $\pm$ 0.052} &                           \textbf{0.868 $\pm$0.018} &          \textbf{0.135 $\pm$0.014} &                     0.766 $\pm$0.005 &                         0.074 $\pm$0.008 \\
        
  \midrule
  
  LogNorm med. &            CQRNN &                      1.713 $\pm$ 0.049 &                            0.923 $\pm$0.02 &          5.054 $\pm$0.174 &                     \textbf{0.754 $\pm$0.004} &                         \textbf{0.064 $\pm$0.005} \\
  LogNorm med. &  Sequential grid &                      1.699 $\pm$ 0.047 &                            0.921 $\pm$0.02 &          4.968 $\pm$0.181 &                     \textbf{0.754 $\pm$0.004} &                         0.098 $\pm$0.014 \\
  LogNorm med. &     Excl. censor &                      2.168 $\pm$ 0.053 &                           1.067 $\pm$0.021 &          12.124 $\pm$0.34 &                     0.749 $\pm$0.004 &                         2.586 $\pm$0.071 \\
  LogNorm med. &     DeepQuantReg &                      2.087 $\pm$ 0.056 &                            1.033 $\pm$0.02 &         10.081 $\pm$0.255 &                     0.748 $\pm$0.003 &                         1.373 $\pm$0.055 \\
  LogNorm med. &      LogNorm MLE &                       \textbf{0.907 $\pm$ 0.05} &                           \textbf{0.824 $\pm$0.018} &          \textbf{0.103 $\pm$0.016} &                      0.75 $\pm$0.003 &                          0.07 $\pm$0.009 \\
         
  \midrule
  
 LogNorm light &            CQRNN &                      0.506 $\pm$ 0.028 &                           \textbf{0.764 $\pm$0.019} &          0.331 $\pm$0.026 &                     \textbf{0.729 $\pm$0.003} &                          0.135 $\pm$0.01 \\
 LogNorm light &  Sequential grid &                      0.532 $\pm$ 0.029 &                           0.767 $\pm$0.018 &           0.517 $\pm$0.04 &                     \textbf{0.729 $\pm$0.003} &                          0.21 $\pm$0.014 \\
 LogNorm light &     Excl. censor &                      1.185 $\pm$ 0.037 &                            0.852 $\pm$0.02 &          1.518 $\pm$0.124 &                     \textbf{0.729 $\pm$0.003} &                         0.655 $\pm$0.047 \\
 LogNorm light &     DeepQuantReg &                      0.831 $\pm$ 0.036 &                           0.804 $\pm$0.019 &          0.912 $\pm$0.061 &                     0.726 $\pm$0.003 &                         0.438 $\pm$0.028 \\
 LogNorm light &      LogNorm MLE &                      \textbf{0.432 $\pm$ 0.034} &                           0.767 $\pm$0.018 &          \textbf{0.095 $\pm$0.015} &                     \textbf{0.729 $\pm$0.002} &                         \textbf{0.079 $\pm$0.004} \\
        
  \midrule
  
  LogNorm same &            CQRNN &                      0.137 $\pm$ 0.021 &                           0.735 $\pm$0.015 &          0.236 $\pm$0.021 &                     0.751 $\pm$0.002 &                         0.055 $\pm$0.007 \\
  LogNorm same &  Sequential grid &                      0.422 $\pm$ 0.054 &                           0.753 $\pm$0.016 &          0.463 $\pm$0.046 &                     0.752 $\pm$0.003 &                         0.101 $\pm$0.018 \\
  LogNorm same &     Excl. censor &                      1.068 $\pm$ 0.043 &                           0.861 $\pm$0.019 &          4.112 $\pm$0.264 &                     0.752 $\pm$0.002 &                         1.306 $\pm$0.046 \\
  LogNorm same &     DeepQuantReg &                      0.394 $\pm$ 0.057 &                           0.763 $\pm$0.016 &          1.301 $\pm$0.265 &                     0.748 $\pm$0.002 &                         0.335 $\pm$0.072 \\
  LogNorm same &      LogNorm MLE &                      \textbf{0.067 $\pm$ 0.013} &                            \textbf{0.73 $\pm$0.015} &         \textbf{ 0.114 $\pm$0.013} &                     \textbf{0.754 $\pm$0.002} &                         \textbf{0.052 $\pm$0.005} \\
         
  \midrule

\end{tabular}}
\label{tab_full_resultsa}
\end{table} 

  \newpage
  
  \begin{table}[ht]
\centering
\resizebox{1.\textwidth}{!}{\begin{tabular}{lllllll}
\toprule
     Dataset &        Method & MSE to true quantile & Uncensored quantile loss & Uncensored D-Calibration & Concordance-Index & Censored D-Calibration  \\
     
      &     &  (lower better) & (lower better) & (lower better) & (higher better) & (lower better)  \\

\midrule
&\multicolumn{6}{c}{\textbf{Type 2 -- real data, synthetic censoring}}\\
       Housing &            CQRNN &                                     - &                            \textbf{0.34 $\pm$0.002} &           \textbf{0.793 $\pm$0.03} &                       0.897 $\pm$0.0 &                          \textbf{0.02 $\pm$0.004} \\
       Housing &     Excl. censor &                                     - &                           0.443 $\pm$0.005 &          2.176 $\pm$0.057 &                     0.895 $\pm$0.001 &                         0.311 $\pm$0.011 \\
       Housing &     DeepQuantReg &                                     - &                           0.399 $\pm$0.004 &          2.474 $\pm$0.066 &                     \textbf{0.902 $\pm$0.001} &                         0.196 $\pm$0.031 \\
       Housing &      LogNorm MLE &                                     - &                             0.6 $\pm$0.002 &          2.794 $\pm$0.022 &                     0.881 $\pm$0.001 &                         1.035 $\pm$0.015 \\
       
       \midrule
       
       Protein &            CQRNN &                                     - &                           \textbf{0.435 $\pm$0.001} &          \textbf{0.275 $\pm$0.008} &                     \textbf{0.847 $\pm$0.001} &                         \textbf{0.027 $\pm$0.001} \\
       Protein &     Excl. censor &                                     - &                           0.631 $\pm$0.002 &           3.45 $\pm$0.053 &                     0.838 $\pm$0.001 &                          1.075 $\pm$0.02 \\
       Protein &     DeepQuantReg &                                     - &                           0.568 $\pm$0.002 &          2.809 $\pm$0.059 &                     0.831 $\pm$0.001 &                         0.495 $\pm$0.011 \\
       Protein &      LogNorm MLE &                                     - &                           0.579 $\pm$0.002 &          0.694 $\pm$0.018 &                     0.817 $\pm$0.002 &                         0.298 $\pm$0.007 \\
              
       \midrule
       
          Wine &            CQRNN &                                     - &                             \textbf{0.6 $\pm$0.005} &          \textbf{0.908 $\pm$0.069} &                     \textbf{0.815 $\pm$0.002} &                         \textbf{0.046 $\pm$0.005} \\
          Wine &     Excl. censor &                                     - &                           0.791 $\pm$0.005 &          6.606 $\pm$0.209 &                     0.799 $\pm$0.003 &                         0.722 $\pm$0.038 \\
          Wine &     DeepQuantReg &                                     - &                           0.717 $\pm$0.005 &          3.211 $\pm$0.159 &                     0.792 $\pm$0.003 &                         0.212 $\pm$0.021 \\
          Wine &      LogNorm MLE &                                     - &                           1.454 $\pm$0.022 &          5.736 $\pm$0.163 &                     0.747 $\pm$0.004 &                         0.784 $\pm$0.022 \\
                 
       \midrule
       
           PHM &            CQRNN &                                     - &                           \textbf{0.408 $\pm$0.001} &          \textbf{0.243 $\pm$0.012} &                     \textbf{0.902 $\pm$0.001} &                         \textbf{0.008 $\pm$0.001} \\
           PHM &     Excl. censor &                                     - &                           0.519 $\pm$0.002 &          3.852 $\pm$0.037 &                     0.901 $\pm$0.001 &                         0.481 $\pm$0.006 \\
           PHM &     DeepQuantReg &                                     - &                           0.479 $\pm$0.002 &          1.589 $\pm$0.057 &                     0.897 $\pm$0.001 &                         0.154 $\pm$0.011 \\
           PHM &      LogNorm MLE &                                     - &                           0.599 $\pm$0.002 &           2.26 $\pm$0.018 &                       0.9 $\pm$0.001 &                         0.538 $\pm$0.007 \\
                  
       \midrule
       
     SurvMNIST &            CQRNN &                                     - &                             \textbf{0.076 $\pm$0.0} &          \textbf{0.308 $\pm$0.023} &                     0.899 $\pm$0.001 &                         \textbf{0.224 $\pm$0.005} \\
     SurvMNIST &     Excl. censor &                                     - &                           0.133 $\pm$0.002 &          2.115 $\pm$0.086 &                     0.896 $\pm$0.001 &                         0.512 $\pm$0.014 \\
     SurvMNIST &     DeepQuantReg &                                     - &                             0.1 $\pm$0.001 &          1.021 $\pm$0.051 &                       \textbf{0.9 $\pm$0.001} &                         0.264 $\pm$0.013 \\
     SurvMNIST &      LogNorm MLE &                                     - &                           0.209 $\pm$0.001 &          4.348 $\pm$0.049 &                     0.894 $\pm$0.001 &                         0.806 $\pm$0.015 \\
            
       \midrule \\
       &\multicolumn{6}{c}{\textbf{Type 3 -- real data, real censoring}}\\
      METABRIC &            CQRNN &                                     - &                                         - &                        - &                     \textbf{0.644 $\pm$0.006} &                         \textbf{0.189 $\pm$0.057} \\
      METABRIC &     Excl. censor &                                     - &                                         - &                        - &                     0.615 $\pm$0.005 &                          7.54 $\pm$0.478 \\
      METABRIC &     DeepQuantReg &                                     - &                                         - &                        - &                     0.601 $\pm$0.006 &                         2.393 $\pm$0.211 \\
      METABRIC &      LogNorm MLE &                                     - &                                         - &                        - &                     0.636 $\pm$0.006 &                          0.72 $\pm$0.106 \\
             
       \midrule
       
          WHAS &            CQRNN &                                     - &                                         - &                        - &                      \textbf{0.85 $\pm$0.005} &                         1.089 $\pm$0.431 \\
          WHAS &     Excl. censor &                                     - &                                         - &                        - &                     0.785 $\pm$0.006 &                         9.391 $\pm$0.709 \\
          WHAS &     DeepQuantReg &                                     - &                                         - &                        - &                     0.774 $\pm$0.008 &                          6.71 $\pm$0.448 \\
          WHAS &      LogNorm MLE &                                     - &                                         - &                        - &                      0.81 $\pm$0.005 &                         \textbf{0.817 $\pm$0.139} \\
                 
       \midrule
       
       SUPPORT &            CQRNN &                                     - &                                         - &                        - &                     0.615 $\pm$0.002 &                         \textbf{0.179 $\pm$0.022} \\
       SUPPORT &     Excl. censor &                                     - &                                         - &                        - &                     0.552 $\pm$0.002 &                         9.606 $\pm$0.231 \\
       SUPPORT &     DeepQuantReg &                                     - &                                         - &                        - &                     0.564 $\pm$0.002 &                         6.747 $\pm$0.176 \\
       SUPPORT &      LogNorm MLE &                                     - &                                         - &                        - &                     \textbf{0.617 $\pm$0.002} &                         2.384 $\pm$0.143 \\
              
       \midrule
       
          GBSG &            CQRNN &                                     - &                                         - &                        - &                     \textbf{0.683 $\pm$0.005} &                         \textbf{0.339 $\pm$0.033} \\
          GBSG &     Excl. censor &                                     - &                                         - &                        - &                     0.673 $\pm$0.005 &                        11.732 $\pm$0.512 \\
          GBSG &     DeepQuantReg &                                     - &                                         - &                        - &                     0.673 $\pm$0.005 &                         8.998 $\pm$0.479 \\
          GBSG &      LogNorm MLE &                                     - &                                         - &                        - &                     0.677 $\pm$0.004 &                          0.793 $\pm$0.08 \\
                 
       \midrule
       
     TMBImmuno &            CQRNN &                                     - &                                         - &                        - &                     0.579 $\pm$0.008 &                          \textbf{0.201 $\pm$0.04} \\
     TMBImmuno &     Excl. censor &                                     - &                                         - &                        - &                      0.52 $\pm$0.008 &                         9.479 $\pm$0.386 \\
     TMBImmuno &     DeepQuantReg &                                     - &                                         - &                        - &                     0.539 $\pm$0.011 &                         4.827 $\pm$0.306 \\
     TMBImmuno &      LogNorm MLE &                                     - &                                         - &                        - &                     \textbf{0.581 $\pm$0.007} &                         0.634 $\pm$0.067 \\
            
       \midrule
       
     BreastMSK &            CQRNN &                                     - &                                         - &                        - &                      0.619 $\pm$0.01 &                          \textbf{0.08 $\pm$0.012} \\
     BreastMSK &     Excl. censor &                                     - &                                         - &                        - &                     \textbf{0.646 $\pm$0.008} &                         4.546 $\pm$0.235 \\
     BreastMSK &     DeepQuantReg &                                     - &                                         - &                        - &                     0.638 $\pm$0.008 &                            3.0 $\pm$0.17 \\
     BreastMSK &      LogNorm MLE &                                     - &                                         - &                        - &                      0.631 $\pm$0.01 &                         0.521 $\pm$0.058 \\
            
       \midrule
       
        LGGGBM &            CQRNN &                                     - &                                         - &                        - &                     0.792 $\pm$0.008 &                         0.372 $\pm$0.074 \\
        LGGGBM &     Excl. censor &                                     - &                                         - &                        - &                      0.782 $\pm$0.01 &                         2.166 $\pm$0.303 \\
        LGGGBM &     DeepQuantReg &                                     - &                                         - &                        - &                     0.781 $\pm$0.011 &                         1.275 $\pm$0.184 \\
        LGGGBM &      LogNorm MLE &                                     - &                                         - &                        - &                     \textbf{0.793 $\pm$0.008} &                         \textbf{0.367 $\pm$0.083} \\
\bottomrule

\end{tabular}}
\label{tab_full_results}

\end{table}

\newpage
\FloatBarrier

\begin{table}[t!]
\centering
\caption{Results comparing CQRNN and sequential grid algorithm on the type 3 datasets, real target data with real censoring. Experiments were repeated over 50 random seeds.}

Raw Results, mean $\pm$ one standard error
\resizebox{0.999\textwidth}{!}{
\begin{tabular}{lllllll}
\toprule
Dataset &          Method & TQMSE & UQL & UnDCal & Concordance-Index (higher better) & Censored D-Calibration (lower better) \\
\midrule
   METABRIC &    CQRNN   &                    - &                  - &      - &         0.643 $\pm$0.003 &         0.218 $\pm$0.026 \\
   METABRIC & Sequential grid &                    - &                  - &      - &         0.648 $\pm$0.002 &          0.399 $\pm$0.04 \\
       WHAS &    CQRNN   &                    - &                  - &      - &          0.86 $\pm$0.002 &         0.721 $\pm$0.091 \\
       WHAS & Sequential grid &                    - &                  - &      - &         0.852 $\pm$0.002 &          5.038 $\pm$0.74 \\
    SUPPORT &    CQRNN   &                    - &                  - &      - &         0.614 $\pm$0.001 &          0.159 $\pm$0.01 \\
    SUPPORT & Sequential grid &                    - &                  - &      - &         0.613 $\pm$0.001 &         0.723 $\pm$0.024 \\
       GBSG &    CQRNN   &                    - &                  - &      - &         0.678 $\pm$0.002 &         0.361 $\pm$0.024 \\
       GBSG & Sequential grid &                    - &                  - &      - &         0.678 $\pm$0.002 &         0.789 $\pm$0.042 \\
  TMBImmuno &    CQRNN   &                    - &                  - &      - &         0.571 $\pm$0.003 &         0.207 $\pm$0.021 \\
  TMBImmuno & Sequential grid &                    - &                  - &      - &         0.572 $\pm$0.003 &         0.375 $\pm$0.028 \\
  BreastMSK &    CQRNN   &                    - &                  - &      - &         0.618 $\pm$0.005 &          0.085 $\pm$0.01 \\
  BreastMSK & Sequential grid &                    - &                  - &      - &         0.597 $\pm$0.006 &         0.227 $\pm$0.016 \\
     LGGGBM &    CQRNN   &                    - &                  - &      - &         0.784 $\pm$0.004 &         0.397 $\pm$0.039 \\
     LGGGBM & Sequential grid &                    - &                  - &      - &         0.781 $\pm$0.004 &         0.491 $\pm$0.041 \\
\bottomrule
\end{tabular}}

\vspace{0.5in}
Difference in means, alongside 95\% confidence interval
\resizebox{0.999\textwidth}{!}{
\begin{tabular}{lrrrrrrrr}
\toprule
Dataset &  Number &  Training time &  Test time &  Parameter & C-Index difference & CQRNN is & CensDCal difference & CQRNN is\\
& quantiles & speed up & speed up & saving & Seq. grid - CQRNN & sig. better? & Seq. grid - CQRNN & sig. better? \\
& & & & & (>0 favours CQRNN) &  & (<0 favours CQRNN) \\
\midrule
   METABRIC &  5 & 5.3$\times$ & 2.5$\times$ &  4.6$\times$ & -0.005 $\pm$ 0.001 & \xmark & -0.181 $\pm$ 0.032 & \cmark \\
       WHAS &  5 & 5.1$\times$ & 3.9$\times$ &  4.6$\times$ &  0.008 $\pm$ 0.001  & \cmark & -4.317 $\pm$ 0.745 & \cmark \\
    SUPPORT &  5 & 5.1$\times$ & 2.7$\times$ &  4.6$\times$ &  0.001 $\pm$ 0.000  & \cmark & -0.564 $\pm$ 0.022 & \cmark \\
       GBSG &  5 & 5.3$\times$ & 5.1$\times$ &  4.6$\times$ & -0.001 $\pm$ 0.000 & \xmark & -0.428 $\pm$ 0.026 & \cmark \\
  TMBImmuno &  5 & 5.1$\times$ & 4.3$\times$ &  4.6$\times$ & -0.000 $\pm$ 0.001 & -- & -0.168 $\pm$ 0.025 & \cmark \\
  BreastMSK &  5 & 5.1$\times$ & 5.0$\times$ &  4.6$\times$ &   0.021 $\pm$ 0.004  & \cmark & -0.141 $\pm$ 0.018 & \cmark \\
     LGGGBM &  5 & 5.4$\times$ & 5.1$\times$ &  4.6$\times$ &  0.003 $\pm$ 0.001  & \cmark & -0.094 $\pm$ 0.038 & \cmark \\
\midrule
CQRNN better: & & & & & & 4/7 & & 7/7 \\
Seq grid better: & & & & & & 2/7 & & 0/7 \\
No sig. difference: & & & & & & 1/7 & & 0/7 \\
\bottomrule
\end{tabular}}


\label{tab_grid_cqrnn_compare_real}
\end{table}

\begin{figure}[b!]
\begin{center}
\vspace{-0.02in}


\hspace{0.45in} Norm linear \hspace{0.45in} Norm uniform \hspace{0.45in} Norm non-linear

\includegraphics[width=0.245\columnwidth]{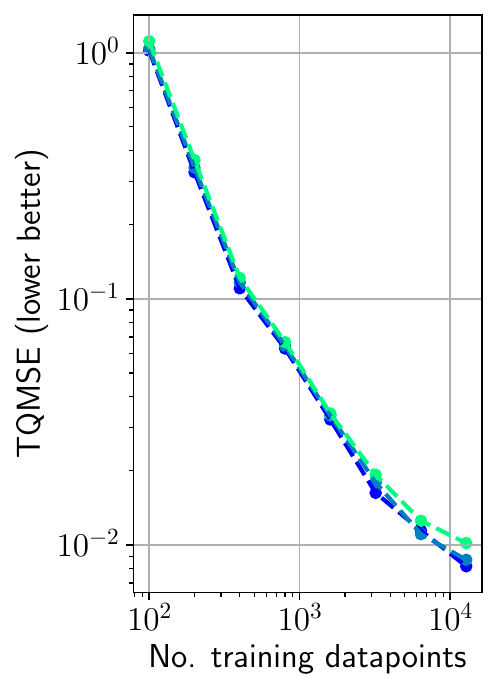}
\includegraphics[width=0.245\columnwidth]{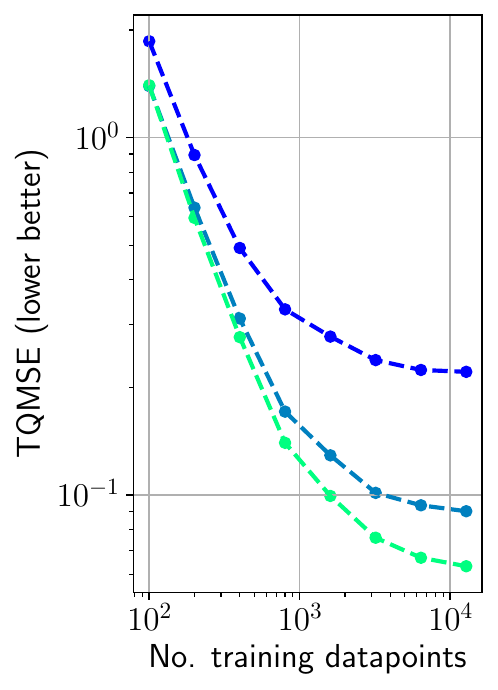}
\includegraphics[width=0.245\columnwidth]{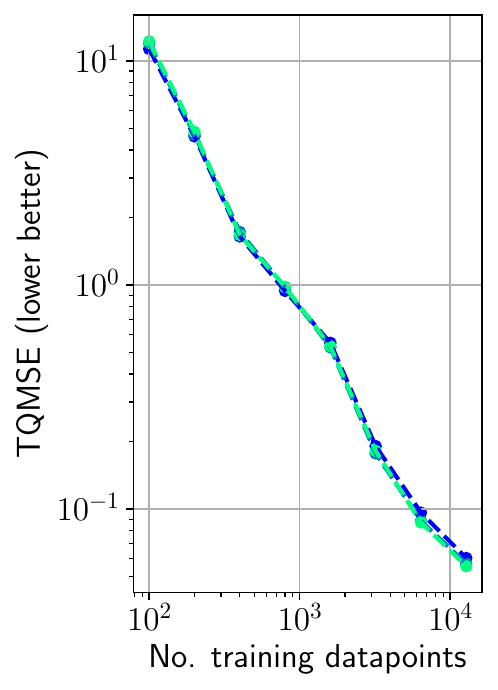}
\put(-0,50){\transparent{1.}\includegraphics[width=0.15\columnwidth]{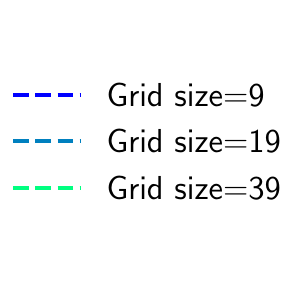}}

\hspace{0.3in} Exponential \hspace{0.7in} Weibull \hspace{0.7in} LogNorm

\includegraphics[width=0.245\columnwidth]{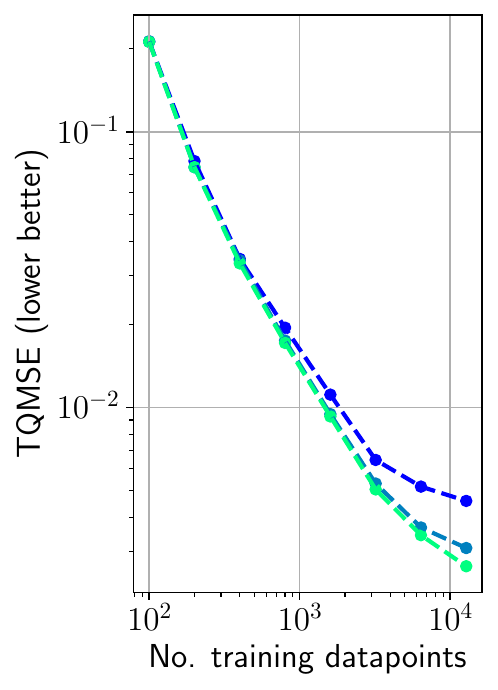}
\includegraphics[width=0.245\columnwidth]{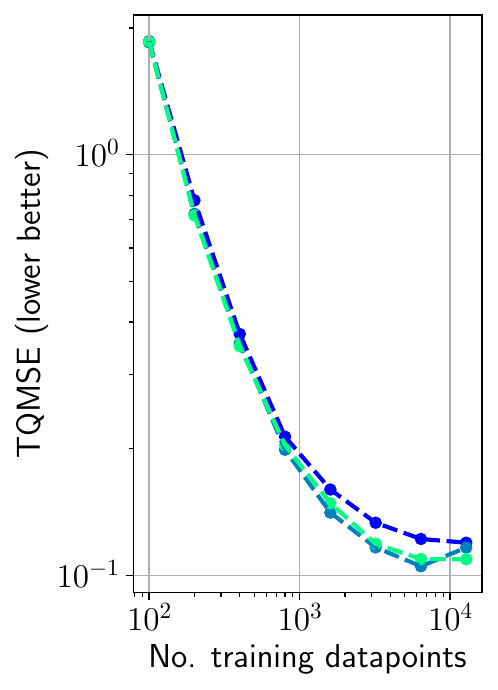}
\includegraphics[width=0.245\columnwidth]{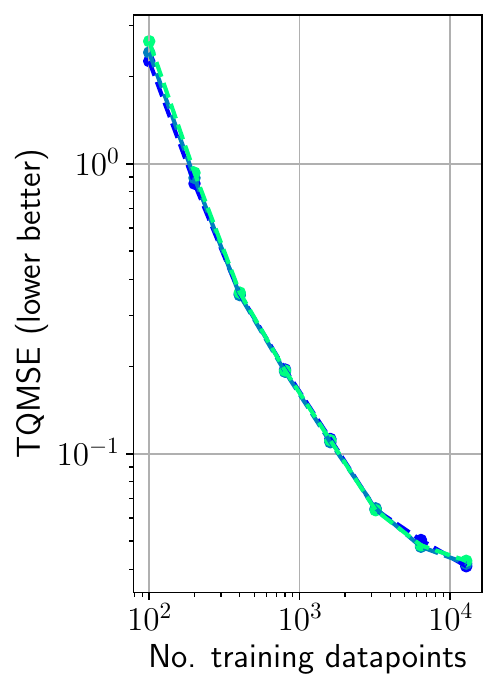}





\caption{Ablations over grid size and number of datapoints for various synthetic datasets using our CQRNN method.}
\label{fig_grid_ablation}
\end{center}
\vskip -0.2in
\end{figure}
\newpage
\FloatBarrier

\subsection{Hyperparameter ablations}
\label{sec_app_further_results_hyper}

The requirements for CQRNN in Algorithm \ref{alg_cqrnn_loss} include two hyperparameters unique to CQRNN (the rest determine the NN and its optimisation) -- the grid of quantiles, $\operatorname{grid}_\tau$, and the large pseudo y value, $y^*$. This section investigates the role of these. It also discusses two modifications which were tested in initial experiments but which were not deemed essential for good performance, and excluded from the proposed method -- mitigating crossing quantiles and interpolating between quantiles.

\textbf{Grid size.}
For evenly spaced grids, as considered in this work, the number of quantiles estimated, $M$, fully determines the grid. One might expect that a larger grid size, with finer increments between quantiles, would lead to a better fit, since censored weights can be estimated more accurately. 
Indeed, the convergence rate of Portnoy's estimator was found to depend on grid size, $\mathcal{O}(1/(MN))$ \citep{Neocleous2006}.

We hypothesised that a larger grid might only deliver benefit when the dataset was sufficiently large for these fine-grained quantiles to be distinguished. 
Hence, Figure \ref{fig_grid_ablation} shows an ablation investigating the interaction between grid size and number of datapoints for each of our type 1 synthetic 1D functions. Experiments were repeated with 100 random seeds, which was required to reduce the error bars sufficiently for comparison. Grid size was varied, $M \in \{9,19,39\}$, and number of datapoints, $N \in \{100,200,400,800,1600,3200,6400,12800\}$. 

For all datasets and grid sizes, TQMSE decreases with dataset size.
In general a larger grid size does produce lower TQMSE, and in three datasets the benefit is significant (Norm uniform, Exponential, Weibull), with the advantage indeed more pronounced with a larger dataset size. 
In two datasets (Norm non-linear, LogNorm), this benefit is slight, and a larger grid is even seen to be slightly harmful on very small datasets.
One dataset (Norm linear) does not follow this trend, where the widest grid proves slightly harmful across all dataset sizes.

\textbf{Pseudo y value.}
\cite{Portnoy2003} proposed that $y^*$ could be set to any large value approximating $\infty$, with the R package `\Verb"quantreg"' setting $y^* = 1e6$. Since the sequential grid algorithm learns quantiles sequentially, it can simply halt if it attempts to estimate a quantile for which only censored datapoints remain, and hence is undefined.

The CQRNN algorithm changes this situation in two ways. Firstly, since the algorithm is no longer sequential and learns all quantiles simultaneously, it does not have the option of halting. Secondly, since a non-linear function is learnt, it is possible that higher quantiles might be undefined in one region of the input space, whilst being learnable in the rest of the space. 
Regressing towards an arbitrarily large $y^*$ value for a portion of the input space could adversely impact the quantile estimate elsewhere.
When all quantiles in $\operatorname{grid}_\tau$ are fully defined, the effect is no different to using $\infty$.

To accommodate these differences, we recommend setting $y^*$ to a more modest value. We define it in terms of the maximum $y$ value in the training set, $y^* = c_{y^*} \max_i y_i$, for a hyperparemeter, $c_{y^*}>1$. 
In real-world problems, a practitioner might use an estimate for this based on their knowledge about the maximum feasible value for the target. In lieu of that, we set $c_{y^*} = 1.2$ for all our experiments (except the below ablation!), which provided reasonable results across datasets.

Table \ref{tab_ymax} presents an ablation on four of our (multidimensional) type 1 synthetic datasets, using $c_{y^*} \in \{1.0, 1.2, 1.5, 2.0, 10,9, 100.0\}$.
For dataset Norm light, $y^*$ has no impact since the target distribution is fully defined under censoring.
However, other datasets have input regions where the higher quantiles are \textit{not} defined due to censoring, and hence higher quantiles are impacted by the magnitude of $y^*$. The optimal value varies by dataset (e.g. $1.2$ is best for Norm heavy, while $10.0$ is best for LogNorm heavy).

\begin{table}[ht]
\centering
\caption{Ablation over psuedo y value, $y^*$. Mean over ten runs, all hyperparameters fixed. } 
\resizebox{0.7\textwidth}{!}{\begin{tabular}{lrrrrrr}
\toprule
    &  \multicolumn{6}{c}{----- Target value for pseudo-datapoints, $y^*$ -----}\\
   Dataset & \small $1.0\max_i y_i$ &  \small  $1.2 \max_i y_i$ &  \small  $1.5 \max_i y_i$ &  \small  $2.0 \max_i y_i$ &  \small   $10.0 \max_i y_i$ &  \small    $100.0 \max_i y_i$ \\
\midrule
 
& \multicolumn{6}{c}{\textbf{TQMSE (lower better)}}\\

    Norm heavy & 1.452 & 0.579 & 1.237 & 6.196 & 591.257 & 57421.824 \\
    Norm light & 0.081 & 0.079 & 0.081 & 0.081 &   0.081 &     0.081 \\
 LogNorm heavy & 2.502 & 2.424 & 2.321 & 2.173 &   1.026 &    19.827 \\
 LogNorm light & 0.614 & 0.506 & 0.385 & 0.252 &   0.147 &     0.147 \\









\bottomrule
\end{tabular}}
\label{tab_ymax}
\end{table}


\textbf{Crossing quantiles.} One issue often discussed in quantile regression is `crossing quantiles'. Higher quantiles should \textit{always} produce higher predictions, i.e. $\hat{y}_{i,{\tau_1}} > \hat{y}_{i,{\tau_2}} \forall \tau_1 > \tau_2$. The crossing quantile issue arises when this condition does not hold. We anticipated that the flexibility of NNs might exacerbate this issue, and we tested two methods to remedy this. 1) Adding a crossing penalty to the loss \citep{Bondell2010}, $\mathcal{L}_\text{cross} = \sum_{i=1}^N \sum_{j=1}^{N_\tau-1} \max [0, c - (\hat{y}_{i,\text{grid}_\tau[j+1]} - \hat{y}_{i,\text{grid}_\tau[j]})]$, where $c$ is the smallest acceptable distance between neighbouring quantiles. 
2) Modifying the NN architecture to enforce monotonicity between quantiles, by constraining each consecutive quantile prediction to add on to the previous one, after passing through a $\operatorname{SoftPlus}$. 
In our experiments, neither method significantly impacted performance. Favouring simplicity, we propose the CQRNN algorithm without these.
It's possible that other methods might have more effect, e.g. \citep{Zhou2020crossing, Brando2022}.
We leave further exploration to future work.

\textbf{Interpolating quantiles.} The CQRNN algorithm sets the estimated censored quantiles, $\mathbf{q}$, by choosing the prediction closest to the censored datapoint, $\hat{q}_j \gets \arg \min_\tau |\hat{y}_{j,\tau} -  y_j|$. In early experiments, we considered an alternative approach that took a linear interpolation between the two nearest quantiles. In initial experiments this wasn't found to significantly improve performance, so we propose the CQRNN algorithm using the simpler $\arg \min$ approach. 


\textbf{Partial vs. full optimisation.} Figure \ref{fig_partial_full} explores the effect of optimising the NNs partially, as is proposed in CQRNN in Algorithm \ref{alg_cqrnn_loss}, compared to fully, as might be done more typically in EM procedures. The figure shows that convergence is fastest when using partial maximisation, when the most up-to-date estimates of $\hat{q}_i$ are used. This ends up being more efficient than freezing $\hat{q}_i$ and only updating after a longer period of optimisation.

\begin{figure}[h!]
\begin{center}
\vspace{-0.02in}


\includegraphics[width=0.32\columnwidth]{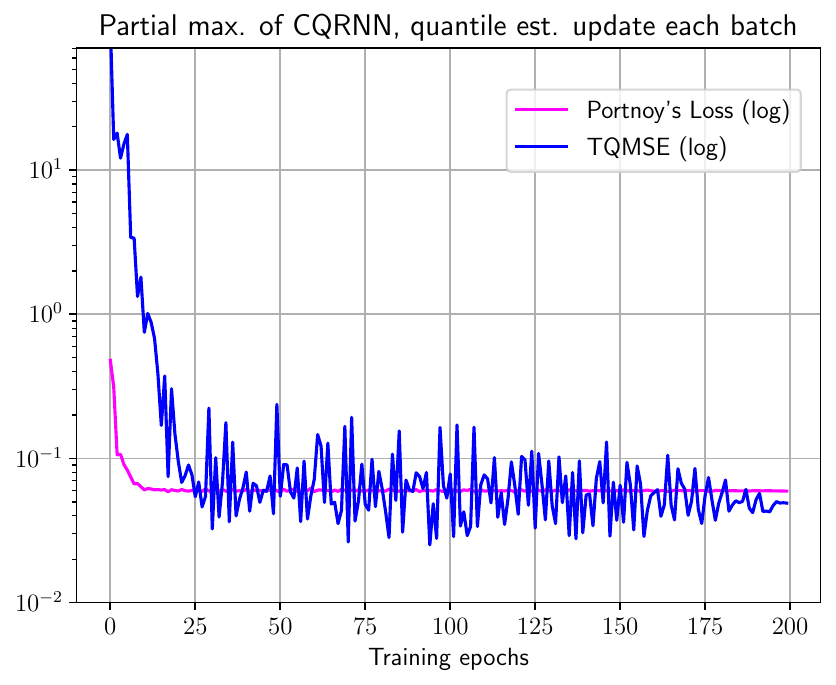}
\includegraphics[width=0.32\columnwidth]{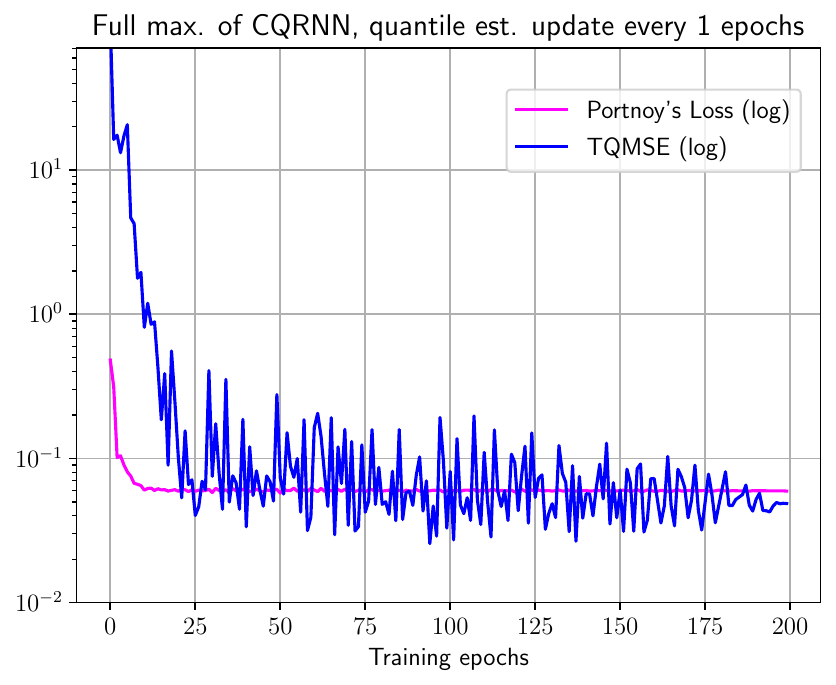}
\includegraphics[width=0.32\columnwidth]{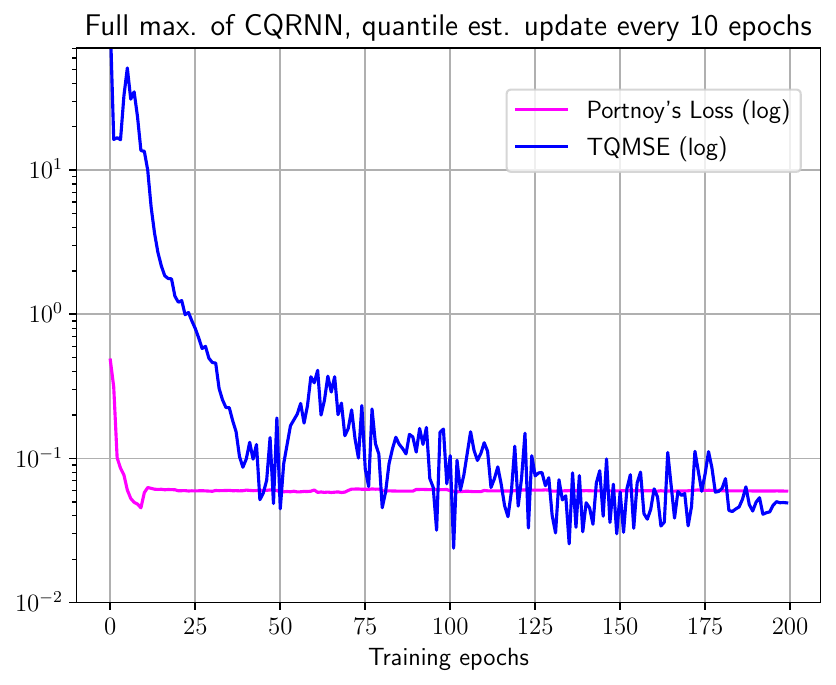}
\includegraphics[width=0.32\columnwidth]{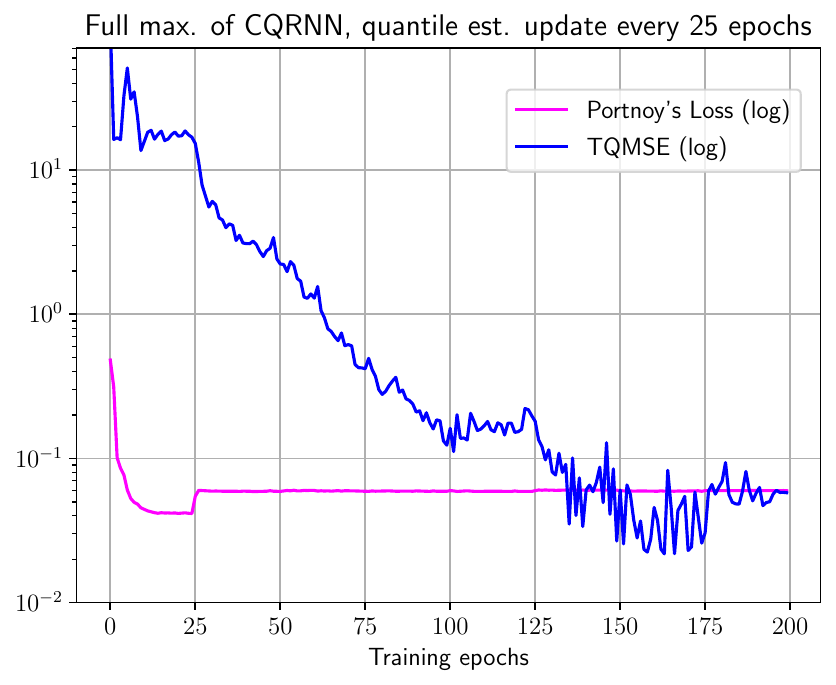}
\includegraphics[width=0.32\columnwidth]{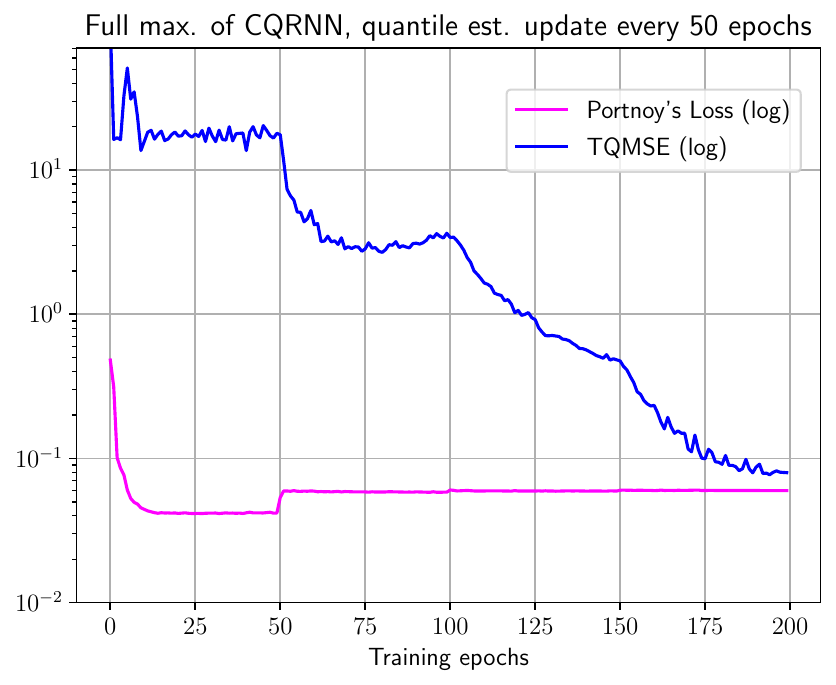}
\includegraphics[width=0.32\columnwidth]{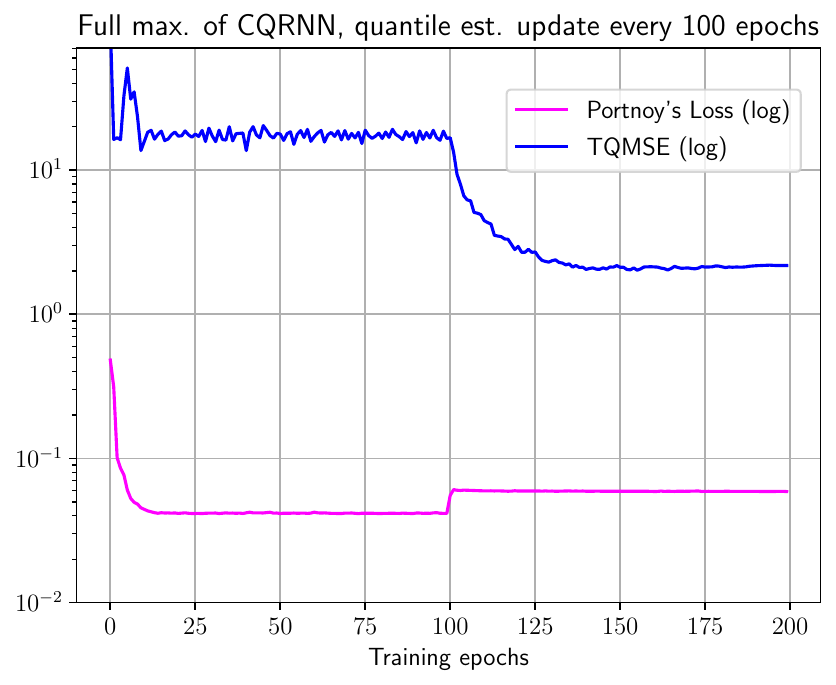}


\caption{This figure explores the effect of optimising the NNs partially compared to fully. It shows training loss and TQMSE over training epochs. The Normal Linear dataset was used. 
}
\label{fig_partial_full}
\end{center}
\vskip -0.2in
\end{figure}

\end{document}